\documentclass{article}

\addtocontents{toc}{\protect\setcounter{tocdepth}{-2}}

\usepackage[accepted]{icml2026}

\usepackage{tikz}
\usetikzlibrary{shapes.geometric, arrows, positioning, calc, fit, backgrounds}
\definecolor{expertgreen}{RGB}{220, 255, 220}
\definecolor{gateblue}{RGB}{220, 220, 255}
\definecolor{studentred}{RGB}{255, 220, 220}
\definecolor{routerorange}{RGB}{255, 240, 200}
\usepackage{pgfplots}
\usepgfplotslibrary{fillbetween}

\usepackage{amsfonts}
\usepackage{nicefrac}
\usepackage{microtype}
\usepackage{graphicx}
\usepackage{subcaption}
\usepackage{booktabs} 
\usepackage{hyperref}
\usepackage{url}

\renewcommand{\algorithmicrequire}{\textbf{Input:}}
\renewcommand{\algorithmicensure}{\textbf{Output:}}

\usepackage{amsmath}
\usepackage{amssymb}
\usepackage{mathtools}
\usepackage{amsthm}

\usepackage[capitalize,noabbrev]{cleveref}

\usepackage{dsfont}
\usepackage[mathscr]{euscript}
\usepackage{expl3}
\usepackage{xcolor}
\usepackage{colortbl}
\usepackage{thmtools}
\usepackage{thm-restate}
\usepackage{mathtools}
\usepackage{multirow}
\usepackage{wrapfig}
\usepackage{accents}

\usepackage{pifont}

\usepackage{MnSymbol}
\DeclareMathAlphabet\mathbb{U}{msb}{m}{n}
\usepackage{xpatch}

\def\Rset{\mathbb{R}}

\DeclareMathOperator*{\E}{\mathbb E}
\DeclareMathOperator*{\argmax}{argmax}
\DeclareMathOperator*{\argmin}{argmin}

\DeclareMathOperator{\supp}{supp}

\DeclareMathOperator{\proj}{\sfP} 
\DeclareMathOperator{\clip}{clip} 

\DeclarePairedDelimiter{\abs}{\lvert}{\rvert} 
\DeclarePairedDelimiter{\bracket}{[}{]}
\DeclarePairedDelimiter{\curl}{\{}{\}}
\DeclarePairedDelimiter{\paren}{(}{)}
\DeclarePairedDelimiter{\norm}{\|}{\|}

\ExplSyntaxOn
\tl_const:Nn \c_my_uc_alphabet_tl { ABCDEFGHIJKLMNOPQRSTUVWXYZ }
\tl_const:Nn \c_my_full_alphabet_tl { ABCDEFGHIJKLMNOPQRSTUVWXYZ abcdefghijklmnopqrstuvwxyz }

\tl_map_inline:Nn \c_my_uc_alphabet_tl
 { \cs_gset:cpn { c#1 } { \mathcal{#1} } }

\tl_map_inline:Nn \c_my_uc_alphabet_tl
 { \cs_gset:cpn { s#1 } { \mathscr{#1} } }

\tl_map_inline:Nn \c_my_full_alphabet_tl
 {
  \str_if_eq:nnTF { b#1 } { bf } 
    { } 
    {
      \str_if_eq:nnTF { sf#1 } { sf }
        { } 
        {
          \cs_gset:cpn { b#1 } { \ensuremath{\mathbf{#1}} } 
          \cs_gset:cpn { sf#1 } { \ensuremath{\mathsf{#1}} }
        }
    }
 }
\ExplSyntaxOff

\newcommand{\Rad}{\mathfrak R}
\newcommand{\bsigma}{{\boldsymbol \sigma}}

\newcommand{\bpi}{{\boldsymbol \pi}}

\newcommand{\h}{\widehat}
\newcommand{\ov}{\overline}
\newcommand{\wt}{\widetilde}
\newcommand{\e}{\epsilon}
\newcommand{\ignore}[1]{}

\newlength{\dhatheight}

\newcommand{\KL}{\sfD_{\mathrm{KL}}}

\newcommand{\JSD}{\sfD_{\mathrm{JSD}}}

\newcommand{\sGone}{\cG_{1}}

\hypersetup{
  breaklinks   = true, 
  colorlinks   = true, 
  urlcolor     = blue, 
  linkcolor    = blue, 
  citecolor   = blue 
}

\usepackage[toc, page, header]{appendix}
\theoremstyle{plain}
\newtheorem{theorem}{Theorem}[section]
\newtheorem{proposition}[theorem]{Proposition}
\newtheorem{lemma}[theorem]{Lemma}
\newtheorem{corollary}[theorem]{Corollary}
\theoremstyle{definition}

\newtheorem{assumption}[theorem]{Assumption}
\theoremstyle{remark}

\usepackage[disable,textsize=tiny]{todonotes}

\icmltitlerunning{A Theoretical Framework for Modular Learning of Robust Generative Models}

\begin{document}

\twocolumn[
  \icmltitle{A Theoretical Framework for Modular Learning of Robust Generative Models}

\begin{icmlauthorlist}
\icmlauthor{Corinna Cortes}{google}
\icmlauthor{Mehryar Mohri}{google,courant}
\icmlauthor{Yutao Zhong}{google}
\end{icmlauthorlist}

\icmlaffiliation{google}{Google Research, New York, NY;}
\icmlaffiliation{courant}{Courant Institute of Mathematical Sciences, New York, NY}

\icmlcorrespondingauthor{Corinna Cortes}{corinna@google.com}
\icmlcorrespondingauthor{Mehryar Mohri}{mohri@google.com}
\icmlcorrespondingauthor{Yutao Zhong}{yutaozhong@google.com}

\icmlkeywords{Generative Models, Optimization, Robustness, Modularity, Large Language Models, Game Theory, Generalization Bounds}

\vskip 0.3in
]

\printAffiliationsAndNotice{}

\setlength{\abovedisplayskip}{5pt plus 1pt minus 1pt}

\setlength{\textfloatsep}{10pt plus 1.0pt minus 2.0pt} 
\setlength{\floatsep}{10pt plus 1.0pt minus 2.0pt}     
\setlength{\intextsep}{10pt plus 1.0pt minus 2.0pt}    



\begin{abstract}
  Training large-scale generative models is resource-intensive and
  relies heavily on heuristic dataset weighting. We address two
  fundamental questions: Can we train Large Language Models (LLMs)
  modularly, combining small, domain-specific experts to match
  monolithic performance, and can we do so robustly for \emph{any} data
  mixture, eliminating heuristic tuning? We present a theoretical
  framework for \emph{modular} generative modeling where a set of
  pre-trained experts are combined via a gating mechanism. We define
  the space of normalized gating functions $\sGone$ and formulate the
  problem as a minimax game to find a single robust gate that
  minimizes divergence to the worst-case data mixture. We prove the
  existence of such a robust gate using Kakutani's fixed-point theorem
  and show that modularity acts as a strong regularizer, with
  generalization bounds scaling with the lightweight gate's
  complexity. Furthermore, we prove that this modular approach can
  theoretically outperform models retrained on aggregate data, with
  the gap characterized by the Jensen-Shannon Divergence. Finally, we
  introduce a scalable Stochastic Primal-Dual algorithm and a
  \emph{Structural Distillation} method for efficient
  inference. Empirical results on synthetic and real-world datasets
  confirm that our  modular architecture effectively mitigates gradient
  conflict and can robustly outperform monolithic baselines.
\end{abstract}

\section{Introduction}

Training large-scale generative models, such as Large Language Models
(LLMs), is notoriously expensive and often impractical to repeat for
every new dataset \citep{Brown2020, Hoffmann2022}. The computational
cost and environmental footprint of these dense models have raised
significant sustainability concerns \citep{Strubell2019,
  Schwartz2020}.  This monolithic paradigm faces two critical
challenges. First, \emph{sustainability and adaptability:} can we
train LLMs modularly, learning small, accurate models on individual
domains (e.g., math, coding) and combining them to match a giant
model? If so, training becomes dramatically cheaper and greener;
updates require training only a new module and the lightweight
combiner, avoiding catastrophic forgetting \citep{Kirkpatrick2017,
  Parisi2019Continual} and enabling the efficient reuse of pretrained
experts \citep{Pfeiffer2023}. In the future, privacy regulations and publisher paywalls could also restrict access to data domains, smaller models trained by the data owners could constitute the only viable path to data access. Second, \emph{robustness:} standard
training relies on heuristic importance weights across datasets
\citep{Gao2020, Touvron2023}, or static optimization targets
\citep{Xie2023DoReMi}, often failing when test distributions differ
from training assumptions \citep{Koh2021}. Can we build a modular LLM
that is accurate for \emph{any} mixture of datasets, eliminating
heuristic weighting entirely?

We provide an affirmative answer to both questions, offering the first
rigorous game-theoretic framework for robust modularity. Unlike
heuristic approaches like simple parameter averaging (Model Soups)
\citep{Wortsman2022}, task arithmetic \citep{Ilharco2023}, or standard
Mixture of Experts which rely on auxiliary load-balancing losses
\citep{Shazeer2017, Fedus2021Switch}, we seek a single system that is
robust to \emph{any} arbitrary mixture of the underlying source
distributions. We propose a \emph{gated solution},
$\pi_g(x) = \sum_k g(x, k) \h \pi_k(x)$, where an adaptive gate
dynamically reweights frozen experts. Our goal is to find a robust
gate $g^*$ that minimizes the divergence to the worst-case data
mixture.

\textbf{Contributions.} Our main contributions are: (1)
\emph{Theoretical Framework:} We define the normalized gate space
$\sGone$ and formulate robustness as a minimax game. We prove the
existence of a robust gate using Kakutani's fixed-point theorem,
establishing a stable upper bound on the worst-case risk
(\cref{th:robust-existence}).
(2) \emph{Generalization Analysis:} We derive bounds showing that
sample complexity scales with the lightweight gate complexity and
the \emph{expert coincidence norm} $C_\Pi$, rather than the massive
expert parameters.
(3) \emph{Comparison with Retraining:} We prove an
information-theoretic bound showing our modular approach can
outperform monolithic retraining, with the performance gap
characterized by the Jensen-Shannon Divergence
(\cref{thm:jsd_gap}).
(4) \emph{Scalable Algorithm \& Inference:} We introduce a
Stochastic Primal-Dual algorithm for the constrained game and a
\emph{Structural Distillation} method \citep{Hinton2015} to map the
non-causal gate to a causal router for efficient autoregressive
inference.
(5) \emph{Empirical Validation:} We demonstrate on synthetic
benchmarks and real-world datasets (Wikipedia, Code, FineWeb) that
our approach mitigates gradient conflict \citep{Yu2020},
outperforming baselines in high-interference regimes.

\textbf{Organization.} \cref{sec:setup} formalizes the
problem. \cref{sec:theoretical-analysis} presents existence proofs and
comparisons. \cref{sec:optimization-robust} details the optimization
algorithm. \cref{sec:sampling,sec:distillation} addresses inference and
distillation. \cref{sec:experiments} presents empirical results.

\textbf{Related Work.} Our framework bridges theoretical
Multiple-Source Domain Adaptation (MSA)
\citep{MansourMohriRostamizadeh2008} with modular architectures. While
rooted in MSA, we diverge from recent value-based routing
\citep{DannMansourMarinovMohri2025} by addressing the
\emph{generative} setting, which introduces the unique mathematical
challenge of enforcing global normalization ($Z_g=1$). Structurally,
our approach differs from Mixture of Experts \citep{Fedus2021Switch}
by operating on \emph{frozen} experts; from static merging methods
like Model Soups \citep{Wortsman2022} by using dynamic,
input-dependent gating; and from learning to defer
\citep{CortesDeSalvoMohri2016} by optimizing for soft probabilistic
mixing rather than hard selection. An extended discussion is provided
in Appendix~\ref{app:related-work}.

\section{Setup \& Problem Formulation}
\label{sec:setup}

Let $D_k$, $k \in [1, p]$, denote $p$ datasets with empirical
distributions $\h \sfp_k$. We assume access to pre-trained models
$\h \pi_k$ approximating each distribution with guarantees
$\KL(\h \sfp_k \parallel \h \pi_k) \leq \e_k$.
We consider \emph{gated
  solutions} $\pi_g(x) = \sum_{k = 1}^p g(x, k) \h \pi_k(x)$, where
$g(x, \cdot) \in \Delta$ is a gating function (see
Figure~\ref{fig:modular_architecture}). Our goal is to approximate any
mixture $\h \sfp_\lambda = \sum_{k = 1}^p \lambda_k \h \sfp_k$ for
$\lambda \in \Delta$.
\ignore{
\setlength{\intextsep}{0pt}
\setlength{\columnsep}{5pt}
\begin{wrapfigure}{r}{0.5\columnwidth}
  \centering
  \resizebox{0.48\columnwidth}{!}{
    \begin{tikzpicture}[
    node distance=1.5cm,
    block/.style={rectangle, draw, fill=blue!10, text width=2.5cm, text centered, rounded corners, minimum height=1.2cm},
    expert/.style={rectangle, draw, fill=green!10, text width=2.5cm, text centered, rounded corners, minimum height=1cm},
    sum/.style={circle, draw, fill=blue!20, node distance=1cm, minimum size=0.8cm},
    input/.style={coordinate},
    output/.style={coordinate},
    label/.style={text width=2cm, align=center, font=\small}
]

\node [input] (input) {};
\node [left=0.0cm of input, align=center] {\textbf{Input} \\ $x$};

\node [block, right=1.5cm of input, yshift=1.5cm] (gate) {\textbf{Adaptive Gate} \\ $g(x, \cdot)$};

\node [expert, right=1.5cm of input, yshift=-0.5cm] (expert1) {Expert $\h \pi_1$};
\node [below=0.3cm of expert1] (dots) {$\vdots$};
\node [expert, below=0.3cm of dots] (expertp) {Expert $\h \pi_p$};

\draw[dashed, thick, color=gray] ($(expert1.north west)+(-0.2,0.2)$) rectangle ($(expertp.south east)+(0.2,-0.2)$);
\node [anchor=south east, color=gray, font=\small\itshape] at ($(expert1.north east)+(0,0.2)$) {Frozen / Pre-trained};

\node [sum, right=2cm of dots] (sigma) {$\sum$};

\node [output, right=1.5cm of sigma] (output) {};
\node [above=0.0cm of output, align=center] {\textbf{Robust Mixture} \\ $\pi_g(x)$};

\draw [->, thick] (input) -- (gate.west);
\draw [->, thick] (input) -- (expert1.west);
\draw [->, thick] (input) -- (expertp.west);

\draw [->, thick] (gate.east) -| node[right=.0cm] {Weights $g(x,k)$} (sigma.north);
\draw [->, thick] (expert1.east) -- (sigma.west);
\draw [->, thick] (expertp.east) -- (sigma.west);

\draw [->, thick] (sigma.east) -- (output);

\end{tikzpicture}
  }
  \vskip -5pt
  \caption{Conceptual Architecture of the Modular Gated Solution.}
  \label{fig:modular_architecture}
  \vskip -10pt
\end{wrapfigure}
}

\begin{figure}[t]
\centering
\scalebox{.6}{
\begin{tikzpicture}[
    node distance=1.5cm,
    block/.style={rectangle, draw, fill=blue!10, text width=2.5cm, text centered, rounded corners, minimum height=1.2cm},
    expert/.style={rectangle, draw, fill=green!10, text width=2.5cm, text centered, rounded corners, minimum height=1cm},
    sum/.style={circle, draw, fill=blue!20, node distance=1cm, minimum size=0.8cm},
    input/.style={coordinate},
    output/.style={coordinate},
    label/.style={text width=2cm, align=center, font=\small}
]

\node [input] (input) {};
\node [left=0.0cm of input, align=center] {\textbf{Input} \\ $x$};

\node [block, right=1.5cm of input, yshift=1.5cm] (gate) {\textbf{Adaptive Gate} \\ $g(x, \cdot)$};

\node [expert, right=1.5cm of input, yshift=-0.5cm] (expert1) {Expert $\h \pi_1$};
\node [below=0.3cm of expert1] (dots) {$\vdots$};
\node [expert, below=0.3cm of dots] (expertp) {Expert $\h \pi_p$};

\draw[dashed, thick, color=gray] ($(expert1.north west)+(-0.2,0.2)$) rectangle ($(expertp.south east)+(0.2,-0.2)$);
\node [anchor=south east, color=gray, font=\small\itshape] at ($(expert1.north east)+(0,0.2)$) {Frozen / Pre-trained};

\node [sum, right=2cm of dots] (sigma) {$\sum$};

\node [output, right=1.5cm of sigma] (output) {};
\node [above=0.0cm of output, align=center] {\textbf{Robust Mixture} \\ $\pi_g(x)$};

\draw [->, thick] (input) -- (gate.west);
\draw [->, thick] (input) -- (expert1.west);
\draw [->, thick] (input) -- (expertp.west);

\draw [->, thick] (gate.east) -| node[right=.0cm] {Weights $g(x,k)$} (sigma.north);
\draw [->, thick] (expert1.east) -- (sigma.west);
\draw [->, thick] (expertp.east) -- (sigma.west);

\draw [->, thick] (sigma.east) -- (output);

\end{tikzpicture}
}
\caption{Conceptual Architecture of the Modular Gated Solution.}
\label{fig:modular_architecture}
\end{figure}

We define the space of \emph{normalized} gating functions $\sGone$ as
the subset of gates $g \in \prod_{x \in \sX_0} \Delta([1, p])$, that
is $g(x, k) \ge 0$ and $\sum_k g(x, k) = 1$ for all $x \in \sX_0$,
satisfying the global normalization constraint:
$\sGone = \curl*{g \colon Z_g = \sum_{x \in \sX_0} \sum_{k = 1}^p g(x,
  k) \h \pi_k(x) = 1}$, where $\sX_0 = \bigcup_k \supp(\h
\sfp_k)$ and is hence finite. For any $g \in \sGone$, the resulting model $\pi_g$ is a
valid probability distribution.
\begin{restatable}{lemma}{GoneProperties}
\label{lemma:Gone-properties}
The family $\sGone$ is non-empty, compact, and convex.
\end{restatable}
See Appendix~\ref{app:Gone-properties} for the proof.  We seek a
single gate $g^* \in \sGone$ robust to the \emph{worst-case} mixture
$\lambda \in \Delta$. We formulate this as a minimax game:
\[
\min_{g \in \sGone} \max_{\lambda \in \Delta([1, p])} \KL(\h \sfp_\lambda \parallel \pi_g).
\]
We use the relative entropy rather than cross-entropy because the
entropy term $H(\h \sfp_\lambda)$ varies with the adversarial choice
of $\lambda$. Minimizing the worst-case cross-entropy
$\max_\lambda \E_{\h \sfp_\lambda}[-\log \pi]$ is thus not equivalent
to minimizing the divergence
$\max_\lambda \KL(\h \sfp_\lambda \| \pi)$, which ensures the model
approximates the distribution $\h \sfp_\lambda$ itself.

\section{Theoretical Analysis}
\label{sec:theoretical-analysis}

We now establish the existence of a robust gate and quantify its
advantages over monolithic retraining.

\subsection{Fixed Mixture \& Robust Existence}
\label{sec:geometric_interpretation}

First, we consider a simple non-adaptive baseline. If we fix the
mixture weights $\lambda$, a constant gate
$g_\lambda(x, k) = \lambda_k$ (which belongs to $\sGone$) achieves an
average error bound.

\begin{restatable}[Fixed Mixture Guarantee]{proposition}{FixedMixture}
\label{prop:fixed-mixture}
For any fixed $\lambda \in \Delta$, the constant gate
$\pi_\lambda = \sum_k \lambda_k \h \pi_k$ satisfies
$\KL(\h \sfp_\lambda \parallel \pi_\lambda) \leq \sum_{k = 1}^p
\lambda_k \e_k \le \max_k \e_k$.
\end{restatable}
See Appendix~\ref{app:fixed-mixture} for the proof. This proposition
guarantees average performance for a known mixture $\lambda$ using a
simple non-adaptive gate. (We derive the exact, though complex,
optimal gate for a fixed $\lambda$ in
Appendix~\ref{app:optimal-fixed-mixture}). However, any static weighting scheme is fundamentally limited by a capacity lower bound of
$\log(\sum_k e^{\e_k})$ for disjoint domains, as formalized below.

\begin{restatable}[Fundamental Capacity Lower Bound for Static Gating]{theorem}{LowerBound}
\label{thm:lower_bound}
Assume the datasets $D_k$ have mutually disjoint supports.  Consider
the class of static gating functions
$\sG_{\text{const}} \subset \sGone$, defined as gates where
$g(x, k) = w_k$ is independent of $x$ for all $k$.  For any static
gate $\bw \in \sG_{\text{const}}$, the worst-case Kullback-Leibler
divergence is lower-bounded by:
\begin{equation*}
  \max_{k \in \{1,\dots,p\}} \KL(\h\sfp_k \| \pi_\bw)
  \ge \log \left( \sum_{j=1}^p e^{\e_j} \right).
\end{equation*}
\end{restatable}
\begin{proof}[Proof Sketch]
  By the global normalization constraint, any static weight vector $\bw$ must lie on the simplex. Because the supports are mutually disjoint, the mixture density simplifies exactly to $\pi_\bw(x) = w_k \h \pi_k(x)$ on the support of task $k$. Evaluating the KL divergence yields $\KL(\h\sfp_k \| \pi_\bw) = \e_k - \log w_k$. Bounding this by the worst-case risk $\delta$ implies $w_k \ge e^{\e_k - \delta}$. Summing over all $p$ experts gives $1 \ge e^{-\delta} \sum_k e^{\e_k}$, which yields the stated bound. See Appendix~\ref{app:capacity_lower_bound} for full details.
\end{proof}

To overcome this lower bound barrier and achieve robustness
to \emph{unknown} $\lambda$, we require an input-dependent gate. 

Next, we show the existence of a gate model with a favorable guarantee
for any target mixture $\lambda$. We will use the \emph{linearized
  game}, which is defined by the payoff
$\wt L(\lambda, g) = \sum_{k=1}^p \lambda_k \KL(\h \sfp_k \parallel
\pi_g)$. Our analysis is presented in terms of the
\emph{Jensen-Shannon Divergence} (JSD), which is a measure of
diversity. For a mixture
$\h \sfp_\lambda = \sum_k \lambda_k \h \sfp_k$, JSD is defined as the
average KL divergence from each source to the mixture:
\[
  \JSD^\lambda(\curl*{\h \sfp_k})
  = \sum_{k = 1}^p \lambda_k \KL(\h \sfp_k \parallel \h \sfp_\lambda).
\]
JSD is non-negative and upper bounded by the Shannon entropy
$H(\lambda) = -\sum_{k = 1}^p \lambda_k \log \lambda_k$.

\begin{restatable}[Robust Existence]{theorem}{RobustExistence}
\label{th:robust-existence}
The linearized modular game admits a saddle point
$(\lambda^*, g^*) \in \Delta([1, p]) \times \sGone$. For any mixture
$\lambda \in \Delta$, the robust gate $g^*$ satisfies for \emph{any}
$\lambda \in \Delta([1, p])$:
\begin{equation*}
  \KL(\h \sfp_{\lambda} \parallel \pi_{g^*})
  \le \log \bracket*{\sum_{k=1}^p e^{\e_k}}
  - H^{\lambda^*}_{\sigma}(K|X) - \JSD^{\lambda}(\curl*{\h\sfp_k}),
\end{equation*}
where
$H^{\lambda^*}_{\sigma}(K|X) = \sum_{k = 1}^p \lambda_k^* \E_{x \sim
  \h\sfp_k}\bracket*{-\log \frac{\sigma_k \h
    \pi_k(x)}{\pi_\sigma(x)}}$ is the target-weighted conditional
entropy of the expert assignment under the robust constant gate
$\pi_{\sigma} = \sum_{k=1}^p \sigma_k \h \pi_k$ defined by the
softmax weights $\sigma_k = e^{\e_k}/\sum_{j=1}^p e^{\e_j}$.
\end{restatable}
\begin{proof}[Proof Sketch]
  The proof relies on constructing a linearized auxiliary game to
  satisfy Kakutani's fixed-point theorem conditions; See
  Appendix~\ref{app:robust-existence} for full details.
\end{proof}
Here, $H^{\lambda^*}_{\sigma}(K|X)$ can be viewed as the \emph{overlap
  gain} and $\JSD^{\lambda}$ as the \emph{diversity gain}.
The upper bound
($V^* \le \text{LSE} - \text{Overlap} - \text{Diversity}$) reveals how
the robust gate leverages task geometry in three limiting regimes (Appendix~\ref{app:geometric_interpretation}):

\textbf{Case 1: The Specialization Limit (Disjoint Experts).}
When supports are mutually disjoint, overlap vanishes
($H^{\lambda^*}_\sigma = 0$) and diversity is maximal
($\JSD^{\lambda} = H(\lambda)$). For balanced tasks, this
diversity gain cancels the capacity cost ($\log p$), guaranteeing
minimax robustness: $V^* \le \max_k \e_k$.

\textbf{Case 2: The Redundancy Limit (Identical Experts).}
When experts are identical with error $\e$, diversity vanishes
($\JSD=0$) but overlap is maximal ($H(\sigma) = \log p$). This overlap
gain exactly refunds the capacity cost ($\log(p e^\e) = \e + \log p$),
recovering the single-expert performance: $V^* \le \e$.

\textbf{Case 3: The Ensemble Mechanism (High Overlap).}
For overlapping experts with distinct errors, the overlap gain becomes
the entropy of the robust weights $H(\sigma)$. Substituting
$H(\sigma) = \log Z - \sum_k \sigma_k \e_k$ into the bound proves an
exact cancellation of the capacity cost (LogSumExp). The bound
collapses to the \emph{weighted average error}
$V^* \le \sum \sigma_k \e_k$, effectively acting as a static ensemble
to minimize risk.
  
\textbf{Prior Knowledge on Mixture Weights.}
In certain applications, we may have prior knowledge suggesting that the
mixture weights encountered at test time will be restricted to a
convex subset $\Lambda \subset \Delta([1, p])$. This knowledge can be
leveraged to derive a more specialized solution with significantly
stronger performance guarantees. Appendix~\ref{sec:prior-knowledge}
gives a detailed analysis of this setting and its benefits.

\textbf{The Least-Favorable Mixture.}
The optimization also yields the adversary's optimal strategy
$\lambda^*$ for the linearized game $\wt L$. This vector identifies
the specific mixture weights that maximize the weighted expert
loss. In practical scenarios where a modular architecture is
infeasible (e.g., extreme latency constraints), $\lambda^*$ provides a
statistically principled target distribution for training a single
static model, eliminating the need for heuristic weighting (see
Appendix~\ref{app:least-favorable}).

\subsection{Comparison with Monolithic Baselines: Interference
  vs. Decoupling}
\label{subsec:comparison_monolithic}

We rigorously contrast the proposed Modular Robustness with standard
Monolithic training by analyzing how each architecture interacts with
the geometry of the task distributions.  See
Appendix~\ref{app:theoretical_comparison} for a more detailed
comparison.

\textbf{The Monolithic Barrier: Diversity as Interference.}
Consider a monolithic model $\pi_{\text{mono}}$ trained to minimize
the loss on the mixture
$\h\sfp_\lambda = \sum_{k=1}^p \lambda_k \h\sfp_k$. The performance on
individual tasks is governed by the \emph{Jensen-Shannon Decomposition
  Identity}. For any model $\pi$, the average task risk decomposes
exactly into two terms:
\begin{equation*}
  \underbrace{\sum_{k=1}^p \lambda_k \KL(\h\sfp_k \parallel \pi)}_{\text{Average Task Risk}}
  = \underbrace{\KL(\h\sfp_\lambda \parallel \pi)}_{\text{Mixture Fit}}
  + \underbrace{\JSD^\lambda(\h\sfp_1, \dots, \h\sfp_p)}_{\text{Interference}}.
\end{equation*}
This equality reveals a fundamental limitation. Even in the limit of
infinite capacity where the model fits the global mixture perfectly
($\KL(\h\sfp_\lambda \parallel \pi) \to 0$), the average risk is
strictly determined by the task diversity, measured by the
Jensen-Shannon Divergence ($\JSD$).  We formalize this in Theorem~\ref{thm:jsd_gap} (Appendix~\ref{app:optimization_gap}). Since the average task risk is
lower-bounded by $\JSD^\lambda$, the worst-case domain loss
$\max_k \KL(\h\sfp_k \| \pi)$ is also necessarily lower-bounded by
this quantity. For a monolithic architecture, diversity manifests as
\emph{Geometric Interference}: the model is forced to collapse
distinct distributions into a single centroid. Consequently,
performance is dominated by the geometry of the problem rather than
the difficulty of the tasks; even if tasks are trivial to solve
individually ($\e_k \approx 0$), the model fails if they are distinct
($\JSD \gg 0$).

\textbf{The Modular Advantage: Diversity as Separability.}
In contrast, the modular gating network effectively inverts this
relationship. The worst-case risk of the robust gate is bounded by
(Theorem~\ref{th:robust-existence}):
\[
  \text{Risk}_{\text{mod}}
  \le \underbrace{\log \paren*{\sum e^{\e_k}}}_{\text{Capacity Cost}}
  - \underbrace{\JSD^{\lambda}(\h\sfp_1, \dots, \h\sfp_p)
  }_{\text{Diversity Gain}}
  - \underbrace{H^{\lambda^*}_{\sigma}(K|X)}_{\text{Overlap}}.
\]
Here, the divergence term appears with a \emph{negative} sign. For the
modular system, task diversity acts as a \emph{Geometric Gain}.  In
the high-diversity regime (disjoint supports), the overlap vanishes
($H \approx 0$). Crucially, if the test mixture is diverse, the
separability gain becomes maximal ($\JSD^{\lambda} \to H(\lambda)$).
As shown in our geometric analysis
(Section~\ref{sec:geometric_interpretation}), this gain effectively
cancels the entropy term in the capacity cost
($\log(\sum e^{\e_k}) \approx \e_{\max} + H(\lambda)$).  Consequently,
for diverse mixtures, the bound simplifies to the intrinsic error of
the worst-case expert:
$\KL(\h \sfp_\lambda \parallel \pi_{g^*}) \le \max_k \e_k$.

\textbf{The Decoupling Hypothesis.}  This analysis identifies a
structural phase transition governed by data geometry. We observe a
\emph{Symmetric Divergence Effect}: the same quantity $\JSD^\lambda$
that acts as an interference penalty for the monolithic model acts as
a separability bonus for the modular gate.  In the \emph{High
  Diversity Regime} (large $\JSD^\lambda$), the Monolithic model hits
an interference floor ($\text{Risk} \ge \JSD^\lambda$), forced to
increase entropy to cover disjoint supports (broad red curve in
Figure~\ref{fig:jsd_gap}).  In contrast, the Modular model exploits
this separation to cancel the capacity cost, maintaining the
low-entropy precision of the original experts (sharp blue peaks).  By
effectively \emph{decoupling} risk from geometry, the modular system
is guaranteed to outperform the monolithic baseline on \emph{any}
mixture $\lambda$ where the intrinsic task difficulty is lower than
the geometric cost of mixing:
$\max_{k} (\e_k) < \JSD^\lambda(\h\sfp)$.

\begin{figure}[t]
\centering
\scalebox{0.6}{
\includegraphics[scale = .5]{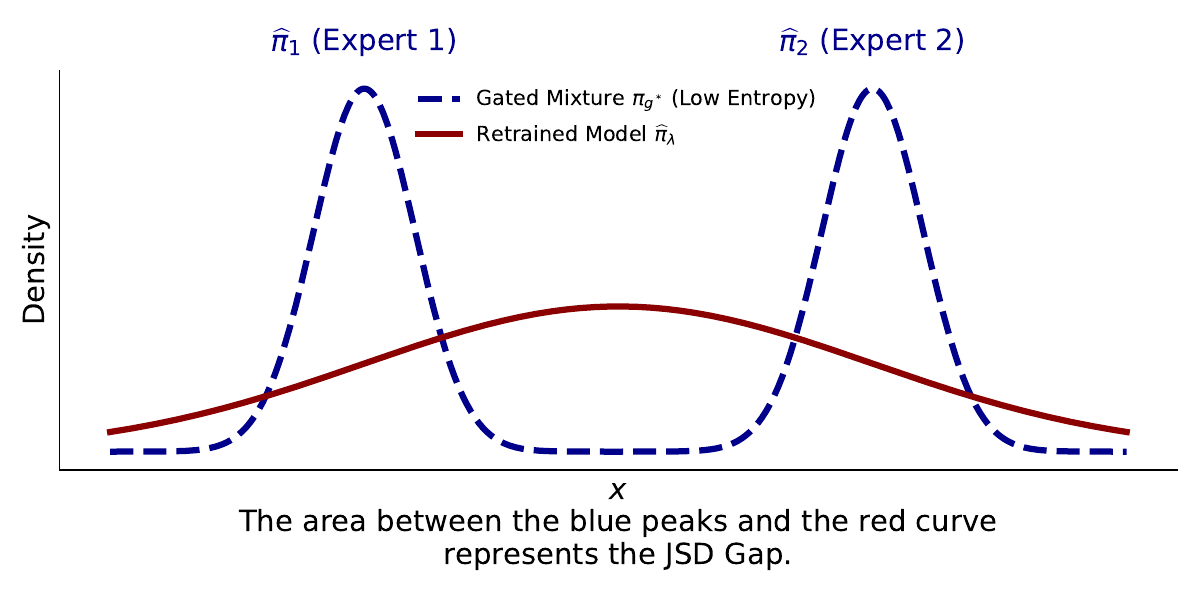}
}
\vspace{-10pt}
\caption{Visualizing the JSD Gap. A gated model (blue) fits distinct
  modes perfectly by routing inputs. A retrained model (red) suffers
  from capacity interference, forcing an entropy increase proportional
  to the JSD.}
\label{fig:jsd_gap}
\vspace{-3pt}
\end{figure}

\ignore{
\textbf{Safety in Convex Settings.}  Finally, one might ask if
modularity sacrifices performance in simpler settings. We prove in the following theorem
that for linear model families (e.g., exponential families), the
modular solution coincides exactly with the optimal retrained model:
$\pi_{g^*} = \h \pi_\lambda$. Thus, modularity acts as a ``safe''
architectural prior: it loses nothing in convex regimes while
providing strictly superior robustness guarantees in the presence of
conflicting, non-convex settings.
}

\textbf{Safety in Convex Settings.}  Finally, one might ask if
modularity sacrifices performance in simpler settings. We prove that in convex settings, the answer is no. We assume $\Pi$ is a linear model family (e.g., exponential families) in the sense of \citet{Csiszar1975}.

\begin{restatable}[Gated Model Coincides with Retraining]{theorem}{CommutingPythagorean}
\label{thm:linear_safety}
Let $\Pi$ be a linear model. For a mixture $\h \sfp_\lambda$, let $\pi_k = \pi^*(\h \sfp_k)$ denote the projection of each component. Then, the best model trained on the mixture $\pi^*(\h \sfp_\lambda)$ coincides exactly with the gated mixture of the best component models:
\[
 \pi^*(\h \sfp_\lambda) = \sum_{k = 1}^p \lambda_k \pi_k.
\]
\end{restatable}

\begin{proof}[Proof Sketch]
  The proof leverages the Pythagorean equality for linear models \citep{Csiszar1975}. By expanding the component divergences, we show that the aggregate mixture risk $\KL(\h \sfp_\lambda \parallel \pi)$ and the expected individual risk $\sum_k \lambda_k \KL(\h \sfp_k \parallel \pi)$ differ only by a constant (the Jensen-Shannon Divergence). Since they share the same unique minimizer, the optimal monolithic model exactly equals the convex combination of the individually projected experts. Full details are provided in Appendix~\ref{app:linear_safety}.
\end{proof}

Thus, modularity acts as a ``safe''
architectural prior: it loses nothing in convex regimes while
providing strictly superior robustness guarantees in the presence of
conflicting, non-convex settings.

\ignore{
\textbf{Rigorous Limits and Safety.}
We further substantiate these insights in
Appendix~\ref{app:analysis_retraining_gating}. Theorem~\ref{th:improvement}
provides a strict lower bound on the retrained model's error,
confirming that capacity interference acts as an irreducible
``confusion cost'' for any distinct tasks
(Proposition~\ref{prop:cond-equality}). Finally,
Theorem~\ref{th:commuting-pythagorean} ensures that modularity is a
safe architectural prior: in convex settings (e.g., exponential
families), the modular solution coincides exactly with the optimal
retrained model, sacrificing nothing in simpler regimes.
}

\subsection{Generalization and Sample Efficiency}
\label{sec:generalization_gap}

The guarantees in Theorem~\ref{th:robust-existence} establish the
existence of a robust gate on the empirical distributions $\h
\sfp_k$. A critical advantage of the modular framework is that this
robustness transfers efficiently to the true population distributions
$\sfp_k$.  The generalization gap of the gated model scales with the
complexity of the lightweight gating network $\sGone$, rather than the
massive complexity of the generative experts. To formalize this, we rely on the vector-valued Rademacher complexity of the gate class and also require expert log-likelihoods to be bounded (Assumption~\ref{ass:boundedness}). The formal assumption and definitions are provided in Appendix~\ref{app:generalization_gap}.

\begin{restatable}[Generalization Bound for Modular Gate Models]{theorem}
  {GeneralizationGate}
\label{th:generalization-gate}
Under Assumption~\ref{ass:boundedness}, for any
$\delta > 0$, with probability at least $1 - \delta$ over the draw of
samples $S_k \sim \sfp_k^m$, the following inequality holds
simultaneously for all $g \in \sGone$ and $\lambda \in \Delta$:
for the generalization gap
$\Gamma = \E_{x \sim \sfp_\lambda} \bracket*{ -\log \pi_g(x) } - \E_{x
  \sim \h \sfp_\lambda} \bracket*{ -\log \pi_g(x) }$: $
  \Gamma
  \leq 2\sqrt{2} C_\Pi \, e^M \Rad_m(\sGone) + M
  \sqrt{\frac{\log (p/\delta) }{2m}},$
where $\Rad_m(\sGone)$ is the Rademacher complexity of the gate class
and $C_\Pi = \sup_x \|(\h \pi_1(x), \ldots, \h \pi_p(x)) \|_2$ is the
\emph{expert coincidence norm}.
\end{restatable}

See Appendix~\ref{app:generalization_gap} for the proof using vector
contraction inequalities. The term $C_\Pi$ acts as a condition number
for modularity, measuring expert overlap: (1) Specialized Experts
(Ideal): If experts have disjoint supports,
$\|\h \bpi(x)\|_2 \approx 1$, so $C_\Pi \approx 1$.  (2) Redundant
Experts (Worst Case): If experts are identical, $C_\Pi = \sqrt{p}$.
Crucially, since the gate is typically a lightweight network (e.g., a
shallow Transformer) compared to the massive experts (LLMs),
$\Rad_m(\sGone) \ll \Rad_m(\Pi)$. This implies that the modular
approach requires significantly fewer samples to learn a robust policy
than retraining a monolithic model.

\ignore{
The guarantees in Theorem~\ref{th:robust-existence} establish the
existence of a robust gate on the empirical distributions $\h p_k$. A
critical advantage of the modular framework is that this robustness
transfers efficiently to the true population distributions $p_k$. As
proven in Appendix~\ref{app:generalization}
(Theorem~\ref{th:generalization-gate}), the generalization gap of the
gated model scales with the Rademacher complexity of the gating
network $\Rad_m(\sGone)$, rather than the much larger complexity of
the generative experts. Specifically, for any $\delta > 0$, with
probability at least $1 - \delta$, the robust loss generalization gap
$\Delta = \max_{\lambda} \E_{\sfp_\lambda}[-\log \pi_{g^*}] -
\max_{\lambda} \E_{\h \sfp_\lambda}[-\log \pi_{g^*}]$ satisfies
\[
  \Delta
  \leq  O\paren*{ C_\Pi \Rad_m(\sGone) } + O\paren*{ \sqrt{\frac{\log p}{m}} },
\]
where $C_\Pi = \sup_x \|\h \bpi(x)\|_2$ is the \emph{expert
  coincidence norm}, which measures overlap. For specialized experts
(disjoint supports), $C_\Pi \approx 1$.  Since the gate is typically a
lightweight network (e.g., a shallow Transformer or MLP) compared to
the massive experts (LLMs), $\Rad_m(\sGone) \ll \Rad_m(\Pi)$, this
implies that the modular approach requires significantly fewer samples
to learn a robust policy than retraining a monolithic model.
}

\section{Optimization Algorithms}
\label{sec:optimization-robust}

The existence result (\cref{th:robust-existence}) guarantees a robust
gate $g^*$ but is non-constructive. To compute this gate, we must
solve the minimax game. Originally, we formulated the problem as
$\min_{g \in \sGone} \max_{\lambda \in \Delta} L(\lambda, g)$ with
payoff $L(\lambda, g) = \KL(\h \sfp_\lambda \parallel \pi_g)$. This
payoff is convex in both parameters (since $\h \sfp_\lambda$ is linear
in $\lambda$ and $\KL$ is convex in first argument), preventing
the direct application of standard descent-ascent guarantees.  See
Appendix~\ref{app:optimization-robust} for a more detailed
discussion.

\subsection{Reformulation via Linearization}

To derive a tractable algorithm, we reformulate the problem into an
equivalent convex-concave game. Since the function
$\lambda \mapsto L(\lambda, g)$ is convex, its maximum over the
simplex $\Delta$ is always achieved at a vertex. Thus,
$\max_{\lambda} L(\lambda, g) = \max_k \KL(\h \sfp_k \parallel \pi_g)$.
This observation allows us to introduce a linearized payoff function:
\[
\wt L(\lambda, g) = \sum_{k=1}^p \lambda_k \KL\paren{ \h \sfp_k \parallel \pi_g }.
\]
This new game shares the same value as the original problem but is
\emph{convex-concave}: linear in $\lambda$ and convex in $g$. This
structure allows us to apply standard no-regret dynamics. Specifically,
if the $\lambda$-player uses Exponentiated Gradient and the $g$-player
uses Online Gradient Descent, the system is guaranteed to converge (Algorithm~\ref{alg:constrained-kl} in Appendix~\ref{app:linearization}).

\begin{restatable}[Convergence of Dynamics]{theorem}{ConvergenceKl}
\label{th:convergence-kl-short}
Let
$\wt V = \min_{g \in \sGone} \max_{\lambda \in \Delta} \wt L(\lambda,
g)$.  If the projection $\Pi_{\sGone}$ onto the normalized gate space
can be computed, then with step sizes
$\eta_\lambda \propto 1/\sqrt{T}$ and $\eta_g \propto 1/\sqrt{T}$, the
time-averaged gate $\ov g_T$ converges to the optimal robust solution:
\[
  \max_{\lambda \in \Delta} \wt L(\lambda, \ov g_T)
  - \wt V \leq O\paren*{\sqrt{\frac{\log p}{T}}}.
\]
\end{restatable}
See Appendix~\ref{app:convergence_analysis} for the proof. This theorem
provides a solid theoretical foundation: if we could enforce the
constraints exactly, we would provably find the robust gate. However,
the projection $\Pi_{\sGone}$ is computationally intractable for large
sequence models because the global normalization constraint
$Z_g = \sum_{x \in \sX_0} \pi_g(x) = 1$ couples the updates across all
inputs $x$ in the support.

\subsection{Scalable Primal-Dual Algorithm}

To scale to large generative models, we parameterize the gate
$g_\theta$ (e.g., as a Transformer Encoder) and enforce the global
constraint via Lagrangian relaxation. We introduce a dual variable
$\mu \in \Rset$ corresponding to the equality constraint $Z_g = 1$,
transforming the problem into a \emph{3-player primal-dual game}:
\[
  \min_{\theta} \max_{\lambda \in \Delta, \mu \in \Rset} \cL(\theta, \lambda, \mu)
  = \underbrace{\sum_{k=1}^p \lambda_k \mathcal{L}_{\text{NLL}}(k, \theta)}_{\text{Robust NLL}} + \underbrace{\mu (Z_{g_\theta} - 1)}_{\text{Penalty}}.
\]
The system simulates dynamics between
three players:

\textbf{1. $\lambda$-player (Adversary):} Maximizes the mixture
  difficulty using Exponentiated Gradient. This effectively upweights
  experts where the gate is currently underperforming.

\textbf{2. $\mu$-player (Constraint):} Performs Dual Ascent to
  enforce global normalization. If the total mass $Z_{g_\theta} > 1$,
  $\mu$ increases, penalizing the gate; if $Z_{g_\theta} < 1$, $\mu$
  decreases.

\textbf{3. $g$-player (Gate):} Updates parameters $\theta$ to
  minimize the Lagrangian via AdamW.

We solve this using stochastic estimates, (see
Algorithm~\ref{alg:primal-dual} in
Appendix~\ref{app:convergence_primal_dual}) and a detailed description
in Algorithm~\ref{alg:primal_dual_stochastic}
(Appendix~\ref{app:stochastic_primal_dual_algorithm}).

\subsection{Efficiency and Convergence}

The Primal-Dual formulation fundamentally alters the computational
profile of the problem, making it feasible for LLMs.

\textbf{Optimization Complexity.} The standard projection onto
$\sGone$ requires solving a constrained quadratic program over the
entire support $\sX_0$, which is impossible for sequence models. In
contrast, our constraint is enforced via a scalar update $\mu$,
costing $O(1)$ per parameter. The partition function $Z$ is estimated
efficiently using the training batch itself as the importance sampling
proposal, avoiding auxiliary data generation.

\textbf{Theoretical Guarantee.}
Crucially, replacing the hard projection with a Lagrangian penalty
does not sacrifice the convergence guarantee. The Primal-Dual dynamics
approximate the solution to the constrained game with the same
asymptotic rate.

\begin{restatable}[Convergence of Primal-Dual Dynamics]{theorem}
  {PrimalDualConvergence}
\label{th:primal-dual-convergence}
Consider the Lagrangian payoff
$\cL(g, \lambda, \mu) = \wt L(\lambda, g) + \mu(Z_g - 1)$. Under the
same convexity assumptions as Theorem~\ref{th:convergence-kl-short},
the time-averaged iterates $(\ov g_T, \ov \lambda_T)$ generated by
Algorithm~\ref{alg:primal-dual} converge to the optimal
robust solution with error $O(1/\sqrt{T})$, and the constraint
violation decays at rate $O(1/\sqrt{T})$:
\[
  \max_{\lambda} \wt L(\lambda, \ov g_T) - \wt V \leq O\paren*{\frac{1}{\sqrt{T}}}
  \quad \text{and} \quad |Z_{\ov g_T} - 1| \leq O\paren*{\frac{1}{\sqrt{T}}}.
\]
\end{restatable}

\begin{proof}[Proof sketch]
  The proof follows from standard regret analysis for Lagrangian games
  (e.g., \citep{cherukuri2017saddle}). The total regret of the
  primal-dual system decomposes into the regret of the
  $\lambda$-player (simplex), the $g$-player (convex set), and the
  $\mu$-player (unconstrained linear). Since all three use no-regret
  algorithms (Exponentiated Gradient, Adam/OGD, and Gradient Ascent),
  the average duality gap and the constraint violation are bounded by
  the average regret, which scales as $O(1/\sqrt{T})$.
\end{proof}
The proof details are given in  Appendix~\ref{app:primal-dual-proof}. This theorem 
ensures that even without the expensive projection step, the algorithm
provably recovers the robust modular gate.

\textbf{Gap between Theory and Practice.} The convergence guarantees
provided in \cref{th:convergence-kl-short} and
\cref{th:primal-dual-convergence} rely on the convexity of the
optimization problem with respect to the gate $g$. In our scalable
implementation, the gate $g_\theta$ is parameterized by a deep neural
network, rendering the objective non-convex with respect to
$\theta$. While strict no-regret guarantees do not apply to this
non-convex setting, we empirically observe that the Primal-Dual
algorithm converges to effective robust solutions, consistent with the
success of similar game-theoretic optimization dynamics in deep
learning (e.g., GANs or adversarial training).

\subsection{Practical Implementation}
\label{sec:practical-implementation}

Translating the theoretical algorithm into a stable training loop
requires addressing two specific numerical challenges: estimating the
global partition function $Z_g$ and avoiding underflow. See
Appendix~\ref{app:scalable_implementation} for a more detailed
discussion.

\textbf{Estimating the Partition Function.}  Calculating the global
sum $Z_g = \sum_{x \in \sX_0} \pi_g(x)$ exactly is intractable. We
rely on a Monte Carlo estimate using the current training batch
$B$. Crucially, to avoid expensive auxiliary sampling, we use the
training batch itself as the proposal distribution for Importance
Sampling. The batch is constructed by sampling uniformly from the $p$
source datasets, effectively drawing
$x \sim \frac{1}{p} \sum_{k} \h \sfp_k$. Under the assumption that
experts approximate their sources ($\h \pi_k \approx \h \sfp_k$), the
empirical mixture closely matches the model mixture proposal
$q(x) = \frac{1}{p} \sum_{k} \h \pi_k(x)$. The estimator becomes: $
  \h Z = \frac{1}{|B|} \sum_{x \in B} \frac{\pi_g(x)}{q(x)}$.
This allows us to reuse the logits computed during the forward pass,
estimating the global constraint with zero additional inference
cost. To reduce variance in the $\mu$-update, we track $\h Z$ using an
Exponential Moving Average (EMA).

\textbf{Log-Space Stability.}  The mixture probability
$\pi_g(x) = \sum_{k} g(x, k) \h \pi_k(x)$ involves summing
probabilities that may be extremely small (e.g., $10^{-100}$ for long
sequences). Direct computation leads to catastrophic underflow. We
strictly perform all operations in log-space using the LogSumExp
trick: $
  \log \pi_g(x) = \text{LogSumExp}_k \paren*{ \log g(x, k) + \log \h \pi_k(x) }$.

\textbf{Quadratic Penalty.}  An alternative to the
Lagrangian method is to relax the hard constraint into a soft
quadratic penalty,
$\min_{g} \max_{\lambda} L'(\lambda, g) + \beta (Z_g - 1)^2$. This
eliminates the need for the $\mu$-player, reducing the
problem to a standard regularized minimax optimization. However, it
only guarantees approximate normalization. We find the Primal-Dual
approach superior as it dynamically adjusts the penalty strength $\mu$
to satisfy the constraint exactly in the limit.

\section{Sampling from the Robust Gated Model}
\label{sec:sampling}

\begin{figure}[t]
\vskip -.08in
\centering
\scalebox{.5}{
\begin{tikzpicture}[
    node distance=1.5cm and 1cm,
    >=stealth',
    thick,
    box/.style={draw, rectangle, rounded corners, minimum width=2.5cm, minimum height=0.8cm, align=center, font=\small},
    expert/.style={box, fill=expertgreen, draw=green!40!black},
    gate/.style={box, fill=gateblue, draw=blue!40!black},
    student/.style={box, fill=studentred, draw=red!40!black},
    router/.style={box, fill=routerorange, draw=orange!40!black},
    process/.style={->, very thick, draw=black!70},
    loss/.style={->, dashed, thick, draw=red!70, label distance=2pt},
    cross/.style={draw=red, very thick, opacity=0.8}
]

    
    \node[font=\bfseries] (titleA) at (0, 0) {A. Monolithic Distillation};
    
    \node[gate, below=0.5cm of titleA] (t_gate_a) {Non-Causal\\Gate $g^*$};
    \node[expert, below=0.2cm of t_gate_a] (t_exp_a) {Frozen Experts\\$\{\widehat{\pi}_k\}$};
    
    \node[left=0.8cm of t_gate_a] (input_a) {$x$};
    \draw[process] (input_a) -- (t_gate_a);
    \draw[process] (input_a) |- (t_exp_a);
    
    \begin{scope}[on background layer]
        \node[draw=gray, dashed, fit=(t_gate_a) (t_exp_a), inner sep=0.2cm, label={right:\scriptsize \textbf{Teacher}}] (teacher_group_a) {};
    \end{scope}

    \draw[loss] (teacher_group_a.south) -- node[right, font=\scriptsize, align=left] {Train via $\mathcal{D}_{\text{robust}}$\\(Next-token pred.)} ++(0, -1.2) coordinate (mid_a);

    \node[student, below=1.2cm of teacher_group_a, minimum height=1.5cm, text width=3cm] (student_a) {
        \textbf{Causal Student}\\ $\pi_{\text{causal}}(x)$\\
        \scriptsize (All knowledge baked in)
    };
    
    \node[expert, right=0.5cm of student_a, opacity=0.4, scale=0.7] (dead_exp) {Experts};
    \draw[cross] (dead_exp.north west) -- (dead_exp.south east);
    \draw[cross] (dead_exp.north east) -- (dead_exp.south west);
    \node[below=0.1cm of dead_exp, font=\scriptsize, text=red] {Discarded};


    \node[font=\bfseries, right=2cm of titleA] (titleB) {B. Structural Transfer};
    
    \node[gate, below=0.5cm of titleB] (t_gate_b) {Non-Causal\\Gate $g^*$};
    \node[left=0.8cm of t_gate_b] (input_b) {$x$};
    \draw[process] (input_b) -- (t_gate_b);

    \draw[loss] (t_gate_b.south) -- node[right, font=\scriptsize, align=left] {Distill Routing\\Weights Only} ++(0, -1.2) coordinate (mid_b);

    \node[router, below=1.2cm of t_gate_b] (s_router) {\textbf{Causal Router}\\ $\gamma_\phi(x_{<t})$};

    \node[expert, below=0.6cm of s_router] (frozen_exp_b) {Original Frozen Experts\\$\{\widehat{\pi}_k\}$};

    \draw[process] (s_router) -- node[right, font=\scriptsize] {Weighted Mixing} (frozen_exp_b);
    \draw[process] (input_b |- s_router) -- (s_router);
    
    \draw[process] (input_b |- s_router) |- (frozen_exp_b);

    \begin{scope}[on background layer]
        \node[draw=blue!50, thick, rounded corners, fit=(s_router) (frozen_exp_b), inner sep=0.2cm, label={right:\scriptsize \textbf{Inference System}}] {};
    \end{scope}

    \draw[gray!30, thick] ($(titleA)!0.5!(titleB) + (0, 0.5)$) -- ($(titleA)!0.5!(titleB) + (0, -6)$);

\end{tikzpicture}
}
\caption{Efficiency Strategies. (A) Monolithic Distillation trains a
  single large model to mimic the ensemble, discarding the original
  experts. (B) Structural Distillation trains a lightweight Causal
  Router to mimic only the gating decisions ($g^*$), preserving the
  original experts. This maintains modularity: upgrading an expert in
  (B) improves the system immediately.}
\vskip -.1in
\label{fig:distillation}
\end{figure}

The optimization procedure yields a robust gate $g^* \in \sGone$ that
guarantees the mixture model
$\pi_{g^*}(x) = \sum_k g^*(x, k) \h \pi_k(x)$ is globally
normalized. However, sampling from this model presents a unique
challenge: the optimal gate $g^*(x, \cdot)$ is \emph{non-causal}. It
determines the mixture weights based on the \emph{complete} sequence
$x$, meaning the probability of the first token theoretically depends
on the last. This breaks the standard autoregressive property required
for efficient token-by-token generation. To sample from $\pi_{g^*}$, we must rely on methods that treat the
model as an unnormalized density or a re-weighted approximation. We
explore two Monte Carlo strategies:

\textbf{Sampling-Importance-Resampling (SIR).}  This method generates
samples asymptotically from the target distribution. We first draw $N$
candidate sequences $x^{(1)}, \dots, x^{(N)}$ from a proposal
distribution $q(x)$ that is easy to sample from—specifically, the
uniform mixture of experts $q(x) = \frac{1}{p} \sum_{k} \h
\pi_k(x)$. We then compute importance weights
$w_i = \pi_{g^*}(x^{(i)}) / q(x^{(i)})$. Finally, we resample a single
sequence $x^*$ from the candidates with probability proportional to
$w_i$. As $N \to \infty$, $x^*$ converges to a draw from $\pi_{g^*}$.
For full details, see Appendix~\ref{sec:SIR}.

\textbf{Exact Rejection Sampling.}  Unlike SIR, Rejection Sampling
provides \emph{exact} samples from $\pi_{g^*}$ for any finite
$N$. This requires an envelope constant $M$ such that
$\pi_{g^*}(x) \leq M q(x)$ for all $x$. A crucial property of our
normalized gate space $\sGone$ is that $g^*(x, k) \leq 1$, which
implies: $
  \pi_{g^*}(x) = \sum_{k=1}^p g^*(x, k) \h \pi_k(x)
  \leq \sum_{k=1}^p \h \pi_k(x) = p \cdot q(x)$.
Thus, we can use the number of experts $M=p$ as a strict envelope. The
algorithm samples a candidate $x \sim q(x)$ and accepts it with
probability $A(x) = \pi_{g^*}(x) / (p \cdot q(x))$. The acceptance
rate is $1/p$, meaning we must generate approximately $p$ candidates
to obtain one valid sample (see
Appendix~\ref{sec:rejection-sampling}).

\textbf{Inference Bottleneck.}  While these methods preserve the
theoretical robustness guarantees, their inference cost scales
linearly with the number of experts $p$. Evaluating the acceptance
probability or importance weight requires running a forward pass on
\emph{all} $p$ experts for every candidate sequence. For large
ensembles (e.g., $p=100$), this cost is prohibitive for real
applications, motivating the need for distillation.

\section{Efficient Inference: Structural Distillation}
\label{sec:distillation}

While the robust gate $g^*$ guarantees optimal performance, its
non-causal nature requires expensive sampling methods like SIR or
Rejection Sampling (Section~\ref{sec:sampling}) during inference. \ignore{ To
recover efficient $O(1)$ autoregressive generation while preserving
the benefits of modularity, we propose \emph{Structural Distillation}.}
With \emph{Structural Distillation} we recover efficient $O(1)$ autoregressive generation while preserving
the benefits of modularity.

\subsection{Monolithic vs. Structural Distillation}

Standard distillation would involve training a single large student
model to mimic the input-output behavior of the ensemble $\pi_{g^*}$
(see Appendix~\ref{sec:efficient-distillation}). While efficient at
inference time, this \emph{Monolithic Distillation}
(Figure~\ref{fig:distillation}(A)) discards the modular structure: if
one expert is updated or a new domain is added, the entire student
model must be retrained from scratch. In contrast, our \emph{Structural Distillation} approach
(Figure~\ref{fig:distillation}(B)) preserves the pre-trained experts
(see detailed analysis in
Appendix~\ref{app:realizability-structured-distillation}). We distill
the robust, non-causal gate $g^*$ into a lightweight \emph{Causal
  Router} $\gamma_\phi$. The inference system remains a mixture of
experts, but the routing decisions are now made causally.

\subsection{The Causal Router \& Objective}

We define the student model $\pi_\gamma$ as a causal mixture of the frozen experts, parameterized by a learnable router $\gamma_\phi$:
$
  \pi_\gamma(x) = \prod_{t=1}^T \pi_\gamma(x_t \mid x_{<t})
  = \prod_{t=1}^T \sum_{k=1}^p \gamma_\phi(x_{<t}, k) \, \h \pi_k(x_t \mid x_{<t}).$
Here, $\gamma_\phi(x_{<t}, \cdot) \in \Delta$ is a distribution over experts predicted by a small causal network (e.g., a shallow Transformer) given only the history. Our goal is to train $\phi$ to minimize the KL divergence from the robust teacher $\pi_{g^*}$ to the student $\pi_\gamma$ over the sequence space $\sX$:
\[
  \min_{\phi} \cJ(\phi) = \KL(\pi_{g^*} \parallel \pi_\gamma)
  = \E_{x \sim \pi_{g^*}} \bracket*{ \log \frac{\pi_{g^*}(x)}{\pi_\gamma(x)} }.
\]
Crucially, this global sequence-level objective decomposes into a tractable token-level optimization.

\begin{restatable}[Decomposition of Structural Distillation]{proposition}
  {DistillationDecomposition}
\label{prop:distillation-decomposition}
Minimizing the sequence-level divergence
$\KL(\pi_{g^*} \parallel \pi_\gamma)$ is equivalent to maximizing the
expected log-likelihood of the student model on trajectories sampled
from the robust teacher. Specifically, the gradient is:
\[
  \nabla_\phi \cJ(\phi)
  = - \mspace{-15mu} \E_{x \sim \pi_{g^*}} \mspace{-5mu} \bracket*{
    \sum_{t=1}^T \nabla_\phi \log \paren*{ \sum_{k=1}^p
      \gamma_\phi(x_{<t}, k) \, \h \pi_k(x_t \mid x_{<t}) } }.
\]
\end{restatable}
See Appendix~\ref{app:distillation-decomposition} for the proof. This
result allows us to train the router using standard MLE on a dataset
of ``robust sequences" generated by the teacher (using Rejection
Sampling). Furthermore, we prove in
Theorem~\ref{th:causal-decomposition}
(Appendix~\ref{app:realizability-structured-distillation}) that this
objective minimizes the \emph{Router Approximation Error}, with no
irreducible structural mismatch.

\subsection{Cached-Logit Distillation Algorithm}
\label{sec:structural-transfer}

A naive gradient update requires evaluating all $p$ experts at every
step. To avoid this bottleneck, we exploit the fact that experts are
frozen and propose a \emph{Cached-Logit} training loop
(Algorithm~\ref{alg:structural-distillation} in Appendix~\ref{app:structural-transfer}). First, we generate a
dataset $\cD$ from $\pi_{g^*}$ using Rejection Sampling or SIR, caching the
expert probability vectors
$\mathbf{P}_t = [\h \pi_1(x_t|x_{<t}), \dots, \h \pi_p(x_t|x_{<t})]$
for every token. While generating this robust synthetic dataset requires evaluating all experts, it is strictly a one-time, offline process that is highly parallelizable across multiple GPUs. As established by our generalization bound (\cref{th:generalization-gate}), the sample complexity scales with the lightweight Causal Router rather than the complex experts, meaning we only need to generate a relatively small dataset to distill the robust policy. Second, we train the router $\phi$ iteratively to maximize the
likelihood of these cached sequences by minimizing
$\mathcal{L} = -\log(\gamma_\phi(x_{<t}) \cdot \mathbf{P}_t)$. This
decouples the expensive expert evaluation (one-time cost) from router
training and bypasses the need for repeated, expensive forward passes through the experts during the distillation optimization loop, yielding a system that is robust, modular, and efficient.

\ignore{
  \subsection{Cached-Logit Distillation Algorithm}

A naive implementation of the gradient update would still require
running all $p$ experts to compute the term
$\h \pi_k(x_t \mid x_{<t})$. However, since the experts are
\emph{frozen}, we can decouple the expensive data generation from the
router training. We propose a \emph{Cached-Logit} training loop
(Algorithm~\ref{alg:structural-distillation} in
Appendix~\ref{app:structural-transfer}):
(Phase 1) Generation \& Caching. We use Rejection Sampling
(Section~\ref{sec:sampling}) to generate a fixed dataset $\cD$ of
sequences from the robust model $\pi_{g^*}$. Crucially, during this
generation, we compute and store the vector of expert probabilities
$\mathbf{P}_t = [\h \pi_1(x_t|x_{<t}), \dots, \h \pi_p(x_t|x_{<t})]$
for every token. This phase incurs a one-time cost proportional to
$p$.
(Phase 2) Router Training. We train the router parameters $\phi$
to maximize the likelihood of the cached sequences. The loss at step
$t$ becomes simply $-\log(\gamma_\phi(x_{<t}) \cdot
\mathbf{P}_t)$. This phase involves only the forward pass of the small
router and a dot product, making it extremely fast and independent of
the expert model size.

This structural approach yields a system that is robust (inheriting
properties from $g^*$), modular (experts remain frozen), and efficient
(inference cost is dominated by the top-$k$ experts selected by the
router).
}

\section{Experiments}
\label{sec:experiments}

\subsection{Empirical Comparison: Gate vs. Monolithic}
\label{sec:empirical-comparison}

In this section, we compare our gated model against standard retrained
baselines on synthetic data. Before presenting the results, we discuss
the nuances of a fair comparison.

\textbf{Fairness and Gradient Conflict.}
Comparing a modular architecture (frozen experts) with a monolithic
model (trained from scratch) is non-trivial. Standard metrics like
parameter count are insufficient. While a gated model might have a
larger \emph{total} parameter count, its effective hypothesis space is
constrained to the convex hull of the experts. A retrained monolithic
model theoretically enjoys greater flexibility, as it can move freely
in weight space to minimize aggregate loss. However, this flexibility comes at a cost: \emph{gradient
  conflict}. When source distributions contain conflicting signals
(e.g., distinct tasks or contradictory rules), a single model trained
on the aggregate objective suffers from destructive interference. The
optimization settles for a high-entropy \emph{compromise} that
underperforms on individual components. Our modular architecture
structurally orthogonalizes these conflicts. Therefore, we frame our
comparison not just on capacity, but on \emph{robustness to
  distribution shift}.

\textbf{Experimental Protocol.}
We define $p=2$ experts ($N_k$ parameters each), pre-trained to
convergence on source domains $D_k$. We evaluate three model classes:

\textbf{Robust Gate Model (Ours):} Combines frozen experts via a gate
  trained with the Primal-Dual Algorithm (Algorithm~\ref{alg:primal_dual_stochastic} in Appendix~\ref{app:stochastic_primal_dual_algorithm}) on the union dataset. The
  total size is $N_{\text{gate}} \approx 1.24 \sum N_k$ (only 24\%
  trainable parameters).

\textbf{Retrained Model (Fixed $\lambda$):} Monolithic models trained on the
  aggregate data (fixed $\lambda=0.5$). We evaluate a \emph{Smaller} version
  ($N = \sum N_k$) and a \emph{Larger} version ($N = 1.5 \sum N_k$) to
  test if increased capacity overcomes interference.

\textbf{Oracle Model:} ``Cheating" baselines retrained from scratch on
  the \emph{exact} test mixture $\lambda$ for every evaluation
  point. We again test Smaller and Larger variants.

\textbf{Synthetic Verification.}
We define a sequence task over vocabulary $V=100$ with two
contradictory domains. Domain A follows
$x_{t+1} = (x_t + 1) \pmod{100}$, while Domain B follows
$x_{t+1} = (x_t - 1) \pmod{100}$, with $x_0$ chosen uniformly at random. For any token $x_t$, the gradients
from A and B are directly opposed. We evaluate performance by varying
the test mixture $\lambda \in [0, 1]$ in steps of $0.1$.
For experimental details, see Appendix~\ref{app:implementation_details}.

\begin{figure}[t]
    
    \centering
    \includegraphics[width=0.49\linewidth]{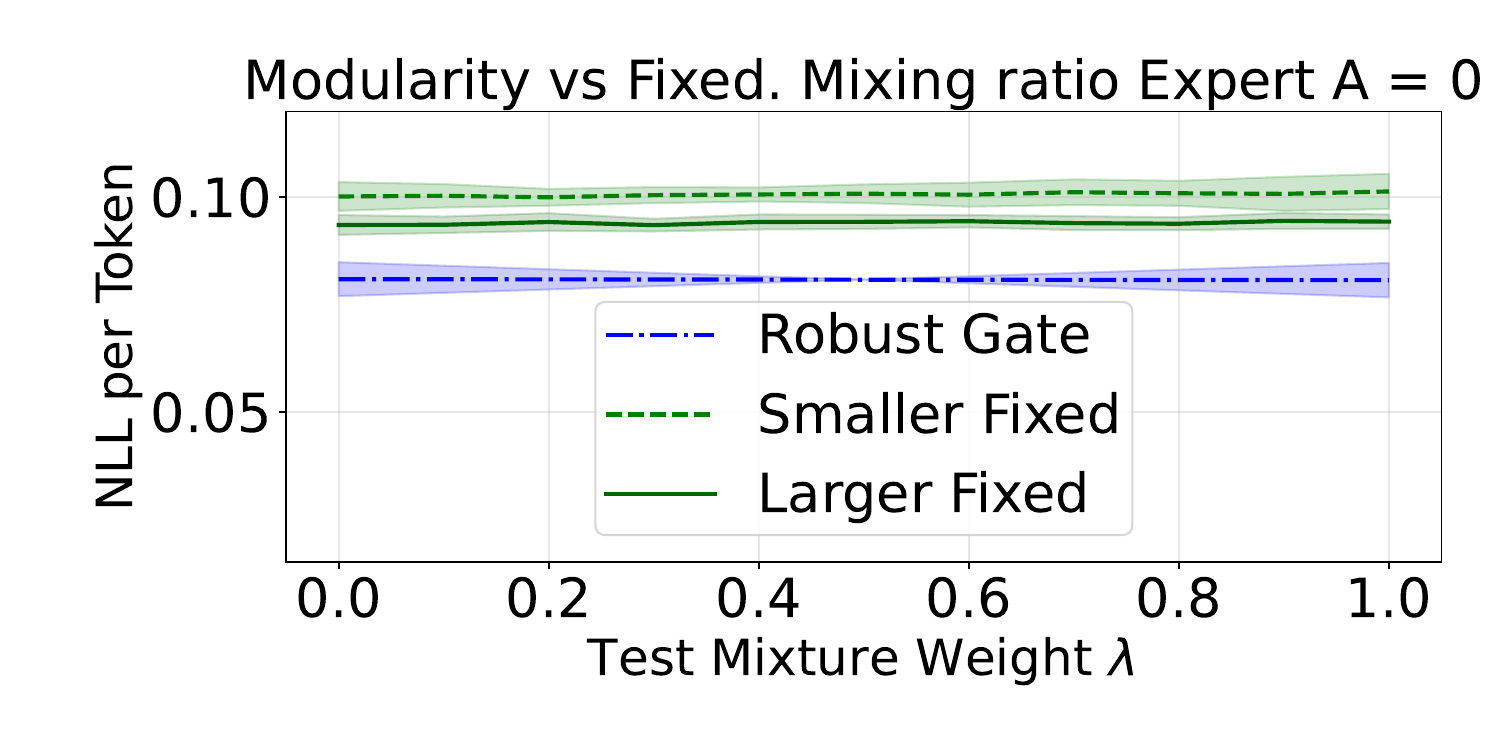}
    \includegraphics[width=0.49\linewidth]{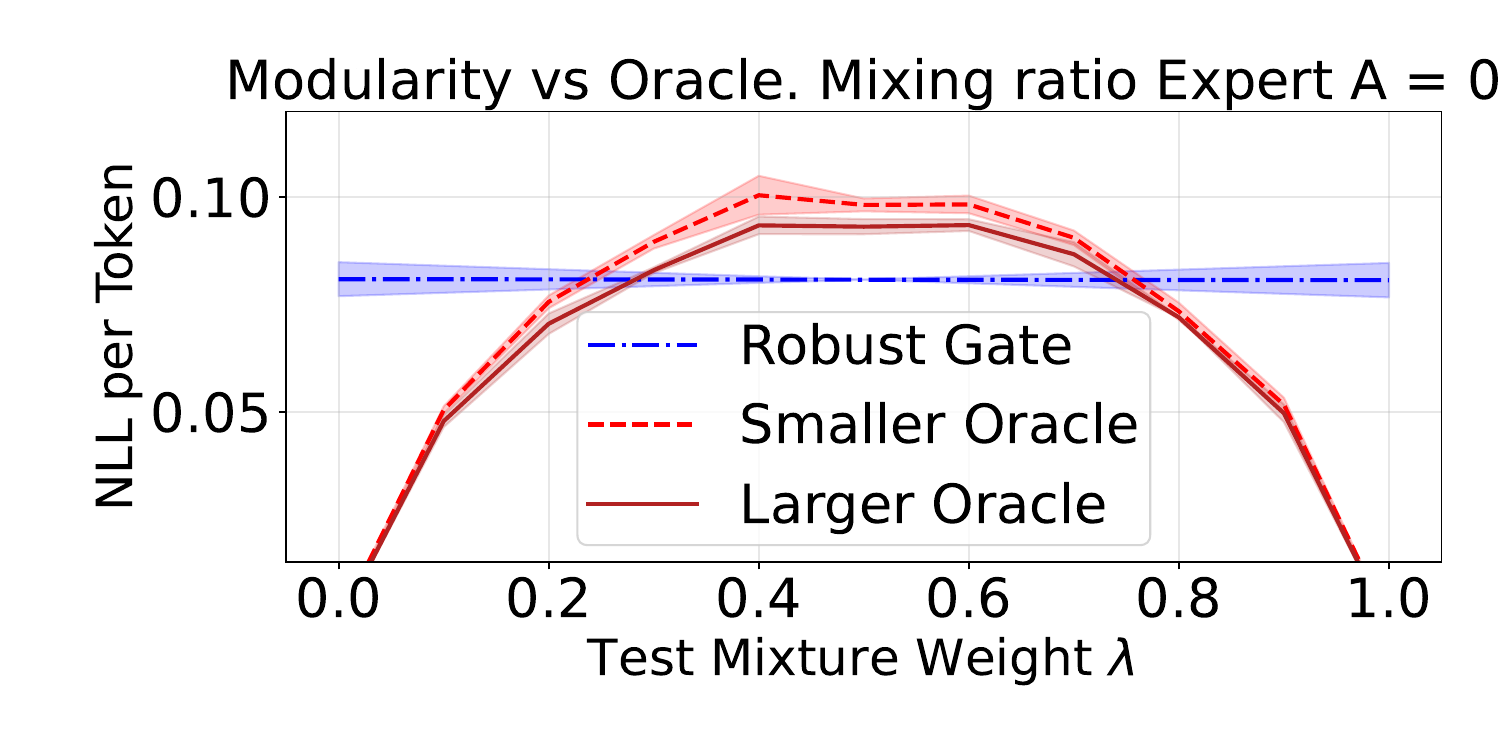}
    \vskip -.1in
    \caption{\textbf{Modularity overcomes gradient conflict.} Left:
      The Robust Gate (blue) outperforms Fixed monolithic baselines
      (green) across all mixtures. Right: Even compared to Oracle
      models (red) trained on the exact test distribution, the Gate
      wins in the high-interference region ($\lambda \approx
      0.5$). The concave shape of the Oracle curves indicates
      destructive interference.}
    \vskip -.1in
    \label{fig:capacity_gap}
\end{figure}

\textbf{Results: The Interference Gap.}
Figure~\ref{fig:capacity_gap} (Left) compares the Robust Gate against
the Fixed baselines. The Fixed models, trained on the conflict-heavy
mixture ($\lambda=0.5$), learn a high-entropy policy that fails to
specialize for either domain. The Gate achieves consistently
lower loss, confirming that modularity is superior to ERM when tasks
are disjoint.

Figure~\ref{fig:capacity_gap} (Right) reveals a more profound insight:
the \emph{Interference Gap}. In the high-entropy region
($\lambda \in [0.3, 0.7]$), the Gate outperforms even the
``cheating" Larger Oracle.  This empirically validates
Theorem~\ref{thm:jsd_gap} (The JSD Gap). The Oracle's performance
curve is distinctly concave: even with perfect knowledge of $\lambda$,
a single set of weights cannot simultaneously master contradictory
rules without increasing entropy (divergence). The modular system
avoids this penalty because the experts remain disjoint, and the gate
simply routes queries to the correct specialist.  At the extremes
($\lambda \approx 0$ or $1$), the task collapses to a single domain,
allowing the Oracles to specialize and naturally surpass the Gate.  Appendix~\ref{app:extra_experiments} provides additional results where the distributions of two experts A and B were less contradictory.

\subsection{Algorithm Stability and Convergence}

A key concern with minimax optimization is stability. We monitored the
dynamics of the Primal-Dual variables during training. The adversary's
mixture weights $\lambda_t$ rapidly converged to
$\lambda \approx [0.5, 0.5]$. This indicates that the gate
successfully balanced the performance across domains
($\e_A \approx \e_B$), reaching a maximum-entropy equilibrium where
the adversary has no incentive to concentrate on a specific task.
Simultaneously, the dual variable $\mu_t$, initialized at $0$,
increased steadily during the first epoch as the gate initialization
($Z \approx 1/p$) violated the constraint, before stabilizing once the
gate learned to satisfy the partition unity $Z_g \approx 1$.  The
system exhibited stable convergence without the oscillations typical
of adversarial training, likely due to the convexity of the inner
maximization over $\lambda$.

\subsection{Experiments with Structural Distillation}

\ignore{
\setlength{\intextsep}{0pt}
\setlength{\columnsep}{5pt}
\begin{wrapfigure}{r}{0.5\columnwidth}
  \centering
    \includegraphics[width=0.5\columnwidth]{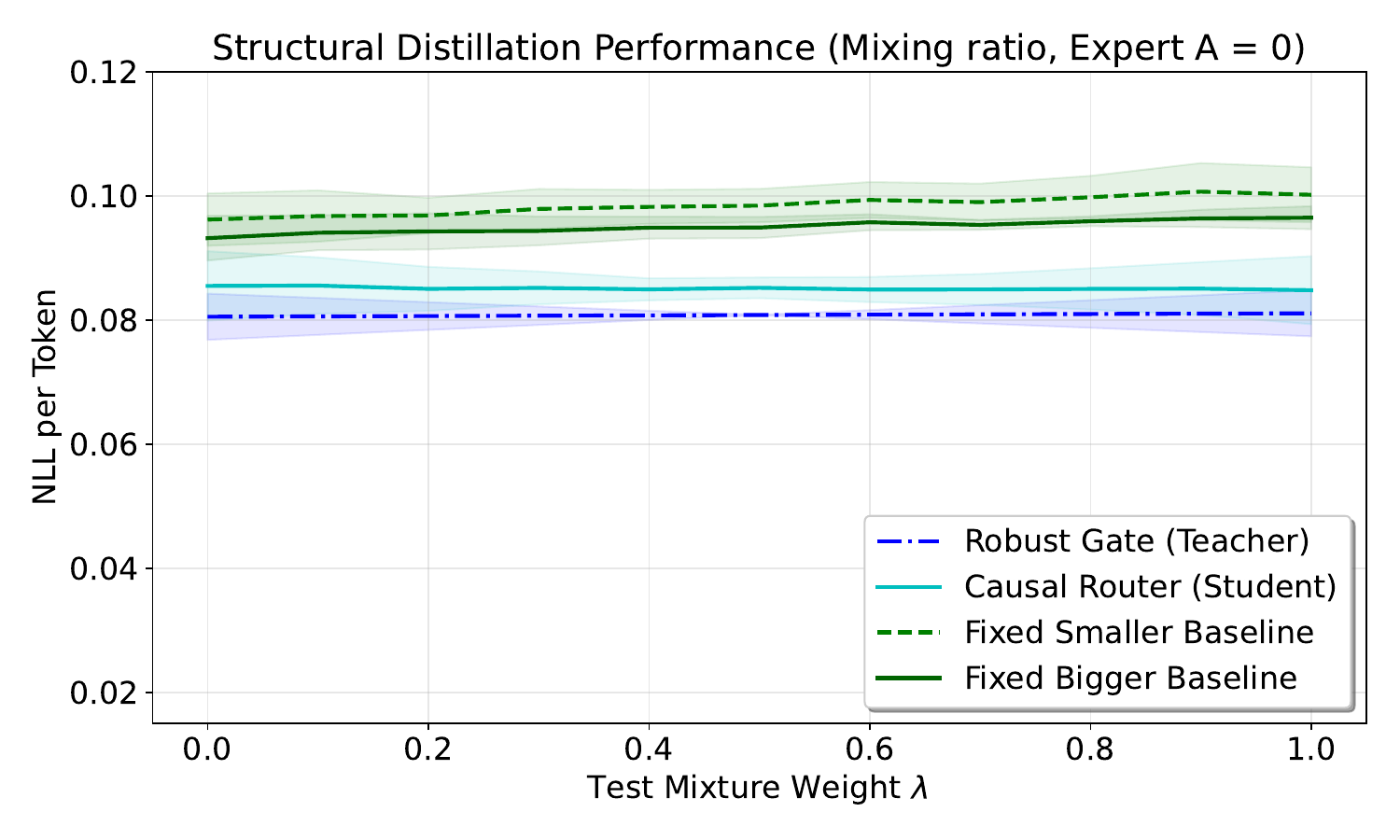}
  \vskip -5pt
    \caption{\textbf{Real-world robustness.} Relative test error (vs. Index 1 baseline) across sample sizes 10K--30K, averaged over 5 runs. The Robust Gate (orange) shows greater stability than the Retrained Model (blue).}
    \label{fig:causalgate}
    \vskip 0pt
\end{wrapfigure}
}

We also evaluate the Structured Distillation algorithm
(Section~\ref{sec:distillation}), which distills the robust mixture
$\pi_{g^*}$ into a \emph{Causal Router} $\gamma_\phi$.  We sampled
5,000 sequences from $\pi_{g^*}$ using rejection sampling. About 3\%
of these sequences contained inversions (switching between rules),
reflecting the non-trivial nature of the task. We trained a
causal Transformer router with $d=10$ (matching
the parameter count of the Larger Fixed Retrained baseline) on these
sequences. 

\begin{figure}[h]
    \centering
    \includegraphics[scale = .3]{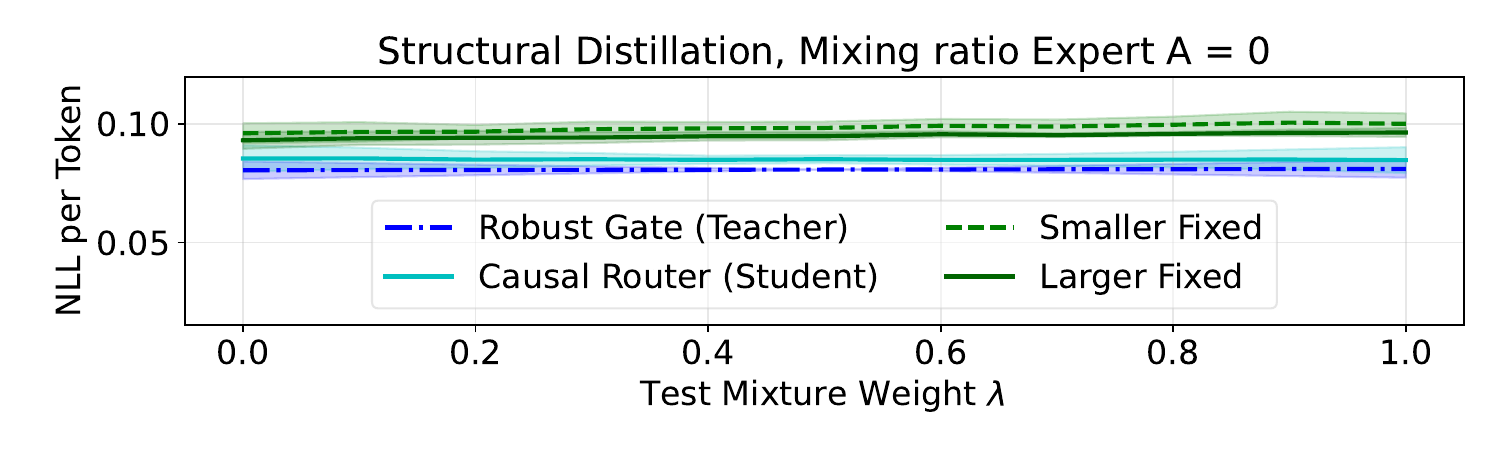}
    \vspace{-10pt}
    \caption{\textbf{Structured Distillation.} The Causal Router
      (cyan) effectively recovers the performance of the non-causal
      Robust Gate (blue),  outperforming the Larger Fixed
      baseline (green).}
    \label{fig:causalgate}
\end{figure}
As shown in Figure~\ref{fig:causalgate}, the distilled Causal Router
(cyan) performs nearly identically to the optimal non-causal Robust
Gate (blue), losing minimal performance despite the architectural
constraint.
 It significantly outperforms the monolithic Larger Fixed
model (green), demonstrating that we can transfer the robustness
benefits into an efficient autoregressive form.

\subsection{Modularity for Real-World Data and LLM Scaling}
\label{sec:experiments-real}

Finally, we experimented with three distinct HuggingFace datasets to
test real-world transfer:  (1) \texttt{\small{wikimedia/wikipedia}}:
High-quality factual prose; (2) \texttt{\small{bigcode/the-stack-smol}}:
Source code across 30+ languages; (3) \texttt{\small{fineweb-edu}}: Filtered
high-quality educational web content.
While Wikipedia and FineWeb share domain characteristics, the Code
dataset represents a significant distribution shift. Merging code
(strict syntax, high repetition) with natural language (ambiguous,
fluid) is known to cause negative transfer in monolithic models.

\textbf{Small-Scale Experiments.} We trained 3 experts \ignore{(each a 2-layer Transformer, 6.5M params)} and a
lightweight gate\ignore{(2-layer, 290K params)} with a  combined 
 $\sim$20M parameters. We compared this against a monolithic
Retrained model of matching size (19.8M, 3 layers). To address standard architectural baselines, we also evaluated a token-level Mixture-of-Experts (MoE) trained from scratch on the uniform mixture of datasets, using shared attention layers and a token-level routing layer over three expert Feed-Forward Networks, guided by a standard load-balancing loss, matching the capacity ($\sim$20M) of the combined modular framework. For additional experimental details, see Appendix~\ref{app:implementation_details_rw}.

\begin{table}[h]
  \vskip 0.02in
    \centering
    \caption{NLL per token on Real-World Data.}
    \vskip -.1in
    \label{tab:real_data}
    \resizebox{\columnwidth}{!}{
    \begin{tabular}{lcc}
        \toprule
        \textbf{Model} & \textbf{NLL (Diff seed)} & \textbf{NLL (Diff data)} \\ 
        \midrule
        Wiki Expert & 5.122 $\pm$ 0.005 & 5.118 $\pm$ 0.011 \\
        Code Expert & 4.722 $\pm$ 0.045 & 5.267 $\pm$ 0.788 \\
        FineWeb Expert & 5.623 $\pm$ 0.004 & 5.623 $\pm$ 0.006 \\
        \midrule
        Retrained Model & 5.133 $\pm$ 0.010 & 5.306 $\pm$ 0.257 \\
        Standard MoE (Scratch) & 5.080 $\pm$ 0.009 & 5.245 $\pm$ 0.269 \\
      \textbf{Gate Model (Ours)} & \textbf{4.994} $\pm$ 0.013
                                                  & \textbf{5.087} $\pm$ 0.141 \\
        \bottomrule
    \end{tabular}
  }
\end{table}

Table~\ref{tab:real_data} details the results (averages over 5 runs) across two settings:
varying initialization seeds and varying data splits. In both cases,
the Gate Model outperforms the Retrained model and the Standard MoE model. As theoretically predicted in Section~\ref{subsec:comparison_monolithic}, joint training over conflicting domains induces gradient interference even in standard MoE architectures due to their shared representation layers and joint optimization. In contrast, our frozen-expert routing approach effectively mitigated this conflict.

Regarding the domain-specific variance, the massive variance of 0.788 stemming from using different training data splits, the ``Diff data'' column, corresponds to the Code Expert (while the FineWeb Expert exhibited a very low variance of 0.006). This is empirically expected: code datasets are highly heterogeneous, with significant structural and syntactic differences across varying programming languages. Consequently, evaluating the Code Expert across different unseen splits naturally leads to a more volatile NLL. Crucially, despite this inherent instability in the monolithic code domain, our Gate Model robustly gates inputs across experts, smoothing out these domain-specific instabilities and dropping the total variance substantially to 0.141. Keeping the training data fixed but varying the seed for the model training, the ``Diff seed'' column, exhibits much smaller variations across models.

This result is
significant: even with a small-scale experiment (20M parameters), the
modular approach navigates the conflict between code and natural
language better than a monolithic model trained on the union. This
validates our hypothesis that structural modularity acts as a
regularizer against negative transfer in real-world deployments. Note that there is no guarantee that a modular model will always outperform a model retrained on all the data. As stated in Theorem~\ref{thm:jsd_gap}, the JSD gap governs the relative performance.

\textbf{Scaling to Massive LLMs (1B+ parameters).} To confirm that our theoretical framework scales to modern massive LLMs, we scaled our experiments to an ensemble of massive experts ($\sim$1.6B parameters each, GPT-2 XL-style configuration with 48 layers and 1600 embedding dimension) pre-trained on different domains. To ensure a strictly fair comparison, the retrained monolithic baseline model was symmetrically scaled to a massive 4.9B parameter architecture (48 layers, 2832 embedding dimension) to exactly match the combined capacity of the three experts plus the gate model. The Robust Gate, which is highly parameter-efficient at only $\sim$2.5M parameters, successfully learned to route across these massive experts without optimization instability, achieving a test NLL of 4.734 $\pm$ 0.006 compared to the monolithic baseline's 4.829 $\pm$ 0.012. This confirms that our framework successfully scales to massive LLMs, decoupling geometric interference and yielding a significant performance improvement of 0.095.

We further examined the stability of the models under varying test set compositions, see Table~\ref{tb:lambda} in Appendix~\ref{app:robustness_details}. The Gate model consistently exhibits lower loss and hence greater robustness to distribution shifts compared to the monolithic baseline. Details are provided in the appendix.

\section{Conclusion}

\ignore{We presented a game-theoretic framework for robust, modular generative
modeling. By optimizing over the space of normalized gates $\sGone$,
we derived a robust gate $g^*$ that guarantees a stable upper bound on
the worst-case risk. Our theoretical analysis identifies a structural
phase transition: while monolithic models incur interference costs
proportional to task diversity (Jensen-Shannon Divergence), our
modular approach decouples tasks, leveraging diversity to cancel
capacity costs. We further proved that modularity acts as a ``safe''
architectural prior, matching optimal retraining in convex regimes
while providing strictly superior guarantees for conflicting
distributions. Finally, we introduced a scalable Stochastic
Primal-Dual algorithm and Structural Distillation, empirically
validating their ability to mitigate gradient conflict on synthetic
and real-world datasets.
}

We presented a game-theoretic framework for robust generative modeling, deriving a gate $g^*$ with bounded worst-case risk. Our analysis identifies a phase transition: while monolithic models suffer interference proportional to the Jensen-Shannon Divergence, modularity decouples tasks to cancel capacity costs. We proved modularity acts as a ``safe'' prior, matching optimal retraining in convex regimes while superior for conflicting distributions. Finally, we validated our scalable Primal-Dual algorithm and Structural Distillation on synthetic and real-world datasets.

\clearpage
\newpage

\section*{Impact Statement}

This paper presents work whose goal is to advance the field of Machine
Learning. There are many potential societal consequences of our work, none
which we feel must be specifically highlighted here.

\bibliography{mod}
\bibliographystyle{icml2026}

\newpage
\appendix
\onecolumn

\renewcommand{\contentsname}{Contents of Appendix}
\tableofcontents
\addtocontents{toc}{\protect\setcounter{tocdepth}{3}} 
\clearpage

\section{Extended Content}

\subsection{Extended Related Work}
\label{app:related-work}

Our proposed framework for robust modularity intersects with several
active areas of research, including model composition, theoretical
routing, and the emerging economics of modular AI ecosystems.

\paragraph{Robustness and Multiple-Source Adaptation.}  Our approach
is rooted in the theory of multiple-source domain adaptation (MSA)
\citep*{MansourMohriRostamizadeh2008, MansourMohriRostamizadeh2012,
HoffmanMohriZhang2018, MohriHoffmanZhang2021, HoffmanMohriZhang2022,
CortesMohriSureshZhang2021}, which seeks to learn predictors robust to
mixtures of source domains. Recently,
\citet{DannMansourMarinovMohri2025} applied similar minimax principles
to the problem of model routing. Their work addresses
\emph{value-based} routing, where the goal is to maximize a scalar
reward (linear regret). Our work can be viewed as the generative
counterpart to this line of research. By moving from linear rewards to
the standard KL divergence objective which is needed for tackling
generative modeling, we face a fundamentally different mathematical
challenge: the resulting optimization problem is convex but
non-linear, and crucially, requires enforcing a global normalization
constraint ($Z_g=1$) on the mixture model. This necessitates the
constrained minimax analysis developed in this paper, distinguishing
our contribution from the unconstrained or locally-constrained
optimization found in value-based routing or standard MSA.
Our problem formulation also shares historical roots with the
\emph{Meta-Pi} network \citep{Waibel}, which combined expert outputs
via a gating mechanism. However, their primary goal was robustness for
the \emph{average} dataset, whereas we target the \emph{worst-case}
mixture. Furthermore, while they demonstrated source-independence
empirically, we provide rigorous existence and convergence guarantees.

\paragraph{Mixtures, Merging, and Composition.}
The concept of combining models has a rich history. \emph{Mixture of
  Experts (MoE)} \citep{Jacobs1991Adaptive, Fedus2021Switch} trains a
routing mechanism jointly with specialized sub-networks. In contrast,
our framework operates on \emph{frozen, pre-trained} experts,
decoupling the routing learning from the generative training. Another
approach is \emph{Model Merging} or \emph{Model Soups}
\citep{Wortsman2022}, which averages weights to find a
single high-performing static model. Our approach differs by
maintaining the experts as discrete entities and using an
input-dependent gate $g(x, \cdot)$ to adapt to distribution shifts
dynamically.

Recent work explores deeper architectural integration.  For example,
\citet{bansal2024llm} introduce \emph{Composition to Augment Language
  Models (CALM)}, which leverages cross-attention to merge
representations from a base LLM and specialized models, expanding
capabilities without full retraining. Distinct from a symmetric
modular view, this method designates one model as an \emph{anchor} and
the other as an \emph{augmenting} counterpart. It is also not clear
how this construction scales beyond two models, as it may require a
quadratic number of pairwise cross-attention connections.
Complementary to this is \emph{model stitching} \citep{jiang2024look},
where pre-trained blocks from disparate models, such as BERT and GPT,
are integrated directly. Similarly, recent frameworks like StitchLLM
\citep{pan2024stitchllm} dynamically route requests across stitched
blocks—for instance, feeding the lower layers of one model into the
upper layers of another—to optimize the trade-off between latency and
accuracy.
Crucially, neither approach provides theoretical analysis or
guarantees for the resulting composed model. In contrast, our approach
preserves experts as black boxes and offers strong theoretical
guarantees for a gating mechanism robust to worst-case distribution
mixtures.

\paragraph{Theoretical Routing and Learning to Defer.}
Our problem shares conceptual similarities with routing in
\emph{learning to defer}, where a learner chooses between predicting
or deferring to experts. Foundational work by
\citet*{CortesDeSalvoMohri2016, CortesDeSalvoMohri2024} established
the theory for learning with rejection in binary classification. This
line of work was significantly expanded to multi-class settings and
predictor-rejector frameworks by \citet{MaoMohriZhong2023c,MaoMohriZhong2024a, MaoMohriZhong2024b,MaoMohriZhong2024d,mao2024regression,MaoMohriZhong2025mastering,mao2025theory,desalvo2025budgeted,CortesMaoMohriZhong2026defid}, supported by \emph{$\sH$-consistency} guarantees \citep{awasthi2022h,awasthi2022multi,mao2023cross,MaoMohriZhong2023structured,MaoMohriZhong2023ranking,MaoMohriZhong2023characterization,MaoMohriZhong2023rankingabs,mao2024universal,mao2024h,mao2024multi,mao2024realizable,MohriAndorChoiCollinsMaoZhong2024learning,cortes2024cardinality,cortes2025balancing,cortes2025improved,mao2025enhanced,MaoMohriZhong2025principled,zhong2025fundamental,MohriZhong2026rllm,mohri2025beyond,MohriZhong2026slin,mohri2026principled,mohri2026generalized}. Our
approach diverges from this literature in three key aspects. First,
unlike standard routing which performs a hard selection of a single
expert, our gated framework induces a distribution over base
models. Second, rather than optimizing for average-case performance,
we address \emph{robustness} against adversarial distribution
mixtures. Finally, while computational cost is a primary consideration
in standard model routing, our current framework focuses purely on
statistical performance guarantees.

\paragraph{Modular Marketplaces and Ecosystems.}
Beyond functional integration, the rise of LLMs has spurred interest
in the economic dynamics of modular
systems. \citet{bhawalkar2025equilibria} analyze ``modular
marketplaces'' from a game-theoretic perspective, focusing on price
equilibria where module owners act strategically to maximize
profit. Broader analyses of the AI ecosystem's evolutionary dynamics
\citep{khansari2021evolutionary} further highlight how the interplay
between large upstream providers (e.g., cloud and foundation models)
and specialized downstream modules is fundamentally reshaping
industrial organization. Our work complements these economic and
ecosystem perspectives by providing the \textit{statistical}
equilibria---ensuring that the aggregated output of these traded
modules remains robust regardless of how they are combined.

\paragraph{Generalization Guarantees.}
We provide rigorous generalization bounds for the modular framework
using vector-valued Rademacher complexity
(Appendix~\ref{app:generalization_gap}). We show that the sample
complexity of learning the robust gate typically admits only a mild
logarithmic dependency on the \emph{number} of experts (provided they
are diverse), theoretically confirming that modularity acts as a
powerful regularizer against overfitting.

\subsection{Problem Formulation Details}
\label{sec:problem-formulation}

Given the setup in \cref{sec:setup}, our objective is to find a single
gating function $g \in \sGone$ such that the resulting model $\pi_g$
is a high-quality approximation of any data mixture $\h
\sfp_\lambda$. We use the relative entropy, $\KL$ divergence, as our
measure of dissimilarity. This leads to two primary formulations.
First, as a preliminary question, we can ask what performance is
achievable for a \emph{single, fixed} mixture $\lambda$. This
corresponds to the standard convex optimization problem:
\[
\min_{g \in \sGone} \KL(\h \sfp_\lambda \parallel \pi_g).
\]
Finding a solution to this problem would provide an optimal gate for a
known, static test distribution.
Second, and more central to our goal of modularity and robustness, we
ask for a \emph{single} gate $g^*$ that performs well against the
\emph{worst-case} mixture $\lambda \in \Delta$. This is a robust
optimization problem that can be formulated as a minimax game:
\[
\min_{g \in \sGone} \max_{\lambda \in \Delta([1, p])}
\KL(\h \sfp_\lambda \parallel \pi_g).
\]
The solution $g^*$ to this game would be a truly robust model,
providing a uniform performance guarantee across the entire ambiguity
set of possible data mixtures.

\noindent Our use of the relative entropy, rather than the
cross-entropy loss, is essential.  While minimizing KL is equivalent
to minimizing cross-entropy for a \emph{fixed} target distribution
(since entropy is constant), this equivalence breaks down in the
robust setting.  The entropy term $H(\h \sfp_\lambda)$ varies with the
adversarial choice of $\lambda$.  Consequently, minimizing the
worst-case cross-entropy
$\max_\lambda \E_{\h \sfp_\lambda}[-\log \pi]$ is not equivalent to
minimizing the worst-case divergence
$\max_\lambda \KL(\h \sfp_\lambda \| \pi)$.  We target the latter to
ensure the model approximates the distribution $\h \sfp_\lambda$
itself, rather than merely covering its support.

\ignore{
Our use of the relative entropy, rather than the cross-entropy loss, is
essential. The two metrics are closely
related:
\[
\KL(\sfp_\lambda \parallel \pi) = \underbrace{-\E_{x \sim
\sfp_\lambda}\bracket*{\log \pi(x)}}_{\text{Cross-entropy}} -
\underbrace{H(\sfp_\lambda)}_{\text{Entropy of } \sfp_\lambda}.
\]
When optimizing a model $\pi$ for a single, fixed distribution
$\sfp_\lambda$, the entropy term $H(\sfp_\lambda)$ is a constant. As a
result, minimizing the $\KL$ divergence with respect to $\pi$ is
equivalent to minimizing the cross-entropy. However, this equivalence
breaks down in our setting: We are not optimizing for one
$\sfp_\lambda$; we seek a single, robust solution $\pi$ that must
achieve a small $\KL$ divergence simultaneously with respect to
all different distributions $\sfp_\lambda$.
}

Our work aims to answer several fundamental theoretical and
algorithmic questions arising from these formulations:

\begin{enumerate}

\item Fixed-mixture performance: For a fixed mixture $\lambda$, does
  there exist a gated solution $\pi_g$ with small divergence
  $\KL(\h \sfp_\lambda \parallel \pi_g)$? More specifically, how does
  this optimal error
  $\min_{g \in \sGone} \KL(\h \sfp_\lambda \parallel \pi_g)$ compare
  to the baseline errors $\e_k$?
\item Robust guarantee: Does there exist a \emph{robust} gated
  solution $\pi_{g^*}$ that achieves small divergence for \emph{all}
  $\lambda \in \Delta$? This is a question of existence for the
  minimax problem
  $\min_g \max_\lambda \KL(\h \sfp_\lambda \parallel \pi_g)$.
\item Construction and bounds: If such a robust solution exists, how 
  can we construct it algorithmically? What explicit, non-asymptotic
  guarantees can we provide for its worst-case performance,
  $\max_{\lambda} \KL(\h \sfp_\lambda \parallel \pi_{g^*})$, in terms
  of the individual expert guarantees $\e_k$?
\item Comparison to aggregate training: For a fixed mixture $\lambda$, how does the performance of our
  modular solution $\pi_g$ compare to that of a model $\h \pi_\lambda$
  trained from scratch on the aggregate data
  $\h \sfp_\lambda = \sum_{k = 1}^p \lambda_k \h \sfp_k$? Understanding this trade-off is key to justifying the modular
  approach over the expensive, non-adaptive retraining baseline.
\end{enumerate}
In \cref{sec:theoretical-analysis}, we will address these questions,
starting with the existence and bounds for the fixed and robust
solutions.

\section{Fundamental Theory (Existence and Generalization)}

\subsection{Properties of the Gate Space}
\label{app:Gone-properties}

\GoneProperties*

\begin{proof}
  $\sGone$ is non-empty since for any $\lambda \in \Delta([1, p])$ it
  contains the constant gate $g_\lambda(x, k) = \lambda_k$.  $\sGone$
  is convex since $\sG$ is convex, as a product of simplices, and
  since the affine equality $Z_g = 1$ is preserved by convex
  combinations.  The base family
  $\sG = \prod_{x \in \sX_0} \Delta([1, p])$ is compact since each
  simplex $\Delta([1, p])$ is compact and the product of compact sets
  (over the finite support $\sX_0$, or even countable sets) is compact
  by Tychonoff's theorem. The constraint function
  $g \mapsto Z_g = \sum_{x, k} g(x, k) \h\pi_k(x)$ is continuous since
  linear. The set $\sGone = \sG \cap \curl*{g \colon Z_g = 1}$ is the
  intersection of a compact set, $\sG$, and a closed set, the level
  set $\curl*{g \colon Z_g - 1 = 0}$ of a continuous
  function. Therefore, $\sGone$ is closed. Since it is a closed subset
  of a compact set, it is also compact.
\end{proof}

\subsection{Fixed-Mixture Analysis}

\subsubsection{Fixed-Mixture Guarantee: Proof of
  Proposition~\ref{prop:fixed-mixture}}
\label{app:fixed-mixture}

\FixedMixture*

\begin{proof}
The result follows directly from the joint convexity of the KL divergence:
\[
  \KL(\h \sfp_\lambda \parallel \pi_\lambda)
  = \KL\paren*{\sum_{k = 1}^p \lambda_k \h \sfp_k
    \parallel \sum_{k = 1}^p  \lambda_k \h \pi_k}
  \leq \sum_{k = 1}^p \lambda_k \KL(\h \sfp_k \parallel \h \pi_k)
  \leq \sum_{k = 1}^p \lambda_k \e_k.
\]
This completes the proof.
\end{proof}

Note that we could in fact search for the optimal gate function $g^*$
and thus gated solution $\pi_{g^*}$, by solving
$g^* = \argmin_{g \in \sGone} \KL(\sfp_\lambda \parallel \pi_g)$,
which is a convex optimization problem since $\sGone$ is convex,
$\pi_g$ is linear in $g$, $\KL$ is jointly convex and composition with
a linear function preserves convexity. However, the resulting
guarantee for $\pi_{g^*}$ is complex and does not admit a simple
expression, as detailed in the following section.

\subsubsection{Fixed-Mixture Optimal Solution: Characterization}
\label{app:optimal-fixed-mixture}

For a fixed mixture weights vector $\lambda \in \Delta([1, p])$, we
consider the convex optimization problem of finding the best
normalized gated model:
\[
\min_{g \in \sGone} \KL(\h \sfp_\lambda \parallel \pi_g).
\]
The following lemma characterizes the unique optimal model
$\pi_{g^*}$.

\begin{figure}[H]
\centering
\includegraphics[scale = .5]{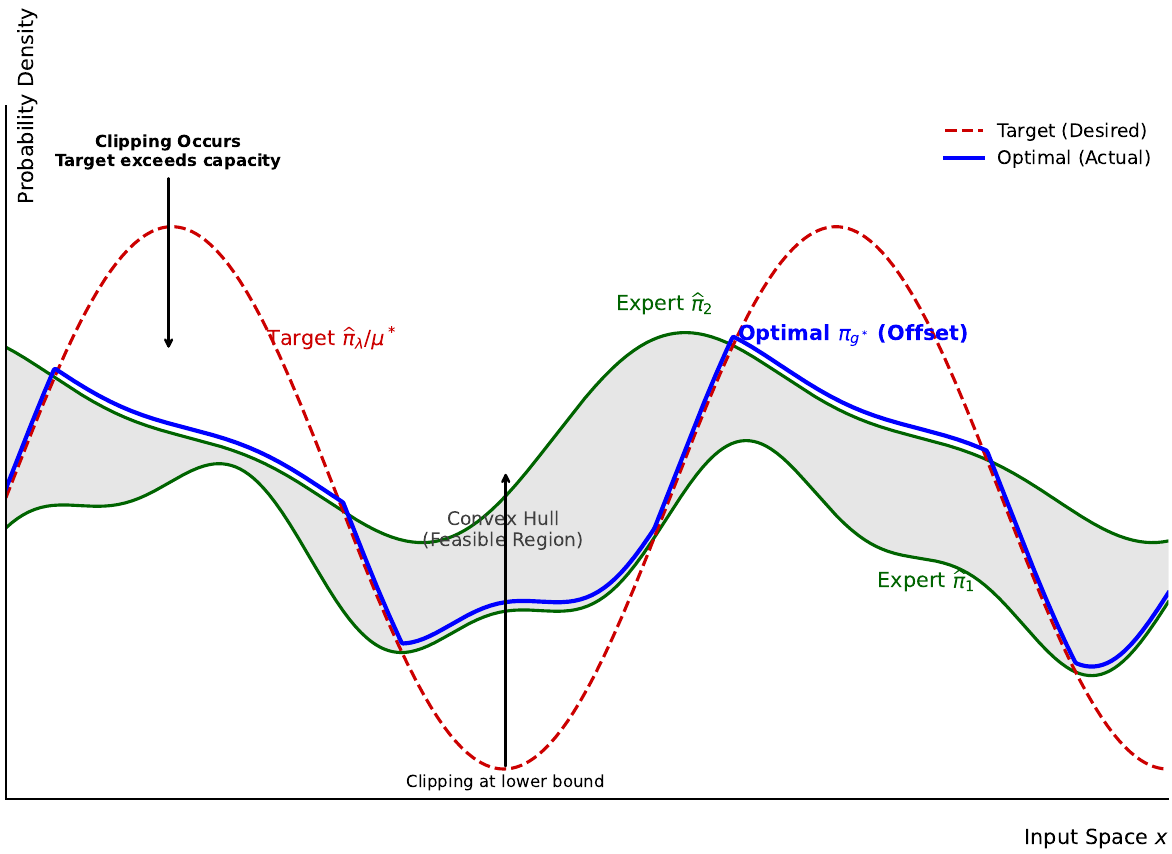}
\caption{Geometric Interpretation of
  Lemma~\ref{lemma:optimal-fixed-form}. The shaded gray
  region represents the convex hull of the experts $\h \pi_1$ and
  $\h \pi_2$. The target distribution (red dashed) exceeds this
  feasible capacity. The optimal gate $\pi_{g^*}$ (solid blue) traces
  the target where possible but is ``clipped" to the expert boundaries
  when the target falls outside the hull.}
\label{fig:clipping_geometry}
\end{figure}

\begin{lemma}[Structure of the Optimal Fixed-Mixture Model]
\label{lemma:optimal-fixed-form}
Let $\h \sfp_\lambda = \sum_k \lambda_k \h \sfp_k$. The optimal model
$\pi_{g^*}$ solving the minimization problem is unique and takes the
form of a clipped version of the mixture distribution
$\h \sfp_\lambda$. Specifically, there exists a unique scalar
$\mu^* > 0$ such that for all $x \in \sX_0$:
\[
  \pi_{g^*}(x)
  = \clip\paren*{ \frac{\h \sfp_\lambda(x)}{\mu^*},
    m(x),  M(x) },
\]
where $m(x) = \min_{k} \h \pi_k(x)$ and $M(x) = \max_{k} \h
\pi_k(x)$. The scalar $\mu^*$ is the unique solution to the
normalization equation $\sum_{x \in \sX_0} \pi_{g^*}(x) = 1$.
\end{lemma}

\begin{proof}
The optimization problem is:
\begin{align*}
  \text{minimize} \quad & \sum_{x \in \sX_0} \h \sfp_\lambda(x)
                          \log \frac{\h \sfp_\lambda(x)}{\pi_g(x)} \\
  \text{subject to} \quad & \pi_g(x) = \sum_{k = 1}^p g(x, k) \h \pi_k(x)
                            \quad \forall x, \\
& \sum_{k=1}^p g(x, k) = 1, \quad g(x, k) \ge 0 \quad \forall x, k, \\
& \sum_{x \in \sX_0} \pi_g(x) = 1.
\end{align*}
Minimizing the $\KL$ divergence is equivalent to maximizing the
expected log-likelihood: $\sum_x \h \sfp_\lambda(x) \log \pi_g(x)$.
The local constraints on $g(x, \cdot)$ imply that for any $x$, the
value $\pi_g(x)$ must lie in the convex hull of the expert predictions
$\curl*{\h \pi_1(x), \dots, \h \pi_p(x)}$. Since these are scalars,
the convex hull is simply the interval $[m(x), M(x)]$.  Thus, we can
reformulate the problem in terms of the model values
$\sfq_x = \pi_g(x)$:
\begin{align*}
\text{maximize} \quad & \sum_{x \in \sX_0} \h \sfp_\lambda(x) \log \sfq_x \\
\text{subject to} \quad & m(x) \leq \sfq_x \le M(x) \quad \forall x, \\
& \sum_{x \in \sX_0} \sfq_x = 1.
\end{align*}
This is a convex optimization problem. We introduce a Lagrange
multiplier $\mu$ for the global equality constraint
$\sum_x \sfq_x = 1$. The Lagrangian is:
\[
  \cL(\sfq, \mu)
  = \sum_{x \in \sX_0} \h \sfp_\lambda(x) \log \sfq_x
  - \mu \bracket*{\sum_{x \in \sX_0} \sfq_x - 1}.
\]
We solve this by maximizing $\cL$ with respect to $\sfq_x$ subject to
the local interval constraints. The problem decomposes for each $x$:
\[
  \max_{m(x) \leq \sfq_x \le M(x)} \quad \h \sfp_\lambda(x) \log
  \sfq_x - \mu \sfq_x.
\]
Let $f_x(q) = \h \sfp_\lambda(x) \log \sfq - \mu \sfq$. The derivative
is $f'_x(\sfq) = \frac{\h \sfp_\lambda(x)}{\sfq} - \mu$. Setting this
to zero gives the unconstrained optimum
$\sfq^* = \frac{\h \sfp_\lambda(x)}{\mu}$.  Since $f_x(\sfq)$ is
concave, the constrained optimum is the projection of the
unconstrained optimum onto the interval $[m(x), M(x)]$. This is
exactly the clipping operation (see
Figure~\ref{fig:clipping_geometry}):
\[
  \sfq_x^*(\mu)
  = \clip\paren*{ \frac{\h \sfp_\lambda(x)}{\mu}, m(x), M(x)}.
\]
The optimal $\mu^*$ is found by enforcing the global constraint
$Z(\mu) = \sum_x \sfq_x^*(\mu) = 1$. The function $Z(\mu)$ is
continuous and monotonically decreasing in $\mu$. Since
$\sum_x m(x) \leq 1$ and $\sum_x M(x) \geq 1$ (as each expert $\pi_k$
sums to 1), there exists a unique $\mu^*$ such that $Z(\mu^*) = 1$.
\end{proof}

\subsection{Capacity Lower Bound}
\label{app:capacity_lower_bound}

\LowerBound*

\begin{proof}
  Let $\bw = [w_1, \dots, w_p]$ be the weight vector. The global
  normalization constraint on $\sGone$ requires:
\begin{equation}
  \sum_{k=1}^p \int g(x, k) \h \pi_k(x) dx
  = \sum_{k=1}^p w_k \underbrace{\int \h \pi_k(x) dx}_{1}
  = \sum_{k=1}^p w_k = 1.
\end{equation}
Thus, the static weights must lie on the simplex.
Under the disjoint support assumption, for any $x \in \supp(D_k)$, the
other experts have zero density ($\h \pi_j(x)=0$ for $j \neq
k$). Thus, the mixture density simplifies exactly to:
\begin{equation}
\pi_\bw(x) = w_k \h \pi_k(x).
\end{equation}
We evaluate the KL divergence for task $k$ using this exact form:
\begin{align}
  \KL(\h\sfp_k \| \pi_\bw)
  & = \E_{x \sim \h\sfp_k}\bracket*{\log \left( \frac{\h\sfp_k(x)}{w_k \h \pi_k(x)} \right) }\\
  & = \underbrace{\E_{x \sim \h\sfp_k}\bracket*{ \log \left( \frac{\h\sfp_k(x)}{\h \pi_k(x)} \right) }}_{\e_k} - \E_{x \sim \h\sfp_k}\bracket*{ \log w_k } \\
  & = \e_k - \log w_k.
\end{align}
Let $\delta = \max_k \KL(\h\sfp_k \| \pi_\bw)$ be the worst-case
risk. Then for all $k$:
\begin{equation}
  \delta
  \ge \e_k - \log w_k \Rightarrow w_k \ge e^{\e_k - \delta}.
\end{equation}
Summing the weights over all $p$ experts:
\begin{equation}
  1 = \sum_{k=1}^p w_k
  \ge \sum_{k=1}^p e^{\e_k - \delta} = e^{-\delta} \sum_{k=1}^p e^{\e_k}.
\end{equation}
Rearranging to solve for the risk $\delta$:
\begin{equation}
  e^{\delta}
  \ge \sum_{k=1}^p e^{\e_k} \implies \delta
  \ge \log \left( \sum_{k=1}^p e^{\e_k} \right).
\end{equation}
This establishes that no static weighting scheme can surpass this
capacity limit.
\end{proof}

\subsection{Robust Existence Proof}
\label{app:robust-existence}

We define the \emph{linearized game} as the one with payoff:
$\wt L(\lambda, g) = \sum_{k=1}^p \lambda_k \KL(\h \sfp_k \parallel
\pi_g)$.

\ignore{
\begin{theorem}[Robust Existence and Geometric Upper Bound]
\label{thm:upper_bound}
The linearized modular game admits a saddle point
$(\lambda^*, g^*) \in \Delta([1, p]) \times \sGone$.  The risk of the
optimal gate $g^*$ satisfies, for \emph{any} task weighting
$\lambda \in \Delta([1, p])$:
\begin{equation}
  \KL(\h \sfp_{\lambda} \parallel \pi_{g^*})
  \le \underbrace{\log \left( \sum_{k=1}^p e^{\e_k} \right)}_{\text{LogSumExp}}
  - \underbrace{H^{\lambda^*}_{\sigma}(K|X)}_{\text{Overlap Gain}}
  - \underbrace{\JSD^{\lambda}(\h\sfp_1, \dots, \h\sfp_p)}_{\text{Diversity Gain}},
\end{equation}
where $\pi_{\sigma} = \sum_{k=1}^p \sigma_k \h \pi_k$ is the robust
constant gate defined by the softmax weights
$\sigma_k = e^{\e_k}/\sum_{j=1}^p e^{\e_j}$, and
$H^{\lambda^*}_{\sigma}(K|X) = \sum_{k = 1}^p \lambda_k^* \E_{x \sim
  \h\sfp_k}\bracket*{-\log \frac{\sigma_k \h
    \pi_k(x)}{\pi_\sigma(x)}}$ is the target-weighted conditional
entropy of the expert assignment under $\pi_{\sigma}$.
\end{theorem}
}

\RobustExistence*
\begin{proof}
  The proof consists of casting the problem as a two-player, zero-sum
  game and showing the existence of a saddle point via Kakutani's
  Fixed Point Theorem.
   
  The original payoff function is
  $L(\lambda, g) = \KL(\h \sfp_\lambda\parallel \pi_g)$. Since $L$ is
  convex in $\lambda$, its maximizers lie strictly at the vertices. To
  satisfy the convexity requirement of Kakutani's theorem, we consider
  the \emph{linearized game} with payoff:
  \[
    \wt L(\lambda, g)
    = \sum_{k=1}^p \lambda_k \KL(\h \sfp_k \parallel \pi_g).
  \]
  We define the best-response functions:
  \[
  \Lambda^*(g) = \argmax_{\lambda' \in \Delta([1, p])} \wt L(\lambda', g), \qquad
  G^*(\lambda) = \argmin_{g' \in \sGone} \wt L(\lambda, g').
  \]
  We show that the correspondence
  $T(\lambda, g) = (\Lambda^*(g), G^*(\lambda))$ satisfies Kakutani's
  conditions: For $G^*(\lambda)$, $\wt L$ is convex in $g$ (as a
  convex combination of convex KL terms), so the set of minimizers is
  convex.  For $\Lambda^*(g)$, $\wt L$ is linear in $\lambda$, so the
  set of maximizers is a convex face of the simplex.  Since $\wt L$ is
  continuous and the domains are compact, Berge's Maximum Theorem
  implies $\Lambda^*$ and $G^*$ have closed graphs.  Kakutani's
  fixed-point theorem thus guarantees the existence of a fixed point
  $(\lambda^*, g^*) \in (\Lambda^*(g^*), G^*(\lambda^*))$, which is a
  saddle point for $\wt L$.

  \textbf{Bounding the Value.} We now bound the worst-case risk of
  this optimal solution. By the saddle point property,
  $\wt L(\lambda^*, g^*) = \max_\lambda \wt L(\lambda, g^*)$. We use the
  fundamental identity relating the mixture risk $L$ and the
  linearized risk $\wt L$: for any $\lambda \in \Delta([1, p])$,
  \begin{align}
    \KL(\h \sfp_{\lambda} \parallel \pi_{g^*})
    & = \wt L(\lambda, g^*) - \JSD^{\lambda}(\h\sfp_1, \dots, \h\sfp_p) \nonumber \\
    & \leq \wt L(\lambda^*, g^*) - \JSD^{\lambda}(\h\sfp_1, \dots, \h\sfp_p).
      \label{eq:jsd_identity}
  \end{align}
  We must bound $\wt L(\lambda^*, g^*)$. Since $g^*$ is the minimizer
  of $\wt L(\lambda^*, \cdot)$ over $\sGone$, its loss is bounded by
  that of \emph{any} specific witness gate. We choose the
  Robust Constant Gate $\pi_{\sigma}$ defined by the softmax
  weights $\sigma_k = e^{\e_k}/Z$, where $Z = \sum e^{\e_j}$ (the solution of the proof of Theorem~\ref{thm:lower_bound}):
  \[
    \wt L(\lambda^*, g^*) \le \wt L(\lambda^*, \pi_{\sigma})
    = \sum_{k=1}^p \lambda_k^* \KL(\h \sfp_k \parallel \pi_{\sigma}).
  \]
  We expand the component KL divergence
  $\KL(\h \sfp_k \parallel \pi_{\sigma})$. Since
  $\pi_{\sigma}(x) \ge \sigma_k \h\pi_k(x)$:
  \begin{align*}
    \KL(\h \sfp_k \parallel \pi_{\sigma})
    &= \E_{x \sim \h\sfp_k} \left[ \log \frac{\h\sfp_k(x)}{\pi_{\sigma}(x)} \right] \\
    &= \E_{x \sim \h\sfp_k} \left[ \log \frac{\h\sfp_k(x)}{\sigma_k \h\pi_k(x)} + \log \frac{\sigma_k \h\pi_k(x)}{\pi_{\sigma}(x)} \right] \\
    &= \underbrace{\KL(\h\sfp_k \parallel \h\pi_k)}_{\e_k} - \log \sigma_k - \underbrace{\E_{x \sim \h\sfp_k}\left[ -\log \frac{\sigma_k \h\pi_k(x)}{\pi_{\sigma}(x)} \right]}_{H_k(K|x)}.
  \end{align*}
  Substituting $\sigma_k = e^{\e_k}/Z$, we have
  $-\log \sigma_k = \log Z - \e_k$. The $\e_k$ terms cancel:
  \[
    \KL(\h \sfp_k \parallel \pi_{\sigma}) = \log Z - H_k(K|x).
  \]
  Averaging over $\lambda^*$:
  \[
    \wt L(\lambda^*, g^*) \le \sum_{k=1}^p \lambda_k^* (\log Z - H_k(K|x))
    = \log Z - H^{\lambda^*}_{\sigma}(K|X).
  \]
  Substituting this upper bound back into Eq.~\eqref{eq:jsd_identity}
  completes the proof.
\end{proof}

\subsection{Geometric Interpretation of the Bound}
\label{app:geometric_interpretation}

The upper bound in Theorem~\ref{th:robust-existence} reveals how the
robust gate leverages the geometry of the task distributions. We
analyze the bound
$V^* \le \text{LSE} - \text{Overlap} - \text{Diversity}$ in three
limiting regimes.

\paragraph{Case 1: The Specialization Limit (Disjoint Experts).}
Consider the case where task supports are mutually disjoint
($\supp(\h\sfp_k) \cap \supp(\h\sfp_j) = \emptyset$).
\begin{itemize}
\item \textbf{Geometry:} The expert assignment is deterministic, so
  the overlap gain vanishes: $H^{\lambda^*}_{\sigma}(K|X) = 0$.

\item \textbf{Diversity:} The diversity gain depends on the test
  mixture $\lambda$. For disjoint supports, it is equal to the
  entropy: $\JSD^{\lambda} = H(\lambda)$.

\item \textbf{Result:} The bound becomes
  $\KL(\h \sfp_\lambda \parallel \pi_{g^*}) \le \log(\sum e^{\e_k}) -
  H(\lambda)$.  If the test mixture is difficult (high entropy, e.g.,
  balanced tasks), then $H(\lambda) \approx \log p$, canceling the
  capacity cost.  This guarantees that the modular model incurs no
  capacity penalty precisely when the task is most complex.
\end{itemize}

\paragraph{Case 2: The Redundancy Limit (Identical Experts).}
Consider the simplified case where all experts and targets are
identical and have equal error $\e$.
\begin{itemize}
\item \textbf{Geometry:} Tasks are indistinguishable ($\JSD=0$) and
  weights are uniform ($H(\sigma) = \log p$).

\item \textbf{Result:} The capacity cost $\log(p e^\e) = \e + \log p$
  is exactly refunded by the overlap gain ($\log p$).
    \[ V^* \le (\e + \log p) - \log p = \e. \]
    Thus, the modular system recovers the performance of a single expert.
\end{itemize}

\paragraph{Case 3: The Ensemble Mechanism (Exact Cancellation).}
We now analyze the general mechanism of overlap for arbitrary
errors. Consider the case where experts overlap fully (identical
targets, $\JSD=0$) but have different errors $\e_k$. The overlap gain
becomes the entropy of the robust weights $H(\sigma)$.  We recall that
for the robust gate, $H(\sigma) = \log Z - \sum_k \sigma_k
\e_k$. Substituting this into the bound:
\begin{align}
    V^* &\le \underbrace{\log Z}_{\text{Capacity Cost}} - \underbrace{\left( \log Z - \sum_{k=1}^p \sigma_k \e_k \right)}_{\text{Overlap Gain } H(\sigma)} - \underbrace{0}_{\text{Diversity}} \\
    V^* &\le \sum_{k=1}^p \sigma_k \e_k.
\end{align}
This derivation proves that in the high-overlap regime, the ``Capacity
Cost'' (LogSumExp) is \emph{exactly cancelled} by the ambiguity of
the gate. The bound collapses to the \emph{weighted average error}
of the experts. The modular system effectively transforms into a
static ensemble, pooling the experts to minimize risk.

\section{Extensions: Prior Knowledge}

\subsection{Guarantees with prior knowledge on mixture weights}
\label{sec:prior-knowledge}

In certain applications, we may have prior knowledge suggesting that the
mixture weights encountered at test time will be restricted to a
convex subset $\Lambda \subset \Delta([1, p])$. This knowledge can be
leveraged to derive a more specialized solution with significantly
stronger performance guarantees.

The existence result from Theorem~\ref{th:robust-existence} extends
directly to this scenario. By considering the linearized game restricted
to the compact convex set $\Lambda$, the convexity of the best-response
sets is preserved, ensuring the existence of a saddle point via Kakutani's
theorem.
The value of this restricted game, $V^*_\Lambda$, is guaranteed to be no
worse than the original game value, $V^*_\Delta$. The adversary
($\lambda$-player) has a smaller set of strategies, which limits their
ability to find high-loss mixtures. This results in a lower
worst-case loss for our solution:
\[
  V^*_{\Lambda}
  = \min_{g \in \sGone} \max_{\lambda \in \Lambda} L(\lambda, g)
  \leq \min_{g \in \sGone} \max_{\lambda \in \Delta([1, p])} L(\lambda, g) = V^*_\Delta.
\]
Intuitively, the gating function no longer needs to defend against
unrealistic mixture weights outside of $\Lambda$. It can therefore
specialize its performance for the known set of likely scenarios.

We now formalize the superiority of this specialized solution. Let
$g^*_\Delta$ be the optimal robust gate found by solving the original
problem over the full simplex $\Delta$, and let $g^*_\Lambda$ be the
optimal gate found by solving the restricted problem over $\Lambda$
(see Figure~\ref{fig:geometry_prior}).

\begin{figure}[t]
\centering
\begin{tikzpicture}[scale=0.8]
    \coordinate (A) at (0,0);
    \coordinate (B) at (6,0);
    \coordinate (C) at (3,5.2); 

    \draw[thick] (A) -- (B) -- (C) -- cycle;
    \node at (3, 5.5) {$\Delta([1,p])$ (Full Simplex)};
    
    \coordinate (L1) at (1.55,1.5);
    \coordinate (L2) at (4.45,1.5);
    \coordinate (L3) at (3,4);
    
    \filldraw[fill=green!10, draw=green!60!black, thick] (L1) -- (L2) -- (L3) -- cycle;
    \node[green!40!black] at (3, 2.2) {$\Lambda$ (Prior)};

    \coordinate (LambdaUnconstrained) at (4.5, 0.5);
    \fill[red!80!black] (LambdaUnconstrained) circle (2pt);
    \node[left, red!80!black] at (LambdaUnconstrained) {$\lambda_{t}'$ (Unconstrained)};

    \coordinate (LambdaProjected) at (4, 1.5); 
    \draw[->, dashed, thick, blue] (LambdaUnconstrained) -- (LambdaProjected);
    \fill[blue] (LambdaProjected) circle (2pt);
    \node[right, blue] at (4.25,1.3) {$\lambda_{t+1}$ (Projected)};

    \draw[<->, gray] (4.6, 0.5) -- (6, 0);
    \node[gray, font=\footnotesize, right] at (5.1, 0.4) {Safety Margin};

  \end{tikzpicture}
\caption{Geometry of Prior Knowledge
  (Section~\ref{sec:prior-knowledge}). The outer triangle represents
  the full probability simplex $\Delta$. The green region $\Lambda$
  represents the subset of valid mixtures defined by prior
  knowledge. The algorithm projects the adversary's updates (red
  point) back onto $\Lambda$ (blue point), tightening the worst-case
  bound as per Theorem~\ref{th:quant-improvement}.}
\label{fig:geometry_prior}
\end{figure}

\begin{restatable}[Dominance of the Specialized Solution]{theorem}{SpecializedDominance}
\label{th:specialized-dominance}
Let $g^*_\Delta$ be the optimal robust gate over the full simplex
$\Delta([1,p])$, and let $g^*_\Lambda$ be the optimal gate over the
restricted convex set $\Lambda \subset \Delta([1,p])$. Then the
worst-case performance of $g^*_\Lambda$ over $\Lambda$ is at least as
good as the performance of $g^*_\Delta$ over the same set:
\[
\max_{\lambda \in \Lambda} \KL\!\paren{\h \sfp_\lambda \parallel \pi_{g^*_\Lambda}}
\leq \max_{\lambda \in \Lambda} \KL\!\paren{\h \sfp_\lambda \parallel \pi_{g^*_\Delta}}.
\]
\end{restatable}

The proof is presented in Appendix~\ref{app:specialized-dominance}. To make the benefit more concrete, we can quantify the improvement
under a Lipschitz assumption, see proof in Appendix~\ref{app:quant-improvement}.

\begin{restatable}[Quantitative Improvement in Game Value]{theorem}{QuantImprovement}
\label{th:quant-improvement}
Assume that for any fixed gate $g \in \sGone$, the mapping
$\lambda \mapsto \KL(\h \sfp_\lambda \parallel \pi_g)$ is
$L$-Lipschitz with respect to the $\ell_1$-norm:
\[
\abs*{\KL(\h \sfp_\lambda \parallel \pi_g) - \KL(\h \sfp_{\lambda'} \parallel \pi_g)}
\leq L \norm{\lambda - \lambda'}_1,
\]
for all $\lambda, \lambda' \in \Delta([1, p])$. Let $V^*_\Delta = \min_{g} \max_{\lambda \in \Delta} \KL(\h \sfp_\lambda \parallel \pi_g)$ 
be the minimax value over the full simplex, and let 
$V^*_\Lambda = \min_{g} \max_{\lambda \in \Lambda} \KL(\h \sfp_\lambda \parallel \pi_g)$
be the value over the restricted set. The improvement is bounded by:
\[
0 \leq V^*_\Delta - V^*_\Lambda \leq L \cdot d_H(\Lambda, \Delta([1,p])),
\]
where
$d_H(\Lambda, \Delta([1, p])) = \max_{\lambda \in \Delta([1, p])}
\min_{\lambda' \in \Lambda} \norm{\lambda - \lambda'}_1$ is the
Hausdorff distance between the sets.
\end{restatable}

\paragraph{Explicit Lipschitz Constant.}
In practice, if the support $\sX_0$ is finite and all probabilities
are strictly positive, an explicit Lipschitz constant is
\[
  L = \max_{k \in [1, p]} \max_{g \in \sGone} \sum_{x \in \sX_0}
  \h \sfp_k(x) \abs*{\log \frac{\h \sfp_k(x)}{\pi_g(x)}}.
\]
This constant bounds the maximum gradient of the loss with respect to
the mixture weights $\lambda$. Crucially, it depends only on the extremal
geometry of the experts and is independent of the mixture $\lambda$.

\paragraph{Example.}
Consider a case with two experts: a high-quality model $\h \pi_1$
($\e_1 = 0.01$) and a low-quality model $\h \pi_2$ ($\e_2 = 0.5$).
The general solution $g^*_\Delta$ must be robust against the
worst-case mixture $\lambda = (0, 1)$, so its guaranteed performance
$V^*$ will be close to $0.5$. However, if we have prior knowledge that
the second source will never constitute more than 5\% of the mixture
(i.e., $\Lambda = \{\lambda \mid \lambda_2 \leq 0.05\}$), the
specialized solution $g^*_\Lambda$ can largely ignore this worst-case
scenario. Its guaranteed performance $V^*_\Lambda$ will be dramatically
lower, focusing on mixtures dominated by the high-quality expert.

\paragraph{Robustness to Mis-specification.}
The specialized gate is naturally robust to small mis-specifications
of the set $\Lambda$. The Lipschitz assumption allows us to bound the
performance on a slightly expanded set $\Lambda_\delta = \{\lambda \in
\Delta([1, p]) \mid \exists \lambda' \in \Lambda, \|\lambda -
\lambda'\|_1 \leq \delta\}$. The performance of the gate $g^*_\Lambda$
degrades gracefully:
\[
  \max_{\lambda \in \Lambda_\delta} L(\lambda, g^*_\Lambda) \leq V^*_\Lambda
  + L \, \delta.
\]

\paragraph{Algorithmic Adaptation.}
The optimization algorithm presented in Section~\ref{sec:optimization-robust}
is easily modified to adapt to this scenario. The update step for the
$\lambda$-player is simply augmented with a projection back onto the
convex set $\Lambda$. First, we compute the standard intermediate
update $\lambda'_{t+1}$:
\[
  \lambda'_{t+1}(k) = \frac{\lambda_t(k) \exp \paren*{\eta_\lambda \ell_t(k)}}
  {\sum_{j=1}^p \lambda_t(j) \exp \paren*{\eta_\lambda \ell_t(j)}}.
\]
Then, we project this distribution onto the restricted set $\Lambda$
to obtain the new weights:
\[
  \lambda_{t+1}
  = \proj_\Lambda (\lambda'_{t+1})
  = \argmin_{q \in \Lambda} \KL(q \parallel \lambda'_{t+1}).
\]
This ensures that the mixture weights always remain within the
specified prior constraints.

\subsection{Proofs for Prior Knowledge}

\subsubsection{Dominance of Specialized Solution}
\label{app:specialized-dominance}

\SpecializedDominance*

\begin{proof}
By definition of the restricted minimax solution $g^*_\Lambda$, we have
\[
\max_{\lambda \in \Lambda} \KL\!\paren{\h \sfp_\lambda \parallel \pi_{g^*_\Lambda}}
= \min_{g \in \sGone} \max_{\lambda \in \Lambda} \KL\!\paren{\h \sfp_\lambda \parallel \pi_g}.
\]
Since $g^*_\Delta \in \sGone$ is a feasible gate, its performance must be
greater than or equal to the minimum over all gates:
\[
\min_{g \in \sGone} \max_{\lambda \in \Lambda} \KL\!\paren{\h \sfp_\lambda \parallel \pi_g}
\le
\max_{\lambda \in \Lambda} \KL\!\paren{\h \sfp_\lambda \parallel \pi_{g^*_\Delta}}.
\]
Combining the two statements gives the claimed inequality.
\end{proof}

\subsubsection{Quantitative Improvement}
\label{app:quant-improvement}

\QuantImprovement*

\begin{proof}
  For any fixed gate $g$, define the worst-case loss over a set $S$ as 
  $F(g, S) = \max_{\lambda \in S} \KL(\h \sfp_\lambda \parallel \pi_g)$.
  
  Let $\lambda^* \in \Delta([1, p])$ be a mixture that achieves the maximum for the full simplex, 
  i.e., $\KL(\h \sfp_{\lambda^*} \parallel \pi_g) = F(g, \Delta)$.
  Let $\lambda_{proj}$ be the point in $\Lambda$ closest to $\lambda^*$ in the $\ell_1$-norm. 
  By the definition of the Hausdorff distance, 
  $\norm{\lambda^* - \lambda_{proj}}_1 \leq d_H(\Lambda, \Delta)$.
  
  Using the Lipschitz assumption:
  \begin{align*}
    F(g, \Delta) - F(g, \Lambda) 
    &= \KL(\h \sfp_{\lambda^*} \parallel \pi_g) - \max_{\lambda \in \Lambda} \KL(\h \sfp_\lambda \parallel \pi_g) \\
    &\leq \KL(\h \sfp_{\lambda^*} \parallel \pi_g) - \KL(\h \sfp_{\lambda_{proj}} \parallel \pi_g) 
    \tag{Since $\lambda_{proj} \in \Lambda$} \\
    &\leq L \norm{\lambda^* - \lambda_{proj}}_1 \\
    &\leq L \cdot d_H(\Lambda, \Delta([1, p])).
  \end{align*}
  This inequality holds for \emph{any} gate $g$. Therefore:
  \[
    F(g, \Delta) \leq F(g, \Lambda) + L \cdot d_H(\Lambda, \Delta).
  \]
  Taking the minimum over $g \in \sGone$ on both sides preserves the inequality:
  \[
    \min_g F(g, \Delta) \leq \min_g F(g, \Lambda) + L \cdot d_H(\Lambda, \Delta).
  \]
  Substituting the definitions $V^*_\Delta = \min_g F(g, \Delta)$ and $V^*_\Lambda = \min_g F(g, \Lambda)$ yields the upper bound.
  The lower bound $V^*_\Delta - V^*_\Lambda \geq 0$ follows immediately because $\Lambda \subset \Delta$, so the maximum over $\Lambda$ can never exceed the maximum over $\Delta$.
\end{proof}

\section{Theoretical Comparison: Modular Gating vs. Monolithic
  Retraining}
\label{app:theoretical_comparison}

In this appendix, we rigorously contrast our proposed Modular Robustness with standard Monolithic training (or retraining on a mixture). We analyze the trade-off from four perspectives:
\begin{enumerate}
    \item \textbf{Qualitative Scenarios (Section \ref{app:qualitative_scenarios}):} We illustrate the trade-off between specialization and flexibility with concrete examples.
    \item \textbf{Optimization Geometry (Section \ref{app:optimization_gap}):} We prove that monolithic models face a fundamental ``Interference Floor" determined by the Jensen-Shannon Divergence.
    \item \textbf{Architectural Safety (Section \ref{app:linear_safety}):} We show that for linear models, our approach coincides with the optimal retrained model.
    \item \textbf{Statistical Generalization (Section \ref{app:generalization_gap}):} We derive bounds showing that modular sample complexity scales with the lightweight gate.
\end{enumerate}

\subsection{Qualitative Scenarios: Retraining vs. Gating}
\label{app:qualitative_scenarios}

We compare the optimal gated model for a fixed $\lambda$,
$\pi_{g^*_\lambda} = \argmin_{\pi \in \Pi_G} \KL(\h \sfp_\lambda
\parallel \pi)$ (where $\Pi_G = \{\pi_g \mid g \in \sGone\}$), with
the optimal retrained model
$\h \pi_\lambda = \argmin_{\pi \in \Pi} \KL(\h \sfp_\lambda \parallel
\pi)$. Both approaches seek the best model for $\h \sfp_\lambda$
within their respective model classes ($\Pi_G$ for gating, $\Pi$ for
retraining).

For a fixed $\lambda$, both optimal solutions are, by definition, at
least as good as the simple, non-adaptive mixture
$\pi_\lambda = \sum_k \lambda_k \h \pi_k$ (which is a feasible point
in both $\sGone$ and $\Pi$, assuming $\Pi$ is closed under convex
combinations). Thus, by Proposition~\ref{prop:fixed-mixture}, both
optimal solutions have their performance upper-bounded by the same
quantity:
\begin{align*}
  \KL(\h \sfp_\lambda \parallel \pi_{g^*_\lambda})
  &\leq \KL(\h \sfp_\lambda \parallel \pi_\lambda)
  \leq \sum_k \lambda_k \e_k
\\
  \KL(\h \sfp_\lambda \parallel \h \pi_\lambda)
  &\leq \KL(\h \sfp_\lambda \parallel \pi_\lambda)
  \leq \sum_k \lambda_k \e_k
\end{align*}
The key difference lies in which model class, $\Pi_G$ (the
$x$-dependent mixtures of experts) or $\Pi$ (the original model
family), is better suited for approximating the specific mixture
distribution $\h \sfp_\lambda$. This leads to a trade-off between
specialization and flexibility, illustrated in the following
scenarios.

\begin{itemize}
\item \textbf{Scenario 1: Retraining Wins (Shuffling of the data may
    introduce new structures in the mixture).}
  The retrained model $\h \pi_\lambda$ will outperform the best gated
  model $\pi_{g^*_\lambda}$ when shuffling of the mixture distribution
  $\h \sfp_\lambda$ introduces significant new patterns or structures
  that are not present in the individual source distributions
  $\h \sfp_k$ and cannot be approximated by an $x$-dependent mixture
  of their expert models $\{\h \pi_k\}$.

  \textbf{Abstract Case:} Let $\Pi$ be a powerful, universal model
  class.  The optimal model for the mixture, $\h \pi_\lambda \in \Pi$,
  may be very accurate (low $\KL$). However, if this optimal
  $\h \pi_\lambda$ is fundamentally different from any model in
  $\Pi_G = \{\sum_k g(x, k) \h \pi_k \mid g \in \sGone\}$, then
  $\h \pi_\lambda \notin \Pi_G$. The best gated model
  $\pi_{g^*_\lambda}$ will be a poor approximation, and its error
  $V_1(\lambda) = \min_{\pi \in \Pi_G} \KL(\h \sfp_\lambda \parallel
  \pi)$ will be large.

  \textbf{Concrete Example:} Let $\h \sfp_1$ be text describing only
  cats, making $\h \pi_1$ a ``cat expert". Let $\h \sfp_2$ be text
  describing only dogs, making $\h \pi_2$ a ``dog expert". Let the
  mixture $\h \sfp_\lambda$ contain both $\h \sfp_1$ and $\h \sfp_2$,
  but from shuffling of the data, possibly also a new set of documents
  that \emph{compare} cats and dogs. The retrained model
  $\h \pi_\lambda \in \Pi$ sees these new comparison articles and
  learns the corresponding language patterns. The gated model
  $\pi_{g^*_\lambda} \in \Pi_G$, however, has no ``comparison expert"
  to choose from. It can only mix the cat expert and the dog expert,
  resulting in a poor approximation of the new comparison texts. In
  this case, we expect
  $\KL(\h \sfp_\lambda \parallel \h \pi_\lambda) < \KL(\h \sfp_\lambda
  \parallel \pi_{g^*_\lambda})$.

\item \textbf{Scenario 2: Gated Model Wins (Conflicting tasks).}
  The gated model $\pi_{g^*_\lambda}$ will outperform the retrained
  model $\h \pi_\lambda$ when the source distributions $\h \sfp_k$
  represent distinct or conflicting tasks that are ``easier" to learn
  separately than together, especially if the base model class $\Pi$
  has limited capacity or suffers from optimization issues like
  catastrophic forgetting.

  \textbf{Abstract Case:} Let $\h \sfp_1$ and $\h \sfp_2$ have largely
  disjoint supports. A retrained model $\h \pi_\lambda \in \Pi$
  trained on the mixture $\h \sfp_\lambda$ may find a poor compromise,
  modeling neither $\h \sfp_1$ nor $\h \sfp_2$ well (a phenomenon
  known as parameter interference or catastrophic forgetting). Its
  error $\KL(\h \sfp_\lambda \parallel \h \pi_\lambda)$ could be high,
  potentially $\gg \max(\e_1, \e_2)$. The optimal gate $g^*_\lambda$,
  however, can learn to be a perfect ``router": it sets
  $g^*(x, 1) \approx 1$ for $x \in \supp(\h \sfp_1)$ and
  $g^*(x, 2) \approx 1$ for $x \in \supp(\h \sfp_2)$. The resulting
  model $\pi_{g^*_\lambda}(x)$ effectively \emph{is} $\h \pi_1(x)$ on
  the first part of the data and $\h \pi_2(x)$ on the second. The
  performance
  $V_1(\lambda) = \KL(\h \sfp_\lambda \parallel \pi_{g^*_\lambda})$
  will be approximately the average of the individual expert errors:
  $V_1(\lambda) \approx \lambda_1 \KL(\h \sfp_1 \parallel \h \pi_1) +
  \lambda_2 \KL(\h \sfp_2 \parallel \h \pi_2) \le \sum_k \lambda_k
  \e_k \le \max_k \e_k$.

  \textbf{Concrete Example:} Let $\Pi$ be a transformer. Let
  $\h \sfp_1$ be an English$\to$French translation dataset ($\h \pi_1$
  is the expert, low $\e_1$) and $\h \sfp_2$ be English$\to$German
  ($\h \pi_2$ is the expert, low $\e_2$). Retraining $\h \pi_\lambda$
  on the 50/50 mixture may cause parameter interference, resulting in
  a model that confuses French and German vocabulary. The
  $x$-dependent gate $g^*_\lambda$ can learn to read the prompt (e.g.,
  ``Translate to French:") and set $g^*(x, 1) = 1$, perfectly routing
  to the correct expert. In this common scenario, we expect
  $\KL(\h \sfp_\lambda \parallel \pi_{g^*_\lambda}) < \KL(\h
  \sfp_\lambda \parallel \h \pi_\lambda)$.
\end{itemize}

\subsection{The Optimization Gap: Interference vs. Decoupling}
\label{app:optimization_gap}

\subsubsection{The Monolithic Barrier: Diversity as Interference}

Consider a monolithic model $\hat{\pi}_{\lambda}$ retrained to minimize the loss on the aggregate mixture $\h\sfp_\lambda = \sum_{k=1}^p \lambda_k \h\sfp_k$. While the model may achieve low training loss on the mixture distribution, its performance on individual tasks is governed by the \emph{Jensen-Shannon Decomposition Identity}. For any model $\pi$, the average task risk decomposes exactly into two terms:
\begin{equation}
    \underbrace{\sum_{k=1}^p \lambda_k \KL(\h\sfp_k \parallel \pi)}_{\text{Average Task Risk}} = \underbrace{\KL(\h\sfp_\lambda \parallel \pi)}_{\text{Mixture Fit}} + \underbrace{\JSD^\lambda(\h\sfp_1, \dots, \h\sfp_p)}_{\text{Interference}}.
\end{equation}
This equality reveals a fundamental limitation. Even in the limit of infinite capacity where the model fits the global mixture perfectly ($\KL(\h\sfp_\lambda \parallel \pi) \to 0$), the average risk is strictly lower-bounded by the task diversity. We formalize this in the following theorem.

\begin{restatable}[The JSD Gap]{theorem}{Improvement}
\label{thm:jsd_gap}
Let $\{\h \sfp_k\}_{k = 1}^p$ be source distributions and let $\e_k = \min_{\pi \in \Pi} \KL(\h \sfp_k \parallel \pi)$ be the best-in-class error for each source. Then, the risk of the optimal retrained model $\h \pi_\lambda$ satisfies:
\[
  \KL(\h \sfp_\lambda \parallel \h\pi_\lambda)
  \geq \sum_{k = 1}^p \lambda_k \e_k - \JSD^\lambda(\h \sfp_1, \ldots, \h \sfp_p).
\]
\end{restatable}

\begin{proof}
  The proof relies on a fundamental identity for the KL divergence of a mixture.
  For any model $\pi$ and letting $\mathcal{X} = \cup_k \supp(\h \sfp_k)$, the following equality holds:
\begin{align*}
\sum_{k = 1}^p \lambda_k \KL(\h \sfp_k \parallel \pi)
& = \sum_{k = 1}^p \lambda_k \sum_{x \in \mathcal{X}} \h \sfp_k (x) \log \frac{\h \sfp_k(x)}{\pi(x)} \\
& = \sum_{k = 1}^p \lambda_k \sum_{x \in \mathcal{X}} \h \sfp_k (x) \left( \log \frac{\h \sfp_\lambda(x)}{\pi(x)} + \log \frac{\h \sfp_k(x)}{\h \sfp_\lambda(x)} \right) \\
& = \sum_{x \in \mathcal{X}} \left( \sum_{k = 1}^p \lambda_k \h \sfp_k (x) \right) \log \frac{\h \sfp_\lambda(x)}{\pi(x)} + \sum_{k = 1}^p \lambda_k \left( \sum_{x \in \mathcal{X}} \h \sfp_k (x) \log \frac{\h \sfp_k(x)}{\h \sfp_\lambda(x)} \right) \\
& = \KL(\h \sfp_\lambda \parallel \pi) + \sum_{k = 1}^p \lambda_k \KL(\h \sfp_k \parallel \h \sfp_\lambda) \\
& = \KL(\h \sfp_\lambda \parallel \pi) + \JSD^\lambda(\h \sfp_1, \ldots, \h \sfp_p).
\end{align*}
Rearranging gives the identity:
\[
\KL(\h \sfp_\lambda \parallel \pi) = \sum_{k = 1}^p \lambda_k \KL(\h \sfp_k \parallel \pi) - \JSD^\lambda(\h \sfp_1, \ldots, \h \sfp_p).
\]
This holds for any $\pi \in \Pi$. We select the optimal model for the mixture, $\h\pi_\lambda$, and use the fact that $\KL(\h \sfp_k \parallel \h\pi_\lambda) \ge \e_k$ for all $k$:
\[
\KL(\h \sfp_\lambda \parallel \h\pi_\lambda) = \sum_{k = 1}^p \lambda_k \KL(\h \sfp_k \parallel \h\pi_\lambda) - \JSD^\lambda(\{\h \sfp_k\}) \ge \sum_{k = 1}^p \lambda_k \e_k - \JSD^\lambda(\{\h \sfp_k\}).
\]
This completes the proof.
\end{proof}

This theorem shows that the retrained model's performance is lower-bounded by the average best-case error $\sum_k \lambda_k \e_k$, offset by the $\JSD$ divergence, which measures the diversity of the sources. For a monolithic architecture, task diversity manifests as \textbf{Geometric Interference}. The model is forced to collapse distinct distributions into a single centroid. Consequently, as tasks become more dissimilar (increasing $\JSD$), the monolithic performance necessarily degrades.

In the theorem, we chose $\e_k$ to represent the smallest possible error. Alternatively, one can simply assume that the error of $\h \pi_\lambda$ on domain $k$ is never smaller than that of the specialist $\h \pi_k$. This is the only natural assumption required for the proof. Under this condition, the theorem establishes a fundamental limit on the retrained model: its performance is bounded below by the average specialist error minus a diversity term.

\subsubsection{Strictness of the Gap (Condition for Equality)}
One might ask if the retrained model can ever overcome this bound. We show that the inequality is strict unless the tasks are degenerate.

\begin{restatable}[Condition for Equality]{proposition}{CondEquality}
\label{prop:cond_equality}
The lower bound on the retrained model's error is met with equality if and only if the optimal model for the mixture, $\h \pi_\lambda$, is simultaneously an optimal model for every individual source distribution $\h\sfp_k$ for which $\lambda_k > 0$. That is, $\KL(\h \sfp_k \parallel \h\pi_\lambda) = \e_k$ for all $k$.
\end{restatable}

\begin{proof}
Assume first that $\KL(\h \sfp_k \parallel \h\pi_\lambda) = \e_k$ for
all $k$ with $\lambda_k > 0$. We start with the general identity for
the KL divergence of a mixture:
\[
\KL(\h \sfp_\lambda \parallel \h\pi_\lambda) = \sum_{k = 1}^p \lambda_k \KL(\h \sfp_k \parallel \h\pi_\lambda) - \JSD^\lambda(\h\sfp_1, \ldots, \h\sfp_p).
\]
Substituting our assumption into this identity directly yields the
desired equality:
\[
\KL(\h \sfp_\lambda \parallel \h\pi_\lambda) = \sum_{k = 1}^p \lambda_k \e_k - \JSD^\lambda(\h\sfp_1, \ldots, \h\sfp_p).
\]

Assume now that the equality holds:
\[
\KL(\h \sfp_\lambda \parallel \h\pi_\lambda) = \sum_{k = 1}^p \lambda_k \e_k - \JSD^\lambda(\h\sfp_1, \ldots, \h\sfp_p).
\]
By equating this assumption with the general identity from the first part, we get:
\[
\sum_{k = 1}^p \lambda_k \e_k - \JSD^\lambda(\{\h\sfp_k\}) = \sum_{k = 1}^p \lambda_k \KL(\h \sfp_k \parallel \h\pi_\lambda) - \JSD^\lambda(\{\h\sfp_k\}).
\]
Canceling the $-\JSD^\lambda(\{\h\sfp_k\})$ term and rearranging gives:
\[
\sum_{k = 1}^p \lambda_k \left( \KL(\h \sfp_k \parallel \h\pi_\lambda) - \e_k \right) = 0.
\]
By the definition of $\e_k$ as the minimum error, the term
$\KL(\h \sfp_k \parallel \h\pi_\lambda) - \e_k$ must be greater than
or equal to zero for all $k$. Since each $\lambda_k$ is also
non-negative, the expression is a sum of non-negative terms. Such a
sum can only equal zero if every term is individually zero. Therefore,
for every $k$ with $\lambda_k > 0$, it must be that:
\[
\KL(\h \sfp_k \parallel \h\pi_\lambda) - \e_k = 0,
\]
which implies $\KL(\h \sfp_k \parallel \h\pi_\lambda) = \e_k$. This
completes the proof.
\end{proof}

This implies that for equality to hold, the single model $\h\pi_\lambda$ must simultaneously be an optimal specialist for every individual source. This requires that the optimal specialist models themselves be identical ($\h \pi_1 = \dots = \h \pi_p$), a scenario plausible only if the source distributions are not meaningfully distinct. Since this represents a highly restrictive, degenerate condition, we can conclude that for any practical application involving distinct sources, the equality will not hold.

\subsubsection{Quantifying the Ambiguity Cost}

We now examine how the performance gap varies as a function of
how well-separated the source datasets are. Recall that the
Jensen-Shannon Divergence coincides with the mutual information:
\[
\JSD^\lambda(\h \sfp_1, \ldots, \h \sfp_p) = I(X; K) = H(\lambda) -
  H(K \mid X)
\]
where $H(K \mid X)$ is the conditional entropy of the source-index
random variable $K$ with $P(K = k) = \lambda_k$ given $X$, and where
$H(\lambda) = -\sum_{k = 1}^p \lambda_k \log \lambda_k$ is the entropy.
$H(K \mid X)$ measures the degree of \emph{overlap} or
\emph{ambiguity} between the mixture components. It quantifies the
expected uncertainty about which component $K$ generated a given
sample $X$. When the supports of the component distributions are
disjoint, the component index $K$ is fully determined by the
observation $X$, so $H(K \mid X) = 0$. As the supports begin to
overlap, certain points $x$ can be explained by multiple components,
creating ambiguity and a nonzero conditional entropy $H(K \mid X = x)$. In
this sense, $H(K \mid X)$ captures how much the mixture \emph{forgets}
about its origin; it represents the information lost about the latent
source $K$ after observing $X$. The more indistinguishable the
components become over their regions of overlap, the larger
$H(K \mid X)$ grows, up to a maximum when all $\h \sfp_k$ coincide and
$K$ is completely unidentifiable from $X$.

The story revealed by this analysis is not about a gap between two
static bounds, but rather the dynamic behavior of a \emph{single
  performance boundary} that separates the two modeling
approaches. This boundary, given by
$B(\lambda) = \sum_{k = 1}^p \lambda_k \e_k - \JSD^\lambda$, is not
fixed; its position is a direct function of the data's ambiguity,
$H(K \mid X)$. It varies from
$\sum_{k = 1}^p \lambda_k \e_k - H(\lambda)$ to
$\sum_{k = 1}^p \lambda_k \e_k$.
For perfectly separated sources where the ambiguity is zero, the
boundary is at its lowest
$\sum_{k = 1}^p \lambda_k \e_k - H(\lambda)$, guaranteeing a
significant performance gap. As the sources begin to overlap,
ambiguity grows, causing the JSD to shrink and the entire boundary to
\emph{rise} up to $\sum_{k = 1}^p \lambda_k \e_k$. This upward shift
quantifies the increasing ``cost of ambiguity'' for both models,
culminating in the boundary reaching its highest point when the
sources are identical and the advantage of gating vanishes.

To make this relationship more concrete and locate a specific dataset
along this trajectory, we now provide an explicit lower bound on the
overlap ambiguity term $H(K \mid X)$ inside the JSD.

\begin{restatable}[Bounds on Overlap Ambiguity]{proposition}{BoundsAmbiguity}
\label{prop:bounds_ambiguity}
The overlap ambiguity term $H(K \mid X)$ admits the following lower bound:
\[
  H(K \mid X) \geq 1 - \sum_{k = 1}^p \lambda_k^2 \sum_x \frac{[\h \sfp_k(x)]^2}{\h \sfp_\lambda(x)}.
\]
\end{restatable}

\begin{proof}
The conditional entropy $H(K \mid X)$ can be expressed as
follows:
\begin{align*}
H(K \mid X)
& = -\sum_x P_X(x)\sum_{k = 1}^p P(K = k \mid X = x) \log P(K = k \mid X = x) \\
& = -\sum_x \h \sfp_\lambda(x)\sum_{k = 1}^p
\frac{\lambda_k\h \sfp_k(x)}{\h \sfp_\lambda(x)}
\log\frac{\lambda_k\h \sfp_k(x)}{\h \sfp_\lambda(x)} \\
& = \sum_{k = 1}^p \sum_x \lambda_k\h \sfp_k(x)
\log\frac{\h \sfp_\lambda(x)}{\lambda_k\h \sfp_k(x)}.
\end{align*}
Thus, we have
$H(K \mid X) = \sum_{k} \lambda_k \sum_x \h \sfp_k(x) \log \left(1 +
  t_k(x)\right)$, where
$t_k(x) = \frac{\sum_{j \neq k} \lambda_j \h \sfp_j(x)}{\lambda_k \h
  \sfp_k(x)}$. Using the inequality
$\log(1 + t) \geq \frac{t}{1 + t}$ for $t \geq 0$:
\begin{align*}
  H(K \mid X)
  & \geq \sum_{k = 1}^p \lambda_k \sum_x \h \sfp_k(x) \paren*{\frac{\sum_{j \neq k} \lambda_j \h \sfp_j(x)}{\lambda_k \h \sfp_k(x) + \sum_{j \neq k} \lambda_j \h \sfp_j(x)} } \\
& = \sum_k \lambda_k \sum_x \h \sfp_k(x) \paren*{\frac{\h \sfp_\lambda(x) - \lambda_k \h \sfp_k(x)}{\h \sfp_\lambda(x)}} \\
& = \sum_x \frac{1}{\h \sfp_\lambda(x)} \sum_{k = 1}^p \paren*{\lambda_k \h \sfp_k(x)\h \sfp_\lambda(x) - \lambda_k^2[\h \sfp_k(x)]^2} \\
  & = \sum_x \frac{\h \sfp_\lambda(x)(\sum_k \lambda_k \h \sfp_k(x))
    - \sum_k \lambda_k^2[\h \sfp_k(x)]^2}{\h \sfp_\lambda(x)} \\
  & = \sum_x \frac{[\h \sfp_\lambda(x)]^2
- \sum_k \lambda_k^2[\h \sfp_k(x)]^2}{\h \sfp_\lambda(x)}.
\end{align*}
The Chi-squared-like form follows from splitting the fraction:
$\sum_x \h \sfp_\lambda(x) - \sum_x \frac{\sum_k \lambda_k^2[\h
  \sfp_k(x)]^2}{\h \sfp_\lambda(x)} = 1 - \sum_k \lambda_k^2 \sum_x
\frac{[\h \sfp_k(x)]^2}{\h \sfp_\lambda(x)}$.
\end{proof}

By substituting this explicit lower bound for the ambiguity
$H(K \mid X)$ back into our main theorem, we arrive at a concrete
guarantee for the retrained model $\KL(\h\sfp_\lambda \parallel \h\pi_\lambda) \geq$:
\[
   \sum_{k = 1}^p \lambda_k \e_k - H(\lambda)
  + \paren*{\sum_x \frac{\sum_{i \neq j} \lambda_i \lambda_j \h \sfp_i(x) \h \sfp_j(x)}{\h \sfp_\lambda(x)}}.
\]
The last term, an interaction score, is a non-negative quantity that
is zero if and only if the sources are disjoint. This final inequality
makes the counter-intuitive conclusion undeniable: the guaranteed
error of a single retrained model does not decrease with overlap, but
rather increases by a computable amount directly related to the
``cross-talk" between conflicting sources. This interaction score is
the tangible price the single model pays for ambiguity, a price the
gated model avoids, thus making the gated model's advantage even more
pronounced in the face of messy, overlapping data.

\subsubsection{The Modular Advantage: Decoupling}

In contrast, the Modular Gating network effectively inverts this
relationship. Theorem \ref{th:robust-existence} establishes that the
risk of the robust gate is bounded by:
\[
    \text{Risk}_{\text{mod}} \le \underbrace{\log \paren*{\sum e^{\e_k}}}_{\text{Capacity Cost}} - \underbrace{\JSD^{\lambda}(\h\sfp_1, \dots, \h\sfp_p)}_{\text{Separability Gain}} - \text{Overlap}.
\]
Here, the divergence term appears with a \textit{negative} sign. For
the modular system, task diversity acts as a \textbf{Geometric Gain}.

\textbf{The Decoupling Hypothesis.} This analysis identifies a structural phase transition.
\begin{itemize}
\item \textbf{High Diversity ($\JSD \gg 0$):} When tasks are distinct
  (disjoint supports), the Monolithic model hits the interference
  floor ($\text{Risk} \ge \JSD$). However, for the Modular model, the
  separability gain becomes maximal ($\JSD^{\lambda} \to H(\lambda)$),
  effectively canceling the capacity cost. The bound converges to
  $V^* \le \max_k \epsilon_k$.
\end{itemize}
While the explicit diversity ``discount'' disappears, the result is profound: the modular architecture has successfully \textbf{decoupled} the tasks. The risk is no longer lower-bounded by the geometry ($\JSD$), but is instead upper-bounded strictly by the intrinsic error of the experts.

\subsection{Safety in Convex Regimes: A Geometric Perspective}
\label{app:linear_safety}

One might ask: does modularity sacrifice performance when tasks are simple and compatible? We prove that in convex settings, the answer is no. We assume $\Pi$ is a linear model family (e.g., exponential families).

Our analysis leverages the \emph{Pythagorean equality theorems} of
\cite{Csiszar1975} and \cite{CsiszarMatus2003}, which hold for
\emph{linear models} defined as follows. We assume throughout this section that $\Pi$ is a \emph{linear model}
in the sense of \cite{Csiszar1975}. That is, there exists a
finite-dimensional vector space of statistics
$\phi = (\phi_1, \ldots, \phi_m)$ and an open convex set
$\cE \subseteq \mathbb{R}^m$ such that
$\Pi = \curl*{ \pi \in \cP(\sX) \colon \E_\pi[\phi] \in \cE}$. The
map $\eta(\pi) = \E_\pi[\phi]$ is called the \emph{expectation
  parameter}. The set $\cE$ is affine, and the mapping
$\pi \mapsto \eta(\pi)$ is one-to-one in the interior of $\Pi$. The
family $\Pi$ thus includes both affine families and regular
exponential families in their expectation parametrization.

For any $\sfp$, we denote by $\pi^*(\sfp)$ its \emph{I-projection} on
$\Pi$, that is
$\pi^*(\sfp) = \argmin_{\pi \in \Pi} \KL(\sfp \parallel \pi)$. The
following fundamental equality is due to
\cite{Csiszar1975,CsiszarMatus2003}.

\begin{lemma}[Pythagorean equality for linear models]
\label{lemma:pythagoras}
For any $\sfp$ and any $\pi \in \Pi$,
\[
  \KL(\sfp \parallel \pi)
  = \KL(\sfp \parallel \pi^*(\sfp)) + \KL(\pi^*(\sfp) \parallel \pi).
\]
\end{lemma}
This equality characterizes the linear models: it expresses the fact
that the projection $\pi^*(\sfp)$ is orthogonal to $\Pi$ in the
geometry induced by $\KL$.

\CommutingPythagorean*

\begin{proof}
  By the Pythagorean equality (Lemma~\ref{lemma:pythagoras}) for linear models applied to each component $\h \sfp_k$ and to an arbitrary model $\pi \in \Pi$; denoting $\pi_k = \pi^*(\h \sfp_k)$ we obtain for each $k$,
\[
\KL(\h \sfp_k \parallel \pi)
= \KL(\h \sfp_k \parallel \pi_k) + \KL(\pi_k \parallel \pi).
\]
Multiplying by $\lambda_k$ and summing up yields
\begin{equation}
\label{eq:sum-identity}
\sum_{k=1}^p \lambda_k \KL(\h \sfp_k \parallel \pi)
= \sum_{k=1}^p \lambda_k \KL(\h \sfp_k \parallel \pi_k)
+ \sum_{k=1}^p \lambda_k \KL(\pi_k \parallel \pi).
\end{equation}
The functions $\pi \mapsto \KL(\h \sfp_\lambda \parallel \pi)$ and
$\pi \mapsto \sum_k \lambda_k \KL(\h \sfp_k \parallel \pi)$ differ by
a constant independent of $\pi$. Specifically,
\[
\sum_{k=1}^p \lambda_k \KL(\h \sfp_k \parallel \pi) - \KL(\h \sfp_\lambda \parallel \pi)
= \JSD^\lambda(\h\sfp_1, \ldots, \h\sfp_p).
\]
This constant difference implies they share the same unique minimizer
over $\pi \in \Pi$. The minimizer of
$\pi \mapsto \sum_k \lambda_k \KL(\h \sfp_k \parallel \pi)$ over
$\pi \in \Pi$ is the same as the minimizer of
$\pi \mapsto \sum_k \lambda_k \KL(\pi_k \parallel \pi)$, which is
equivalent to minimizing $-\E_{\pi_g}[\log \pi]$, where
$\pi_g = \sum_k \lambda_k \pi_k$. Since $\Pi$ is a linear model, the
convex combination $\pi_g$ lies in $\Pi$. Therefore, the unique
minimizer of $\KL(\pi_g \parallel \pi)$ is $\pi = \pi_g$. This
establishes that $\pi_g$ is the minimizer of
$\KL(\h \sfp_\lambda \parallel \pi)$, proving the first claim:
\[
\pi_g = \pi^*(\h \sfp_\lambda).
\]
Moreover, setting $\pi = \pi_g$ in \eqref{eq:sum-identity} yields the following useful error decomposition:
\[
\sum_{k=1}^p \lambda_k \KL(\h \sfp_k \parallel \pi_g)
= \sum_{k=1}^p \lambda_k \KL(\h \sfp_k \parallel \pi_k)
+ \sum_{k=1}^p \lambda_k \KL(\pi_k \parallel \pi_g).
\]
The second term on the right-hand side is, by definition, the
Jensen-Shannon divergence of the projections,
$\JSD^\lambda(\pi_1, \ldots, \pi_p)$. The left-hand side can be
rewritten using the constant-shift identity we established earlier:
\[
\sum_{k=1}^p \lambda_k \KL(\h \sfp_k \parallel \pi_g) = \KL(\h \sfp_\lambda \parallel \pi_g) + \JSD^\lambda(\h\sfp_1, \ldots, \h\sfp_p).
\]
Combining these two equalities gives the decomposition:
\[
\KL(\h \sfp_\lambda \parallel \pi_g) = \sum_{k=1}^p \lambda_k \KL(\h \sfp_k \parallel \pi_k) + \JSD^\lambda(\pi_1, \ldots, \pi_p) - \JSD^\lambda(\h\sfp_1, \ldots, \h\sfp_p).
\]
\end{proof}

The proof also yields the following decomposition of the gated model's error.

\begin{corollary}[Error Decomposition for Linear Models]
\label{cor:error_decomposition}
Under the assumptions of Theorem~\ref{thm:linear_safety}, the error of the gated model decomposes as:
\[
  \KL(\h \sfp_\lambda \parallel \pi^*(\h \sfp_\lambda)) = \sum_{k=1}^p \lambda_k \KL(\h \sfp_k \parallel \pi_k) + \JSD^\lambda(\pi_1, \ldots, \pi_p) - \JSD^\lambda(\h\sfp_1, \ldots, \h\sfp_p).
\]
\end{corollary}

This theorem ensures that modularity is a ``safe'' architectural prior: it loses nothing in convex regimes while providing strictly superior robustness guarantees in the presence of conflicting, non-convex distributions. The best model trained on the mixture $\h \sfp_\lambda$ coincides with the gated model obtained by mixing the best component models $\pi_k$. In particular, retraining on the mixture does not improve upon gating.

\subsection{The Generalization Gap: Sample Efficiency}
\label{app:generalization_gap}

Beyond optimization, modularity offers a distinct statistical advantage. The guarantees for the robust gate transfer efficiently from empirical data to the true population. To establish these generalization bounds, we first introduce the necessary assumption and formally define the vector-valued Rademacher complexity.

\begin{assumption}[Bounded Expert Likelihoods]
\label{ass:boundedness}
We assume that the pre-trained experts are bounded away from zero on
the support of the data distribution. That is, there exists a constant
$M > 0$ such that for all $x \in \sX_0$ and all $k \in [1, p]$,
the negative log-likelihood is bounded:
\[
  \abs*{\log \h \pi_k(x)} \le M.
\]
This implies that for any valid probability mixture $\pi_g(x)$, the
probability is lower-bounded by $e^{-M}$.
\end{assumption}

In the context of Large Language Models (LLMs) with finite vocabulary
sizes, this assumption is naturally satisfied. The standard Softmax
function yields strictly positive probabilities for all tokens. Unless
an expert explicitly masks a token to $-\infty$ (assigning it zero
probability), the log-likelihood remains finite.

\paragraph{Vector-Valued Rademacher Complexity.}
Let $\cF$ be a class of functions mapping $\sX_0$ to $\Rset^p$ and
denote by $f_j$ the $j$-th component of $f \in \cF$. The empirical
Rademacher complexity associated to such a vector-valued class
of functions is defined by
\[
  \h \Rad_S(\cF)
  = \frac{1}{m} \E_{\bsigma} \bracket*{ \sup_{f \in \cF} \sum_{i=1}^m \sum_{j=1}^p
    \sigma_{i,j} f_j(x_i) },
\]
with $\sigma_{i,j}$'s independent Rademacher variables, that is
independent uniformly distributed random variables taking values in
$\curl*{-1, +1}$. Its expectation,
$\Rad_m(\cF) = \E_S\bracket*{\h \Rad_S(\cF)}$, is the Rademacher
complexity of $\cF$. Note that for $p = 1$, this coincides with the
standard notion of Rademacher complexity for real-valued
functions. The vectorial extension is also called \emph{factor graph
  Rademacher complexity} in \citep{CortesKuznetsovMohriYang2014}, in
the context of structured prediction.

We will denote by $\sfp_k$ the true distribution according to which
the dataset $D_k$ is drawn (or a sample $S_k$), and will denote by
$\Rad^k_{m_k}(\sGone)$ the Rademacher complexity of $\sGone$ for the
distribution $\sfp_k$ and a sample size $m_k$. Note that Theorem~\ref{th:generalization-gate} considers the simplified case where all experts have equal sample sizes $m_k = m$ and uniform complexity across distributions.

\GeneralizationGate*
\begin{proof}
  Let $\cL_{\sGone} = \curl*{x \mapsto -\log \pi_g(x) \mid g \in \sGone}$
  be the loss class associated with the gating hypothesis space $\sGone$.
  For a fixed $k \in [p]$, by the standard Rademacher
  complexity generalization bounds \citep[see e.g.,][]{MohriRostamizadehTalwalkar2018} for functions bounded by $M$, for any $\delta > 0$, with probability at least $1 - \delta$, the following inequality holds for all $g \in \sGone$:
  \[
  \E_{x \sim \sfp_k} \bracket*{ -\log \pi_g(x) }
  \le \E_{x \sim \h \sfp_k} \bracket*{ -\log \pi_g(x) }
  + 2 \Rad^k_{m_k}(\cL_{\sGone})
    + M \sqrt{\frac{\log \frac{1}{\delta} }{2m_k}}.
  \]
  Thus, by the union bound, the following inequality holds
  for all $k \in [p]$ with probability at least $1 - \delta$:
  \[
  \E_{x \sim \sfp_k} \bracket*{ -\log \pi_g(x) }
  \le \E_{x \sim \h \sfp_k} \bracket*{ -\log \pi_g(x) }
  + 2 \Rad^k_{m_k}(\cL_{\sGone})
    + M \sqrt{\frac{\log \frac{p}{\delta} }{2m_k}}.
  \]
  Multiplying each inequality by $\lambda_k$ and summing up yields
  that with probability at least $1 - \delta$, the following holds for
  all $g \in \sGone$ and $\lambda \in \Delta$:
\[
  \E_{x \sim \sfp_\lambda} \bracket*{ -\log \pi_g(x) }
  \le \E_{x \sim \h \sfp_\lambda} \bracket*{ -\log \pi_g(x) }
  + \sum_{k = 1}^p \lambda_k \bracket*{ 2\Rad^k_{m_k}(\cL_{\sGone})
    + M \sqrt{\frac{\log \frac{p}{\delta} }{2m_k}}}.
\]
We now bound $\Rad^k_{m_k}(\cL_{\sGone})$ in terms of
$\Rad^k_{m_k}(\sGone)$, using the vector contraction established in
\citep{CortesKuznetsovMohriYang2014} (Lemma A.1) and \citep{Maurer2016Vector}. This inequality holds for $\ell_2$-Lipschitz functions.

We view the gate function class $\sGone$ as a vector-valued hypothesis
class mapping inputs $x \in \sX_0$ to the simplex $\Delta \subset \Rset^p$. The loss function for a function
$g \in \sGone$ and a sample $x_i$ is defined as
$- \log(\pi_g(x_i)) = \Psi_i(\bu) = -\log(\bu \cdot \h \bpi(x_i))$,
where $\h \bpi(x_i) = (\h \pi_1(x_i), \dots, \h \pi_p(x_i))$
and $\bu = (g(x_i, 1), \ldots, g(x_i, p))$.
Under Assumption~\ref{ass:boundedness}, for any fixed sample $x_i$,
$\Psi_i$ is Lipschitz continuous with respect to the $\ell_2$ norm.
The gradient of $\Psi_i$ with respect to the vector $\bu$ is:
\[
  \nabla \Psi_i(\bu) = \frac{-1}{\bu \cdot \h \bpi(x_i)} \h \bpi(x_i).
\]
To determine the Lipschitz constant, we examine the $\ell_2$ norm of
the gradient:
\[
  \norm{\nabla \Psi_i(\bu)}_2
  = \frac{\norm{\h \bpi(x_i)}_2}{\abs{\bu \cdot \h \bpi(x_i)}}.
\]
First, consider the numerator. By the definition of the coincidence norm,
we have $\norm{\h \bpi(x_i)}_2 \le C_\Pi$.
Second, consider the denominator. By Assumption~\ref{ass:boundedness}, the mixture probability is
lower-bounded by $e^{-M}$.  Thus:
\[
  \norm{\nabla \Psi_i(\bu)}_2 \le \frac{C_\Pi}{e^{-M}} = C_\Pi \, e^M.
\]
The function is therefore $C_\Pi \, e^M$-Lipschitz with respect to the $\ell_2$
norm. Thus, by the vector contraction lemma, we have
\[
  \Rad^k_{m_k}(\cL_{\sGone}) \leq \sqrt{2} C_\Pi \, e^M \Rad^k_{m_k}(\sGone).
\]
Plugging in the right-hand side in the inequality previously proven
completes the proof.
\end{proof}

\subsubsection{Comparison with Monolithic Generalization}
Let $\h \pi_{\text{scratch}}$ be a monolithic model retrained on the aggregate data. Standard learning theory bounds its generalization error by the complexity of the full model class $\Pi$:
\[
  \text{GenGap}(\h \pi_{\text{scratch}}) \approx O\paren*{ \Rad_m(\Pi) }.
\]
In contrast, our modular solution scales with $\Rad_m(\sGone)$. Since the gate is typically a lightweight network (e.g., a shallow Transformer with $10^6$ parameters) compared to the massive experts (LLMs with $10^9+$ parameters), we have $\Rad_m(\sGone) \ll \Rad_m(\Pi)$. The term $C_\Pi$ measures the effective overlap of the experts. In the worst case of redundant experts, $C_\Pi \approx \sqrt{p}$. However, in the ideal modular setting where experts are specialized (disjoint supports), $C_\Pi \approx 1$. Thus, even with the overlap factor, the modular approach requires significantly fewer samples to learn a robust policy than retraining a monolithic model.



\section{Optimization and Algorithms}
\label{app:optimization-robust}

Let us analyze the $\min_g \max_\lambda$ game with this new payoff.
The $\lambda$-player seeks to maximize $\wt L(\lambda, g)$ over
$\lambda \in \Delta([1, p])$. Since $\wt L(\lambda, g)$ is linear in
$\lambda$, its maximum is also achieved at a vertex. This gives:
\[
\max_{\lambda \in \Delta([1, p])} \wt L(\lambda, g)
= \max_{k \in [1, p]} \KL \paren*{\h \sfp_k \parallel \pi_g}.
\]
This is the \emph{exact same} objective function for the $g$-player as in
our equivalent problem. Therefore, the solution to this new game is
the same as the solution to our original problem:
\[
\min_{g \in \sGone} \max_{\lambda \in \Delta} \wt L(\lambda, g)
=
\min_{g \in \sGone} \max_{k \in [1, p]} \KL \paren*{\h \sfp_k \parallel \pi_g}
=
\min_{g \in \sGone} \max_{\lambda \in \Delta} L(\lambda, g).
\]
This new game $\wt L(\lambda, g)$ is convex-concave:
\begin{itemize}
\item Convex in $g$: $\wt L$ is a positive, weighted sum of functions
  $\KL(\h \sfp_k \parallel \pi_g)$, each of which is convex in $g$.
\item Concave in $\lambda$: $\wt L$ is linear in $\lambda$, which is a
  special case of concavity.
\end{itemize}
We can therefore find the robust solution $g^*$ by applying standard
no-regret algorithms to this tractable convex-concave game $\wt L(\lambda, g)$.

\subsection{Game Dynamics and Reformulation}
\label{app:linearization}

\subsubsection{Algorithm and Player Dynamics}

The algorithm simulates the game $\wt L(\lambda, g)$ where two players, a
$\lambda$-player and a $g$-player, iteratively update their strategies.

The $\lambda$-player (maximizer) chooses a mixture distribution
$\lambda_t \in \Delta([1, p])$. This is a standard ``learning from
experts" problem. At each round, given the opponent's gate $g_t$, the
``gain" associated with selecting expert $k$ is its contribution to
the payoff:
\[
\ell_t(k) = \KL\!\paren{\h \sfp_k \,\|\, \pi_{g_t}}.
\]
The $\lambda$-player's goal is to find the mixture $\lambda$ that
maximizes $\sum_k \lambda_k \ell_t(k)$. We use the Exponentiated
Gradient algorithm, which updates the weights to favor experts with
higher gain:
\[
  \lambda_{t+1}(k)
  = \frac{\lambda_t(k) \exp \paren{\eta_\lambda \ell_t(k)}}
  {\sum_{j=1}^p \lambda_t(j) \exp \paren{\eta_\lambda \ell_t(j)}}.
\]
The $g$-player (minimizer) chooses a gate $g_t \in \cG$. Given the
$\lambda$-player's new mixture $\lambda_{t+1}$, its goal is to
minimize the loss $\wt L(\lambda_{t+1}, g_t)$. The $g$-player updates its
gate using Online Gradient Descent. The gradient of the loss
$\wt L(\lambda_{t+1}, g_t)$ with respect to the gate weights $g_t(x, k)$
is:
\begin{align*}
  v_t(x, k)
  &= \nabla_{g_t(x, k)} \wt L(\lambda_{t+1}, g_t) \\
  &= \nabla_{g_t(x, k)} \sum_{j=1}^p \lambda_{t+1}(j) \KL(\h \sfp_j \parallel \pi_{g_t}) \\
  &= \sum_{j=1}^p \lambda_{t+1}(j) \nabla_{g_t(x, k)} \KL(\h \sfp_j \parallel \pi_{g_t}) \\
  &= \sum_{j=1}^p \lambda_{t+1}(j) \paren*{ -\frac{\h \sfp_j(x) \h \pi_k(x)}{\pi_{g_t}(x)} } \\
  &= - \frac{\paren*{\sum_{j=1}^p \lambda_{t+1}(j) \h \sfp_j(x)} \h \pi_k(x)}{\pi_{g_t}(x)} \\
  &= -\frac{\h \sfp_{\lambda_{t+1}}(x) \h \pi_k(x)}{\pi_{g_t}(x)}.
\end{align*}
The gate is then updated by taking a step in the negative gradient
direction and projecting back onto the feasible set, for each $x \in \cX_0$:
\[
g_{t+1}(x,\cdot)
\gets \Pi_{\sGone} \paren*{g_t(x,\cdot) - \eta_g \, v_t(x,\cdot)}.
\]
The final solution is the time-average of the iterates, as this is
what is guaranteed to converge to the equilibrium:
$\ov g_T = \frac{1}{T} \sum_{t=1}^{T} g_t$.

\begin{algorithm}[t]
\caption{Robust Gate via No-Regret Dynamics (EG + OGD)}
\label{alg:constrained-kl}
\begin{algorithmic}[1]
  \STATE {\bfseries Input:} Models $\h \pi_1, \ldots, \h \pi_p$;
  datasets $D_1, \ldots, D_p$; learning rates $\eta_\lambda, \eta_g$;
  iterations $T$.
  \STATE {\bfseries Initialize:} $\lambda_0(k) = 1/p$;
  $g_0(x, k) = 1/p$ for all $x, k$.
  \FOR{$t=0$ {\bfseries to} $T-1$} 
  \STATE \textbf{(Compute Gains)} For each expert $k \in [1,p]$, compute
its gain for the $\lambda$-player:
$\ell_t(k) = \KL(\h \sfp_k \| \pi_{g_t})$.
  \STATE \textbf{($\lambda$-update)} Update mixture weights using
Exponentiated Gradient on gains $\ell_t$:
$\lambda_{t+1}(k) \propto \lambda_t(k) \exp(\eta_\lambda \ell_t(k))$
and re-normalize.
  \STATE \textbf{($g$-update)} Construct mixture
$\h \sfp_{\lambda_{t+1}}$ from new $\lambda_{t+1}$.
  \STATE Compute gradient
$v_t(x, k) = -(\h \sfp_{\lambda_{t+1}}(x)/\pi_{g_t}(x)) \h \pi_k(x)$
for all $(x, k)$.
  \STATE Compute intermediate update $g'_{t+1} = g_t - \eta_g v_t$.
  \STATE \label{line:projection-kl} Project onto constrained space:
  $g_{t+1} = \Pi_{\sGone}(g'_{t+1})$.
  \ENDFOR
\STATE {\bfseries Output:} The time-averaged gate
$\ov g_T = \frac{1}{T}\sum_{t=1}^T g_{t}$.
\end{algorithmic}
\end{algorithm}

\subsubsection{Projection Challenge}

The primary difficulty lies in Line~9 of
\cref{alg:constrained-kl}. The projection $\Pi_{\sGone}$
requires finding the closest point $g$ to $g'_{t+1}$ (typically in
Euclidean norm) such that $g$ satisfies both:
\begin{enumerate}
\item Local normalization: $g(x, \cdot) \in \Delta([1,p])$ for all
$x \in \sX_0$.
\item Global normalization:
$\sum_{x \in \sX_0} \sum_{k=1}^p g(x, k) \h \pi_k(x) = 1$.
\end{enumerate}
This is a projection onto the intersection of the product space of
simplices $\sG$ and an affine subspace defined by the linear equality
$Z_g = 1$.

\subsection{Convergence analysis}
\label{app:convergence_analysis}

Assuming an efficient oracle for the projection $\Pi_{\sGone}$, the
standard convergence theorems for no-regret dynamics in convex-concave
games apply to our reformulated game $\wt L(\lambda, g)$.

\ignore{
\begin{restatable}[Algorithmic Convergence]{theorem}{ConvergenceKl}
\label{th:convergence-kl-extended}
Let
$\wt V = \min_{g \in \sGone} \max_{\lambda \in \Delta} \wt L(\lambda,
g)$ be the value of the convex-concave game with payoff
$\wt L(\lambda, g) = \sum_k \lambda_k \KL(\h \sfp_k \parallel \pi_g)$.
Suppose gains $\ell_t(k)$ are bounded, $\ell_t(k) \in [0, M_\lambda]$,
and gradient norms $\norm{v_t}_2$ are bounded, $\norm{v_t}_2 \le
M_g$. If the projection $\Pi_{\sGone}$ can be computed efficiently,
then with appropriate learning rates
$\eta_\lambda = \sqrt{2 \log p / (T M_\lambda^2)}$ and
$\eta_g = D_{\sGone} / (M_g \sqrt{T})$ (where $D_{\sGone}$ is the
diameter of $\sGone$), the time-averaged strategy $\ov g_T$ converges
to the optimal robust solution:
$ \max_{\lambda \in \Delta([p])} \wt L(\lambda,\ov g_T) - \wt V \leq
\frac{D_{\sGone} M_g}{\sqrt{T}} + \frac{M_\lambda \sqrt{2 \log
    p}}{\sqrt{T}} = O \paren[\bigg]{\sqrt{\frac{\log p}{T}}}.  $ Since
$\max_{\lambda} \wt L(\lambda, g) = \max_{\lambda} L(\lambda, g)$ for
any $g$, this convergence to $\wt V$ implies convergence to the
solution of the original robust problem.
\end{restatable}
}

\ConvergenceKl*
\begin{proof}
The proof relies on the regret bounds for the players' algorithms
and the connection between average regret and the duality gap for
time-averaged strategies in convex-concave games.
Let $R_T^\lambda$ be the regret of the $\lambda$-player (using EG) and
$R_T^g$ be the regret of the $g$-player (using OGD over
$\sGone$). Standard bounds yield:
$R_T^\lambda \le M_\lambda \sqrt{2T \log p}$ and
$R_T^g \le D_{\sGone} M_g \sqrt{T}$.
Let $\ov\lambda_T = \frac{1}{T}\sum_{t=1}^T \lambda_t$ and
$\ov g_T = \frac{1}{T}\sum_{t=1}^T g_t$.
By the properties of convex-concave games, the duality gap of the
time-averaged strategies is bounded by the sum of average regrets:
\[
\max_{\lambda \in \Delta} \wt L(\lambda, \ov g_T) -
\min_{g \in \sGone} \wt L(\ov \lambda_T, g)
\le \frac{R_T^\lambda + R_T^g}{T}.
\]
The value of the game is $\wt V = \min_g \max_\lambda \wt L(\lambda, g)$.
By weak duality, we know that for any $\ov \lambda_T$,
$\min_{g \in \sGone} \wt L(\ov \lambda_T, g) \le \wt V$.
Therefore, we can bound the suboptimality of $\ov g_T$:
\begin{align*}
\max_{\lambda \in \Delta} \wt L(\lambda, \ov g_T) - \wt V
&\le \max_{\lambda \in \Delta} \wt L(\lambda, \ov g_T) -
\min_{g \in \sGone} \wt L(\ov \lambda_T, g) \\
&\le \frac{R_T^\lambda + R_T^g}{T} \\
&\le \frac{M_\lambda \sqrt{2T \log p} + D_{\sGone} M_g \sqrt{T}}{T} \\
&= \frac{D_{\sGone} M_g}{\sqrt{T}} + \frac{M_\lambda \sqrt{2 \log p}}{\sqrt{T}}.
\end{align*}
This proves convergence of the time-averaged strategy $\ov g_T$ to
the value of the game $\wt V$. As established in the reformulation,
$\wt V$ is the value of the original problem, and thus $\ov g_T$ is a
near-optimal robust gate.
\end{proof}

The analysis extends to the case where a subset
$\Lambda \subseteq \Delta([1, p])$ is used
(Section~\ref{sec:prior-knowledge}).

\begin{restatable}[Algorithmic Convergence on $\Lambda$]{theorem}{ConvergenceLambda}
\label{th:convergence-lambda}

Let $\wt V_\Lambda = \min_{g \in \sGone} \max_{\lambda \in \Lambda}
\wt L(\lambda, g)$ be the value of the restricted convex-concave game,
where $\wt L(\lambda, g) = \sum_k \lambda_k \KL(\h \sfp_k \parallel
\pi_g)$. Let $\ov g_T$ be the time-averaged gate obtained by running
the no-regret algorithm (Algorithm~\ref{alg:constrained-kl}, modified
with $\lambda$-updates projected onto $\Lambda$). Assume the gains and
gradients are bounded. Then, the algorithm admits the following
convergence guarantee:
\[
 \E\bracket*{\max_{\lambda \in \Lambda} \wt L(\lambda, \ov
  g_T)} - \wt V_\Lambda \leq O\paren*{\sqrt{\frac{\log p}{T}}}.
\]

\end{restatable}

\begin{proof} The proof is a direct extension of
 Theorem~\ref{th:convergence-kl-short}. The game remains
 convex-concave. The $g$-player's algorithm and regret bound are
 unchanged. The $\lambda$-player now runs a projected online mirror
 descent (OMD) algorithm, specifically, Exponentiated Gradient with a
 projection. The regret of this algorithm is still bounded relative
 to the best fixed strategy in $\Lambda \subseteq \Delta([1,
 p])$. Since the OMD algorithm uses the negative entropy regularizer
 (which leads to EG), the regret bound remains
 $R_T^\lambda \le O(\sqrt{T \log p})$. The standard analysis bounding
 the duality gap by the average regret,
 $\frac{1}{T}(R_T^\lambda + R_T^g)$, applies directly, yielding the
 $O(\sqrt{\log p / T})$ convergence rate.
\end{proof}

\subsection{Leveraging the Least-Favorable Mixture}
\label{app:least-favorable}

Our minimax solution provides two crucial outputs: the robust gate
$g^*$, and the least-favorable mixture $\lambda^*$. It is important to
note that $\lambda^*$ corresponds to the saddle point of the
\emph{linearized} game
$\wt L(\lambda, g) = \sum_k \lambda_k \KL(\h \sfp_k \parallel \pi_g)$,
rather than the original convex-convex payoff $L$.

Consequently, $\lambda^*$ identifies the specific weighting of source
domains that is maximally challenging for the modular ensemble in the
linearized regime. This insight is highly valuable in practical
scenarios where engineers must train a single, static model $\h \pi$
on the aggregated data $\h \sfp_\lambda$ (e.g., for reduced inference
latency). Instead of resorting to heuristic choices for the mixture
weights $\lambda$ (such as uniform $\lambda_k = 1/p$ or weights based
on dataset size), the optimal $\lambda^*$ resulting from the no-regret
algorithm provides a statistically principled alternative.

By training a new model on the data mixture
$\h \sfp_{\lambda^*} = \sum_k \lambda^*_k \h \sfp_k$, the resulting
model $\h \pi_{\lambda^*}$ is optimized for the distribution where the
underlying expert ensemble is most vulnerable (in terms of the upper
bound $\wt L$). This strategy effectively turns the worst-case
scenario for the gated model into the training objective for the
static model.

\subsection{Convergence of Primal-Dual Dynamics:
  Proof of Theorem~\ref{th:primal-dual-convergence}}
\label{app:convergence_primal_dual}

The pseudocode of the Primal-Dual algorithm is given
in Algorithm~\ref{alg:primal-dual}.

\begin{algorithm}[h]
   \caption{Stochastic Primal-Dual Training Loop}
   \label{alg:primal-dual}
\begin{algorithmic}[1]
   \FOR{iteration $t = 1$ {\bfseries to} $T$}
   \STATE \textbf{Data Sampling:} Sample batch $B = \cup_k B_k$ from experts.
   \STATE \textbf{Forward Pass:} Compute logits $g_\theta(x)$ and expert log-probs for $x \in B$.
   \STATE \textbf{Constraint Est.:} Estimate $\h Z \approx \frac{1}{|B|}\sum_{x \in B} \frac{\pi_g(x)}{q(x)}$ via Importance Sampling, where $q(x) = \frac{1}{p}\sum \h \pi_k(x)$ is the uniform mixture proposal.
   \STATE \textbf{$\lambda$-Step (Adversary):} Update mixture weights:
   \STATE \quad $\lambda_k \leftarrow \lambda_k \cdot \exp(\eta_\lambda \cdot \ell_k)$ \COMMENT{$\ell_k$: loss on domain $k$}
   \STATE \textbf{$\mu$-Step (Constraint):} $\mu \leftarrow \mu + \eta_\mu (\h Z - 1)$.
   \STATE \textbf{$g$-Step (Gate):} Update $\theta$ to minimize $\cL$ via AdamW.
   \ENDFOR
\end{algorithmic}
\end{algorithm}

\label{app:primal-dual-proof}

\PrimalDualConvergence*

\begin{proof}
  The proof relies on viewing the optimization of the Lagrangian
  $\mathcal{L}(\theta, \lambda, \mu)$ as a zero-sum game between a
  primal player (controlling $g_\theta$) and a dual player
  (controlling $\lambda, \mu$). We analyze the convergence using the
  framework of online convex optimization (OCO) and regret bounds.

\paragraph{1. The Lagrangian and Duality Gap.}
Recall the Lagrangian of the reformulated game:
\[
\mathcal{L}(g, \lambda, \mu) = \wt L(\lambda, g) + \mu(Z_g - 1) = \sum_{k=1}^p \lambda_k \KL(\h \sfp_k \parallel \pi_g) + \mu \paren*{\sum_{x \in \sX_0} \pi_g(x) - 1}.
\]
This function is convex in the primal variable $g$ (as established in
Lemma 3.1 and Section 5.1) and linear (concave) in the dual variables
$\lambda, \mu$.  Let $w = (\lambda, \mu)$ denote the combined dual
variables. The algorithm generates a sequence of iterates
$(g_t, w_t)_{t=1}^T$.  We define the \emph{duality gap} for the
time-averaged iterates $(\ov g_T, \ov w_T)$ as:
\[
\text{Gap}(\ov g_T, \ov w_T) = \max_{w \in \mathcal{W}} \mathcal{L}(\ov g_T, w) - \min_{g \in \sGone} \mathcal{L}(g, \ov w_T),
\]
where $\mathcal{W}$ is a compact subset of the dual space containing
the optimal dual solution $w^*$.

\paragraph{2. Regret Decomposition.}
A standard result in game dynamics \citep[e.g.,][]{freund1999adaptive,
  cesa2006prediction} states that the duality gap is bounded by the
sum of the average regrets of the players.  Let $R_T^g$ be the regret of the
$g$-player minimizing $\mathcal{L}(\cdot, w_t)$ and $R_T^w$ be the
combined regret of the dual players maximizing $\mathcal{L}(g_t, \cdot)$:
\[
\text{Gap}(\ov g_T, \ov w_T) \leq \frac{R_T^g + R_T^w}{T}.
\]

\paragraph{3. Bounding the Regrets.}
We analyze the regret for each player based on their specific update
rules in Algorithm~\ref{alg:primal-dual}:

\begin{itemize}
\item \textbf{The $\lambda$-player (Simplex):} Updates $\lambda$ using
  Exponentiated Gradient (EG). For linear losses with gradients
  bounded by $M_\lambda$, the regret of EG over the simplex is bounded
  by:
  \[
    R_T^\lambda \le M_\lambda \sqrt{2 T \log p}.
  \]
    
\item \textbf{The $\mu$-player (Scalar Constraint):} Updates $\mu$
  using Gradient Ascent (Dual Ascent). Assuming the constraint
  violation (gradient w.r.t $\mu$) is bounded by
  $M_\mu = \max_g |Z_g - 1|$ and the optimal $\mu^*$ lies in a bounded
  range $[-D_\mu, D_\mu]$, standard Gradient Ascent bounds give:
  \[
    R_T^\mu \le D_\mu M_\mu \sqrt{T}.
  \]
    
\item \textbf{The $g$-player (Gate Parameters):} Updates $g$ (via
  $\theta$) using AdamW (a variant of Online Mirror Descent). Under
  the convexity assumption of $\wt L$ w.r.t $g$ and bounded gradients
  $M_g$, the regret is bounded by:
  \[
    R_T^g \le D_{\sGone} M_g \sqrt{T}.
  \]
\end{itemize}

\paragraph{4. Convergence Rate.}
Summing these terms, the total average regret scales as:
\[
\frac{R_T^{total}}{T} \le \frac{C \sqrt{T}}{T} = O\paren*{\frac{1}{\sqrt{T}}}.
\]
Thus, the duality gap decays at a rate of $O(1/\sqrt{T})$.

\paragraph{5. Recovering the Objectives.}
The convergence of the duality gap implies convergence of both the
objective value and the constraint satisfaction:
\begin{itemize}
\item \textbf{Robust Loss:}
  $\max_\lambda \wt L(\lambda, \ov g_T) - \wt V \le \text{Gap}(\ov
  g_T, \ov w_T) \le O(1/\sqrt{T})$.
\item \textbf{Constraint Violation:} The Lagrangian term
  $\mu(Z_g - 1)$ implies that if the constraint is violated
  ($|Z_{\ov g_T} - 1| > \epsilon$), the dual player $\mu$ would
  exploit this to maximize the gap. Therefore, the constraint
  violation is also bounded by the gap:
    \[
      |Z_{\ov g_T} - 1|
      \le \frac{\text{Gap}(\ov g_T, \ov w_T)}{|\mu^*|}
      \le O\paren*{\frac{1}{\sqrt{T}}}.
    \]
\end{itemize}
This completes the proof that the Primal-Dual algorithm converges to
the optimal robust solution while asymptotically satisfying the
normalization constraint.
\end{proof}

\newpage
\section{Scalable Implementation and Inference}
\label{app:scalable_implementation}

To scale the robust modular framework to high-dimensional generative
models such as Transformers, we must address two practical challenges:
characterizing the functional form of the gate $g$, and enforcing the
global normalization constraint $Z_g = 1$ during stochastic
optimization. This section details the system architecture and the
Primal-Dual algorithm used to solve the minimax game.

\subsection{Architecture Parameterization}

We parameterize the components of the modular system as follows:

\textbf{1. The Experts ($\h \pi_k$):} The ensemble consists of $p$
pre-trained, frozen autoregressive models (e.g., GPT-style Causal
Transformers). For a sequence $x = (x_1, \dots, x_T)$, each expert $k$
provides a conditional probability distribution
$\h \pi_k(x_t \mid x_{<t})$. The total log-probability of a sequence
is $\log \h \pi_k(x) = \sum_{t=1}^T \log \h \pi_k(x_t \mid x_{<t})$.

\textbf{2. The Gate ($g_\theta$):} Unlike the experts, the gate
function is \emph{non-causal}. It observes the entire sequence $x$ to
determine the optimal routing weights. We parameterize $g_\theta$ as a
\textbf{Transformer Encoder} (e.g., BERT-style) with parameters
$\theta$.

Here's a mathematical definition of the gate function.  We
parameterize the gate function $g_\theta$ as a non-causal,
bidirectional Transformer Encoder. Unlike the experts, which must be
causal to generate text, the gate observes the full input sequence
$x = (x_1, \dots, x_T)$ to determine the optimal mixing weights. The
computation is defined as follows:

\begin{align*}
  H^{(0)} & = \text{Embed}(x) + \text{PosEnc} \quad \in \Rset^{T \times d} \\
  H^{(L)} & = \text{TransformerEncoder}_\theta \paren*{ H^{(0)} }
            \quad \in \Rset^{T \times d} \\
  v & = \text{Pool}(H^{(L)}) = \frac{1}{T} \sum_{t=1}^T H^{(L)}_t
      \quad \in \Rset^d \quad (\text{Global Mean Pooling}) \\
  w & = W_{\text{out}} v + b_{\text{out}} \quad \in \Rset^p \\
  g_\theta(x) & = \text{Softmax}(w) \quad \in \Delta([1, p]).
\end{align*}
Here, $d$ is the hidden dimension of the gate model, $L$ is the number
of encoder layers, and $p$ is the number of experts. The global
pooling step aggregates the bidirectional context into a single vector
$v$, ensuring that the routing decision $g_\theta(x)$ is based on the
entire sequence content.

\subsection{The Stochastic Primal-Dual Algorithm}
\label{app:stochastic_primal_dual_algorithm}

The algorithm solves the saddle-point problem defined by the Lagrangian:
\[
  \min_{\theta} \max_{\lambda \in \Delta, \mu \in \mathbb{R}}
  \left[ \sum_{k=1}^p \lambda_k \cL_{\text{NLL}}(k, \theta)
    + \mu (Z_{g_\theta} - 1) \right]
\]

\paragraph{Hyperparameters:}
\begin{itemize}
\item $\eta_g$: Learning rate for Gate (e.g., $10^{-4}$, using AdamW).
\item $\eta_\lambda$: Learning rate for Adversary (e.g., $0.1$,
  using SGD/Exponentiated Gradient).
\item $\eta_\mu$: Learning rate for Constraint (e.g., $10^{-2}$,
  using SGD).
\item $\alpha$: Moving average factor for estimating global $Z$
  (e.g., 0.9).
\end{itemize}

\paragraph{Initialization:}
\begin{itemize}
    \item Initialize Gate parameters $\theta$.
    \item Initialize $\log \lambda = [0, \dots, 0]$ (uniform distribution).
    \item Initialize $\mu = 0$.
    \item Initialize Global Normalization Estimate $\ov Z = 1.0$.
\end{itemize}

\begin{algorithm}[tb]
   \caption{Stochastic Primal-Dual Training Loop}
   \label{alg:primal_dual_stochastic}
\begin{algorithmic}[1]
   \FOR{iteration $t = 1$ {\bfseries to} $T$}
   \STATE \textbf{1. Data Sampling:}
   \STATE Sample a batch $B_k$ of size $M$ from each source dataset $D_k$.
   \STATE Combine into a super-batch $B = \bigcup_k B_k$ of size $p \times M$.
   
   \STATE \textbf{2. Forward Pass (Gate \& Experts):}
   \FOR{every $x \in B$}
   \STATE Compute expert log-probs: $L_k(x) = \log \h \pi_k(x)$ for all $k$.
   \STATE Compute gate logits $g_\theta(x)$ and weights $w(x) = \text{Softmax}(g_\theta(x))$.
   \STATE Compute mixture log-prob via LogSumExp:
   \STATE \quad $\log \pi_g(x) = \text{LogSumExp}_k \left( \log w_k(x) + L_k(x) \right)$.
   \STATE Compute unnormalized mass density: $m(x) = \exp(\log \pi_g(x))$.
   \ENDFOR

   \STATE \textbf{3. Constraint Estimation (Importance Sampling):}
   \STATE The proposal distribution is the uniform mixture $q(x) = \frac{1}{p}\sum \h \pi_k(x)$.
   \STATE Note: Samples $x \in B$ follow the empirical mixture $\frac{1}{p} \sum \h p_k$.    
   \STATE Assumption: $\h \pi_k \approx \h p_k$, so $B$ serves as samples from $q(x)$.    
   \STATE Estimate IS weights: $w_{IS}(x) = \pi_g(x) / q(x)$.
   \STATE Estimate Z: $\h Z = \frac{1}{|B|} \sum_{x \in B} w_{IS}(x)$.
   \STATE Update moving average: $\ov Z \leftarrow \alpha \ov Z + (1-\alpha) \h Z$.
   
   \STATE \textbf{4. $\lambda$-Player Update (Adversary):}
   \STATE Calculate loss per domain $k$: $\ell_k = \frac{1}{|B_k|} \sum_{x \in B_k} - \log \pi_g(x)$.
   \STATE Update $\lambda$ (Exponentiated Gradient):
   \STATE \quad $\lambda_k \leftarrow \lambda_k \cdot \exp(\eta_\lambda \cdot \ell_k)$.
   \STATE Renormalize: $\lambda \leftarrow \lambda / \sum_j \lambda_j$.
   
   \STATE \textbf{5. $\mu$-Player Update (Dual Ascent):}
   \STATE Goal: Maximize $\mu(\ov Z - 1)$.
   \STATE $\mu \leftarrow \mu + \eta_\mu (\ov Z - 1)$.
   
   \STATE \textbf{6. $g$-Player Update (Primal Minimization):}
   \STATE Construct Total Loss $\mathcal{J}$:
   \STATE \quad $\mathcal{J} = \underbrace{\sum_{k=1}^p \lambda_k \ell_k}_{\text{Robust NLL}}
   + \underbrace{\mu (\h Z - 1)}_{\text{Lagrangian Penalty}}$
   \STATE Compute gradients $\nabla_\theta \mathcal{J}$.
   \STATE Update $\theta$ using Optimizer (AdamW).
   \ENDFOR
\end{algorithmic}
\end{algorithm}

We solve the constrained minimax problem by relaxing the global
normalization constraint via a Lagrange multiplier $\mu \in
\Rset$. The objective function is the Lagrangian:
\[
  \min_{\theta} \max_{\lambda \in \Delta, \mu \in \Rset} \cL(\theta, \lambda, \mu)
  = \underbrace{\sum_{k=1}^p \lambda_k \E_{x \sim \h \sfp_k} \bracket*{ -\log \pi_{g_\theta}(x)
    }}_{\text{Robust NLL}}
  + \underbrace{\mu \paren*{ \sum_{x \in \sX_0} \pi_{g_\theta}(x) - 1 }
  }_{\text{Normalization Penalty}}
\]
Algorithm~\ref{alg:primal_dual_stochastic} details the stochastic
updates. We use three distinct optimizers: \textbf{Exponentiated
  Gradient} for the simplex-constrained adversary $\lambda$,
\textbf{Dual Ascent} for the constraint $\mu$, and \textbf{AdamW} for
the gate parameters $\theta$ (see
Figure~\ref{fig:primal_dual_dynamics}).

\begin{figure}[t]
\centering
\begin{tikzpicture}[
    node distance=2.5cm,
    player/.style={circle, draw, very thick, minimum size=1.8cm, align=center, font=\bfseries},
    update/.style={->, >=stealth, thick, text width=2cm, align=center, font=\footnotesize}
]

    \node[player, fill=red!10] (lambda) {$\lambda$-Player \\ (Adversary)};
    \node[player, fill=blue!10, below left=of lambda] (gate) {$g$-Player \\ (Gate)};
    \node[player, fill=yellow!10, below right=of lambda] (mu) {$\mu$-Player \\ (Constraint)};

    
    \draw[update, bend right=20] (lambda) -- node[right] {Mixture $\widehat{p}_\lambda$} (gate);
    
    \draw[update, bend left=30] (gate) to node[left] {Loss $\ell_k$ \\ (KL Div)} (lambda);
    
    \draw[update] (gate) -- node[below] {Norm $Z_g$} (mu);
    
    \draw[update, bend right=20] (mu) to node[above right] {Penalty $\mu(Z-1)$} (gate);
    
    \node[above=0.1cm of lambda, font=\small\itshape] {Exp. Gradient};
    \node[below=0.1cm of gate, font=\small\itshape] {AdamW};
    \node[below=0.1cm of mu, font=\small\itshape] {Dual Ascent};

\end{tikzpicture}
\caption{Dynamics of the Primal-Dual Game (Algorithm~\ref{alg:primal_dual_stochastic}). The
  optimization is modeled as a 3-player game. The $\lambda$-player
  maximizes the mixture difficulty using Exponentiated Gradient. The
  $g$-player minimizes the robust loss. The $\mu$-player enforces the
  global normalization constraint ($Z_g=1$) via Dual Ascent.}
\label{fig:primal_dual_dynamics}
\end{figure}

\subsection{Practical Implementation Details}

\paragraph{Log-Space Stability.}
The mixture probability $\pi_g(x) = \sum_{k} g(x, k) \h \pi_k(x)$
involves summing probabilities that may be extremely small (e.g.,
$10^{-100}$ for long sequences). Direct computation leads to
underflow. We strictly perform all operations in log-space using the
LogSumExp:
\[
  \log \pi_g(x)
  = \log \bracket*{\sum_k \exp \paren[\big]{ \log g(x, k) + \log \h \pi_k(x) }}.
\]


\paragraph{Estimating the Partition Function $Z_g$.}
Calculating the global sum $Z_g = \sum_{x \in \sX_0} \pi_g(x)$ exactly
is intractable. We rely on a Monte Carlo estimate using the current
training batch $B$. To estimate $Z$, consistent with
Algorithm~\ref{alg:primal_dual_stochastic}, we use importance sampling
as in Section~\ref{sec:SIR}:
\[
  \h Z
  = \frac{1}{|B|} \sum_{x \in B} \frac{\pi_g(x)}{\frac{1}{p} \sum_{k = 1}^p \h \pi_k(x)}.
\]
To reduce variance in the $\mu$-update, we maintain an Exponential
Moving Average (EMA) of the normalization constant $\ov Z$. A
warm-up period where $\mu$ is fixed to 0 allows the gate to learn
discriminative features before the constraint forces the probability
mass to contract.

In our implementation, we use the training batch $B$ itself to compute
this estimate.  The batch $B$ is constructed by sampling uniformly
from the source datasets $D_k$, so
$x \sim \frac{1}{p} \sum_{k=1}^p \h p_k$.  Under the assumption
that the pre-trained experts are reasonable approximations of their
training data ($\hat{\pi}_k \approx \h p_k$), the empirical mixture
closely approximates the model mixture proposal
$q(x) \approx \frac{1}{p} \sum_{k=1}^p \h p_k$.  This allows us to
reuse the forward-pass data for the constraint estimation without
generating separate synthetic samples from the experts.

\paragraph{Variance Bound for Partition Function Estimation.}
The assumption that the empirical batch approximates the proposal ($q(x) \approx \h \sfp_{\text{mix}}(x)$) is a natural consequence of the base experts being well-trained density estimators of their respective domains ($\KL(\h \sfp_k \parallel \h \pi_k) \le \e_k$). The error in estimating $Z_g$ stems from standard Importance Sampling variance. Since our normalized gate space ensures that $g(x, k) \le 1$, the robust model's likelihood is strictly bounded by the envelope $p \cdot q(x)$ (as discussed in Section~\ref{sec:rejection-sampling}). Thus, the importance weights $w_{IS}(x) = \pi_g(x) / q(x)$ are strictly bounded by $p$. This provides a rigorous formal guarantee that the variance of our Monte Carlo estimator $\h Z$ is bounded by $\mathcal{O}(p / |B|)$, where $|B|$ is the batch size. In practice, tracking $\h Z$ via an Exponential Moving Average (EMA) algorithmically smooths out any remaining batch-level variance, ensuring stable Primal-Dual dynamics.

Crucially, this estimator relies on the batch mean to approximate the
expectation over $q(x)$, rather than summing over the entire support
$\sX_0$.  This avoids the need to know the total support size
$|\sX_0|$ (which is intractable for sequence models), making the
constraint enforcement computationally feasible.

\subsection{Sampling Algorithms}

\subsubsection{The SIR algorithm}
\label{sec:SIR}

The primary goal in this section is to draw sequence samples from the
robust, gated model $\pi_{g^*}$. While our optimization successfully
finds a gate $g^*$ that ensures the distribution
$\pi_{g^*}(x) = \sum_k g^*(x, k) \hat{\pi}_k(x)$ is globally
normalized, the resulting model presents a unique challenge for
standard LLM inference. Standard LLMs generate text autoregressively,
predicting the next token based solely on past tokens. However, our
optimal gate $g^*(x, \cdot)$ is non-causal: it determines the mixture
weights based on the complete sequence $x$, meaning the probability of
the first token theoretically depends on the last. Since we cannot
generate tokens one by one if the routing logic depends on the
finished sentence, we must resort to approximation methods like
Sampling-Importance-Resampling (SIR) that generate complete candidates
first and score them later.

The SIR algorithm derives its name from its three distinct stages,
each addressing a specific part of this non-causal hurdle:

\begin{enumerate}

\item Sampling: Since we cannot sample directly from the target
  $\pi_{g^*}$, we first generate a set of $N$ candidate sequences from
  a \emph{proposal distribution} $q(x)$ that is easy to sample
  from. We define this proposal as the uniform mixture of our experts,
  $q(x) = \frac{1}{p} \sum \hat{\pi}_k(x)$. This allows us to use
  standard autoregressive generation: we simply pick an expert at
  random and have it generate a full sequence.

\item Importance: We acknowledge that these candidates were drawn from
  the \emph{wrong distribution} ($q$ instead of $\pi_{g^*}$). To
  correct for this, we assign an importance weight $w(x)$ to each
  candidate, calculated as the likelihood ratio
  $w(x) = \pi_{g^*}(x) / q(x) = \frac{\sum_{k = 1}^p g^*(x, k) \, \h
    \pi_k(x)} {\frac{1}{p} \sum_{k=1}^p \h \pi_k(x)}$. This step is
  computationally expensive but feasible because we are evaluating
  completed sequences; we can pass the full candidate $x$ to the
  non-causal gate $g^*$ to compute its true probability under the
  robust model.

\item Resampling: Finally, to obtain samples that approximate the
  robust target distribution, we resample from our pool of
  candidates. A candidate is selected with probability proportional to
  its importance weight, ensuring that sequences with high probability
  under the robust model $\pi_{g^*}$ are more likely to be chosen as
  the final output.

\end{enumerate}

Here's the pseudocode of the SIR algorithm.

\begin{algorithm}[tb]
\caption{Sampling via Sampling-Importance-Resampling (SIR)}
\label{alg:sir-sampling-kl}
\begin{algorithmic}[1]
  \STATE {\bfseries Input:} Robust gate $g^* \in \sGone$,
  expert models $\{\h \pi_k\}_{k=1}^p$, number of candidates $N$.
  \STATE {\bfseries Initialize:} Empty lists $C \leftarrow []$ (candidates)
  and $W \leftarrow []$ (weights).
  \STATE \COMMENT{Step 1: Generate N candidates from the proposal $q(x)$.}
  \FOR{$i = 1$ {\bfseries to} $N$}
    \STATE Sample expert $k \sim \text{Uniform}(\{1, \ldots, p\})$.
    \STATE Sample sequence $x^{(i)} \sim \h \pi_k(x)$ (autoregressively).
    \STATE Append $x^{(i)}$ to $C$.
  \ENDFOR
  \STATE \COMMENT{Step 2: Compute importance weights.}
  \FOR{$i = 1$ {\bfseries to} $N$}
    \STATE $x \leftarrow C[i]$.
    \STATE \COMMENT{This step requires evaluating $g^*$ and
      all $p$ experts on $x$.}
    \STATE Compute $\pi_{g^*}(x) = \sum_{k = 1}^p g^*(x, k) \, \h \pi_k(x)$.
    \STATE Compute $q(x) = \frac{1}{p} \sum_{k=1}^p \h \pi_k(x)$.
    \STATE $w_i \leftarrow \pi_{g^*}(x) / q(x)$ if $q(x) > 0$
    else ($0$ if $\pi_{g^*}(x)=0$,
    $\infty$ otherwise). \COMMENT{Handle $q(x)=0$ case.}
    \STATE Append $w_i$ to $W$.
  \ENDFOR
  \STATE \COMMENT{Step 3: Resample one candidate based on weights.}
  \STATE Filter out candidates with non-finite weights. Let indices be $I_{finite}$.
  \STATE Calculate total weight $W_{sum} = \sum_{j \in I_{finite}} W[j]$.
  \IF[If all weights are zero or infinite]{$W_{sum} = 0$ {\bfseries or not}
    is finite($W_{sum}$)}
  \STATE Sample $i^*$ uniformly from $\{1, \ldots, N\}$.
  \COMMENT{Fallback: uniform choice or error}
  \ELSE
  \STATE Define normalized probabilities $P_i = W[i] / W_{sum}$
  for $i \in I_{finite}$ (0 otherwise).
    \STATE Sample an index $i^* \sim \text{Categorical}(\{P_i\}_{i=1}^N)$.
  \ENDIF
  \STATE $x^* \leftarrow C[i^*]$.
  \STATE {\bfseries Return:} The final sample $x^*$.
\end{algorithmic}
\end{algorithm}

The inclusion of the fallback mechanism (Line 19) is a safeguard
against numerical instability (underflow), rather than a theoretical
necessity.  Structurally, the absolute continuity assumption required
for SIR consistency is strictly satisfied: since both the target
$\pi_{g^*}$ and the proposal $q$ are mixtures of the \emph{same} set
of experts $\{\h \pi_k\}$, the support of the target is contained
within the support of the proposal
($\supp(\pi_{g^*}) \subseteq \bigcup_k \supp(\h \pi_k) = \supp(q)$).
Therefore, the theoretical case where $q(x)=0$ and $\pi_{g^*}(x)>0$
(leading to infinite weights) is impossible.  Biased uniform sampling
is thus only triggered in rare cases of floating-point underflow.  The
computational cost is concentrated in the weight calculation (Line
11). Evaluating the target density $\pi_{g^*}(x)$ requires a full
forward pass of the gate and all $p$ experts for each of the $N$
candidate sequences. Thus, the inference cost scales as $O(N p)$,
motivating the need for the efficient distillation methods proposed in
Section~\ref{sec:efficient-distillation}.

\subsubsection{An exact sampling alternative: rejection sampling}
\label{sec:rejection-sampling}

While SIR provides asymptotic guarantees, \emph{Rejection Sampling}
offers a method to obtain \emph{exact} samples from the robust
distribution $\pi_{g^*}$, provided we can strictly bound the ratio
between the target and the proposal.

Recall that Rejection Sampling is a fundamental Monte-Carlo technique
used to generate observations from a complex target distribution
$\pi(x)$ using a simpler, tractable proposal distribution $q(x)$. The
main idea is: if we can find a constant $M$ such that the scaled
proposal $M q(x)$ always \emph{envelopes} the target (i.e.,
$\pi(x) \leq M  q(x)$ for all $x$), we can sample from $q(x)$ and
stochastically accept points that fall under the curve of
$\pi(x)$. Samples where $M \cdot q(x)$ is much larger than $\pi(x)$
are rejected more frequently, effectively \emph{carving out} the
correct distribution from the proposal.

In our framework, we can again use the uniform mixture of experts as
the proposal, $q(x) = \frac{1}{p} \sum_{k=1}^p \h \pi_k(x)$. A crucial
property of our normalized gate space $\sGone$ allows us to derive the
strictly required bound $M$. Since the gating weights $g^*(x, k)$ are
probabilities bounded by $1$, the robust model's likelihood is
strictly bounded by the sum of the individual experts:
\[
  \pi_{g^*}(x)
  = \sum_{k=1}^p g^*(x, k) \h \pi_k(x)
  \le \sum_{k=1}^p \h \pi_k(x) = p \, q(x).
\]
This provides the strict envelope constant $M = p$.

\begin{algorithm}[tb]
\caption{Exact Sampling via Rejection Sampling}
\label{alg:rejection-sampling}
\begin{algorithmic}[1]
  \STATE {\bfseries Input:} Robust gate $g^* \in \sGone$,
  expert models $\{\h \pi_k\}_{k=1}^p$.
  \STATE {\bfseries Output:} An exact sample $x^* \sim \pi_{g^*}$.
  \LOOP
    \STATE Sample expert $k \sim \text{Uniform}(\{1, \ldots, p\})$.
    \STATE Sample candidate $x \sim \h \pi_k(x)$. \COMMENT{Proposal $x \sim q(x)$}
    \STATE Compute acceptance probability:
    \[
    A(x) = \frac{\pi_{g^*}(x)}{M \, q(x)} = \frac{\pi_{g^*}(x)}{p \, q(x)}
    \]
    \STATE Sample $u \sim \text{Uniform}([0, 1])$.
    \IF{$u \le A(x)$}
        \STATE {\bfseries return} $x$
    \ENDIF
  \ENDLOOP
\end{algorithmic}
\end{algorithm}

The efficiency of Rejection Sampling is strictly determined by the
constant $M = p$, which represents the ratio of the area under the
enveloping proposal $M q(x)$ to the area under the target
$\pi_{g^*}(x)$. Geometrically, the algorithm samples points uniformly
under the envelope; the probability that such a point also falls under
the target curve is exactly $1/M$. Consequently, the number of trials
required to find a successful sample follows a geometric distribution
with an expected value of $M = p$.

This reveals a clear trade-off: for a moderate number of experts
(e.g., $p \leq 10$), the \emph{computational waste} of rejecting
candidates is a reasonable price for obtaining unbiased, exact
samples. However, because the bound $M = p$ grows linearly with the
ensemble size, the acceptance rate $1/p$ drops rapidly. For large $p$
(e.g., $p=100$), one would discard approximately $99\%$ of generated
candidates, making the method prohibitive. In these high-dimensional
regimes, SIR becomes the preferred alternative.

\subsection{Baseline: Efficient Sampling via Monolithic Distillation}
\label{sec:efficient-distillation}

While the Rejection Sampling and SIR algorithms provide exact or
asymptotically exact samples from the robust model $\pi_{g^*}(x)$,
their inference cost scales linearly with the number of experts
(requiring $\mathcal{O}(p)$ or $\mathcal{O}(Np)$ operations per sample), which may be prohibitive for
large-scale deployment. The bottleneck is the non-causal nature of the
optimal gate $g^*(x, k)$, which depends on the complete sequence $x$,
preventing efficient caching or standard token-by-token generation.

To achieve efficient autoregressive sampling in constant time with respect to $p$, we can distill the robust knowledge
from the non-causal target model $\pi_{g^*}$ into a new \emph{causal
  student model} $\pi_{\text{causal}}$. This student model is
parameterized as a standard causal Transformer with parameters
$\theta$, ensuring the factorization
$\pi_{\text{causal}}(x) = \prod_{t=1}^T \pi_{\text{causal}}(x_t \mid
x_{<t})$.

We train the student model by minimizing the Kullback-Leibler
divergence from the robust target $\pi_{g^*}$ to the student
$\pi_{\text{causal}}$ over the space of sequences:
\begin{align*}
& \min_{\theta} \KL \paren*{ \pi_{g^*} \parallel \pi_{\text{causal}} }\\
  &= \min_{\theta} \E_{x \sim \pi_{g^*}}
  \bracket*{ -\sum_{t=1}^T \log \pi_{\text{causal}}(x_t \mid x_{<t}) } - H(\pi_{g^*}).
\end{align*}

In practice, this is equivalent to maximizing the log-likelihood of
the student model on a dataset of synthetic sequences generated from
$\pi_{g^*}$. The training procedure is as follows:
\begin{enumerate}

\item \textbf{Generate Data:} Use the exact Rejection Sampling method
  (\cref{alg:rejection-sampling}) or SIR (\cref{alg:sir-sampling-kl})
  to generate a large dataset of robust sequences
  $\cD_{\text{robust}} = \{x^{(i)}\}_{i=1}^M$ drawn from $\pi_{g^*}$.

\item \textbf{Train Student:} Train the causal Transformer
  $\pi_{\text{causal}}$ on $\cD_{\text{robust}}$ using standard
  cross-entropy loss (next-token prediction).
\end{enumerate}
This distillation step transfers the robustness guarantees of the
non-causal gate into the weights of the causal student. At inference
time, the expensive ensemble $\pi_{g^*}$ is discarded, and samples are
drawn efficiently from $\pi_{\text{causal}}$ using standard
autoregressive decoding.

\subsection{Structural Distillation Theory}
\label{app:structural_distillation}

\subsubsection{Decomposition Proof}
\label{app:distillation-decomposition}

\DistillationDecomposition*

\begin{proof}
We expand the definition of the KL divergence:
\[
  \cJ(\phi) = \E_{x \sim \pi_{g^*}} [\log \pi_{g^*}(x)]
  - \E_{x \sim \pi_{g^*}} [\log \pi_\gamma(x)].
\]
The first term is the negative entropy of the teacher distribution
$\pi_{g^*}$ and is fixed. Thus, minimizing the KL divergence is
equivalent to maximizing the second term.  Unlike the teacher, the
student model $\pi_\gamma$ is defined to be causal and
autoregressive. Therefore, its log-probability factors into a sum of
conditional log-probabilities:
\[
  \log \pi_\gamma(x) = \sum_{t=1}^T \log \pi_\gamma(x_t \mid x_{<t}).
\]
Substituting this back into the expectation yields:
\[
  \max_{\phi} \E_{x \sim \pi_{g^*}} \bracket*{ \sum_{t=1}^T \log
    \paren*{ \sum_{k=1}^p \gamma_\phi(x_{<t}, k) \, \h \pi_k(x_t \mid x_{<t}) } }.
\]
Taking the gradient gives the result. 
\end{proof}

\subsubsection{Theoretical Analysis of Distillation}
\label{app:realizability-structured-distillation}

To establish a rigorous theoretical footing for structural
distillation (\cref{sec:distillation}), here, we analyze the
divergence between the distribution induced by the non-causal teacher,
$\pi_{g^*}$, and the causal student, $\pi_{\gamma}$.  We explicitly
decompose this error into the sum of step-wise divergences between the
student router and the optimal Bayesian posterior of the teacher,
proving that there is no irreducible structural mismatch.

\textbf{Definitions and Model Classes.} Let $\sX$ be the token vocabulary and $\sX^T$ be the space of trajectories.

The Teacher (Mixture of Products). 
The robust gate $g^* \in \sGone$ defines a mixture over expert
trajectories. The likelihood of a sequence $x$ is:
\[
  \pi_{g^*}(x) = \sum_{k=1}^p g^*(x, k)
  \underbrace{\paren*{ \prod_{t=1}^T \h \pi_k(x_t \mid x_{<t}) }}_{\text{Expert } k \text{ trajectory}}.
\]
This represents a \emph{Mixture of Products}. The latent expert choice
$k$ is sampled once per sequence, maintaining mode consistency (e.g.,
sticking to one domain for the whole sentence).

The Student (Product of Mixtures). 
The causal router $\gamma_{\phi}$ defines a distribution where mixing
happens at every step $t$:
\[
  \pi_{\gamma}(x) = \prod_{t=1}^T \underbrace{\paren*{ \sum_{k=1}^p \gamma_{\phi}(x_{<t}, k)
      \, \h \pi_k(x_t \mid x_{<t}) }}_{\text{Step-wise mixture}}.
\]
This represents a \emph{Product of Mixtures}. The effective expert
weight $\gamma$ can change at every token.

\textbf{The Bayes-Optimal Causal Router.} We do not merely assume a \emph{good} router exists. Instead, we
derive the optimal causal policy $\gamma^*$ that minimizes the
approximation error to the teacher.

\begin{proposition}[The Posterior Mean Router]
\label{prop:posterior-mean}
For any history $h = x_{<t}$, the optimal causal routing weights
$\gamma^*_k(h)$ are given by the \emph{posterior probability} of
expert $k$ given the history, under the teacher distribution
$\pi_{g^*}$:
\[
  \gamma^*_k(x_{<t})
  = P_{\pi_{g^*}}(K=k \mid x_{<t})
  = \frac{\E_{x' \sim \pi_{g^*}} \bracket*{ \mathbb{I}[x'_{<t} = x_{<t}] \cdot g^*(x', k) }}{ \pi_{g^*}(x_{<t}) }.
\]
\end{proposition}

\begin{proof}
  The student model is a product of mixtures:
  $\pi_\gamma(x_t|h) = \sum_k \gamma_k(h) \h \pi_k(x_t|h)$.  The
  teacher model, despite being non-causal in parameterization, implies
  a valid marginal conditional distribution:
  \[
    \pi_{g^*}(x_t \mid h) = \sum_{k=1}^p P_{\pi_{g^*}}(k \mid h) \, \h \pi_k(x_t \mid h).
  \]
  By setting $\gamma^*_k(h) = P_{\pi_{g^*}}(k \mid h)$, the student's
  conditional distribution becomes identical to the teacher's
  conditional distribution at every step.  Thus, this choice of
  $\gamma^*$ is optimal (achieving zero local divergence).
\end{proof}

\textbf{Exact Decomposition of the Distillation Error.} We now provide an exact decomposition of the total distillation error
$\KL(\pi_{g^*} \parallel \pi_{\gamma})$ using the chain rule of
relative entropy.  This replaces heuristic approximations with a
rigorous bound.

\begin{theorem}[Exact Chain Rule Decomposition]
\label{th:causal-decomposition}
Let $\pi_{\gamma}$ be the student model parameterized by $\phi$. The
total divergence decomposes exactly into a sum of step-wise
divergences:
\[
  \KL(\pi_{g^*} \parallel \pi_{\gamma})
  = \sum_{t=1}^T \E_{x_{<t} \sim \pi_{g^*}} \bracket*{
    \KL \paren*{ \pi_{g^*}(\cdot \mid x_{<t}) \parallel \pi_{\gamma}(\cdot \mid x_{<t}) }
  }.
\]
Furthermore, this error is upper-bounded by the divergence between the
routing policies:
\[
  \KL(\pi_{g^*} \parallel \pi_{\gamma})
  \leq \sum_{t=1}^T \E_{x_{<t} \sim \pi_{g^*}} \bracket*{
    \KL \paren*{ \gamma^*(\cdot \mid x_{<t}) \parallel \gamma_\phi(\cdot \mid x_{<t}) }
  }.
\]
\end{theorem}

\begin{proof}
  The first equality is the standard Chain Rule for Kullback-Leibler
  divergence applied to autoregressive sequence models.
  
  For the inequality, recall that the conditional distributions are mixtures:
  $P(\cdot|h) = \sum_k \gamma^*_k(h) \h \pi_k(\cdot|h)$ and
  $Q(\cdot|h) = \sum_k \gamma_{\phi, k}(h) \h \pi_k(\cdot|h)$.
  By the joint convexity of the KL divergence,
  $\KL(\sum_k \lambda_k P_k \parallel \sum_k \mu_k P_k) \le \KL(\lambda \parallel \mu)$.
  Applying this to our mixtures:
  \[
    \KL \paren*{ \sum_k \gamma^*_k \h \pi_k \parallel \sum_k \gamma_{\phi, k} \h \pi_k }
    \leq \KL(\gamma^* \parallel \gamma_\phi).
  \]
  Summing this bound over all time steps $t$ completes the proof.
\end{proof}

Interpretation.
This theorem clarifies that there is no irreducible ``structural
mismatch" error ($\mathcal{E}_{struct} = 0$) because the teacher's
distribution is perfectly realizable by a causal product of mixtures
using the posterior weights $\gamma^*$.  The total error is driven
entirely by the \emph{Router Approximation Error}: the inability of
the parameterized router $\gamma_\phi$ (e.g., a small Transformer) to
perfectly match the complex posterior distribution $\gamma^*$ induced
by the non-causal gate.

\textbf{Consistency of the Algorithm.} Finally, we confirm that the standard distillation objective minimized
by Algorithm~\ref{alg:structural-distillation} is equivalent to
minimizing the router approximation error.

\begin{corollary}
  Minimizing the sequence-level objective
  $\cJ(\phi) = \KL(\pi_{g^*} \parallel \pi_{\gamma})$ is equivalent to
  minimizing the expected step-wise divergence between the true
  posterior router $\gamma^*$ and the student router $\gamma_\phi$.
\end{corollary}

\begin{proof}
  From Theorem~\ref{th:causal-decomposition}, the total divergence is
  exactly the sum of expected local divergences:
  \[
    \KL(\pi_{g^*} \parallel \pi_{\gamma})
    = \sum_{t=1}^T \E_{x_{<t} \sim \pi_{g^*}} \bracket*{
      \KL \paren*{ \gamma^*(\cdot \mid x_{<t}) \parallel \gamma_\phi(\cdot \mid x_{<t}) }
    }.
  \]
  The terms $\gamma^*(\cdot|x_{<t})$ are fixed targets derived from
  the teacher. Therefore, gradient descent on the global objective
  $\cJ(\phi)$ directly minimizes the discrepancy between the student's
  routing decisions and the optimal Bayesian update at every time
  step.
\end{proof}

\subsection{Efficient Structural Distillation via Cached Logits (Algorithm~\ref{alg:structural-distillation})}
\label{app:structural-transfer}

\begin{algorithm}[H]
\caption{Efficient Structural Distillation via Cached Logits}
\label{alg:structural-distillation}
\begin{algorithmic}[1]
  \STATE {\bfseries Input:} Robust gate $g^*$, frozen experts $\{\h \pi_k\}$,
  dataset size $M$, router $\gamma_\phi$.
  
  \STATE \textbf{Phase 1: Data Generation \& Caching}
  \STATE Generate $M$ sequences $\{x^{(i)}\}$ from $\pi_{g^*}$
  using Rejection Sampling (Alg.~\ref{alg:rejection-sampling}) or SIR.
  \STATE Initialize dataset $\cD \leftarrow \emptyset$.
  
  \FOR{each sequence $x^{(i)}$ and time step $t$}
    \STATE Run all $p$ experts to get next-token probabilities:
    \STATE $p_{t, k}^{(i)} = \h \pi_k(x_t^{(i)} \mid x_{<t}^{(i)})$
    for $k \in \{1, \dots, p\}$.
    \STATE Store tuple $(x_{<t}^{(i)}, x_t^{(i)}, \bp_t^{(i)})$ in $\cD$.
    \COMMENT{$\bp_t^{(i)}$ is a vector of size $p$}
  \ENDFOR

  \STATE \textbf{Phase 2: Router Training}
  \REPEAT
    \STATE Sample batch of tuples $(h, y, \bp)$ from $\cD$. \COMMENT{$h$: history, $y$: target token}
    \STATE Compute router weights: $\mathbf{w} = \gamma_\phi(h) \in \Delta([1, p])$.
    \STATE Compute mixture probability: $P_{\text{mix}} = \mathbf{w} \cdot \bp = \sum_{k=1}^p w_k p_k$.
    \STATE Compute Loss: $\cL = -\log(P_{\text{mix}})$.
    \STATE Update $\phi \leftarrow \phi - \eta \nabla_\phi \cL$.
  \UNTIL{Convergence}
  
  \STATE {\bfseries Output:} Causal Router $\gamma_\phi$.
\end{algorithmic}
\end{algorithm}

\subsection{Discussion}

We have presented a hierarchy of sampling strategies for the robust
gated model, establishing a trade-off between theoretical exactness,
inference latency, and modularity.

\textbf{Exactness vs. Efficiency.}  The sampling-based methods (SIR
and Rejection Sampling) provide the strongest theoretical
guarantees. As $N \to \infty$, SIR recovers the exact robust
distribution $\pi_{g^*}$, and Rejection Sampling provides exact
samples for any $N$. These methods ensure that the worst-case
performance bound ($\KL \le \max \e_k$) established in
\cref{th:robust-existence} holds precisely. However, the computational
cost of evaluating all $p$ experts for every candidate sample is often
prohibitive for real-time applications.

\textbf{The Role of Distillation.}  The distillation approaches
(\cref{sec:efficient-distillation,app:structural-transfer}) bridge the
gap between theory and practice. By compressing the non-causal
knowledge of $g^*$ into a causal student model, we recover standard
autoregressive inference speeds. This comes at the cost of introducing
a distillation error, $\KL(\pi_{g^*} \parallel \pi_{\text{student}})$,
which represents the loss in robustness due to approximation.

\textbf{Modularity and Structural Distillation.}  Standard causal
distillation results in a monolithic student model, discarding the
modular nature of the original experts. In contrast, the Structural
Distillation method preserves the pre-trained experts, learning only a
lightweight routing policy. This maintains the system's
adaptability (if an expert is improved, the overall system improves
without full retraining) while significantly reducing the inference
overhead compared to the raw non-causal gate. This structural approach
represents the most promising direction for deploying robust, modular
generative models at scale.

\section{Experimental Details and Additional Results}
\label{app:experiments_extended}

In this appendix, we provide the detailed experimental setups and additional results regarding domain overlap that complement the main analysis in Section~\ref{sec:experiments}.

\subsection{Additional Experimental Results: Robustness to Domain Overlap}
\label{app:extra_experiments}

We further carried out experiments where the distributions of our two experts A and B were less contradictory. We did so by mixing up Domain A to have a fraction of Domain B. We experimented with fractions of zero (the experiment described in the main text with `clean' distributions), a fraction of 0.5 and a fraction of 0.75, at which point Domain A only contains 25\% of its original data.

The two additional experiments are illustrated in Figure~\ref{fig:capacity_gap-0.5} and Figure~\ref{fig:capacity_gap-0.75}. The Oracle models maintain their advantage for very skewed test distributions, but the Robust Gate model demonstrates the best performance for most test distributions, despite having fewer total parameters than the larger-sized models.

\begin{figure}[t]
    \centering
    \includegraphics[width=0.49\linewidth]{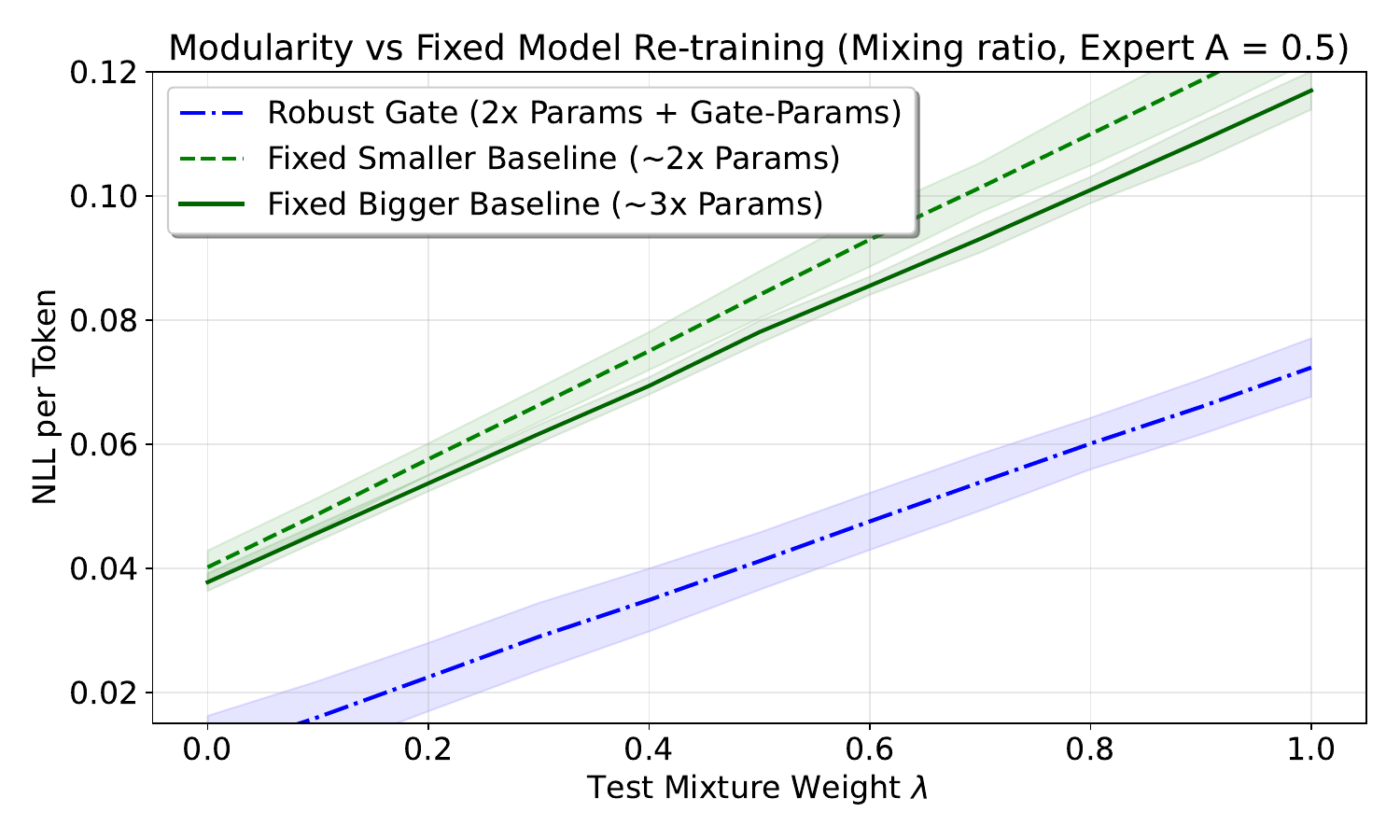}
    \hfill
    \includegraphics[width=0.49\linewidth]{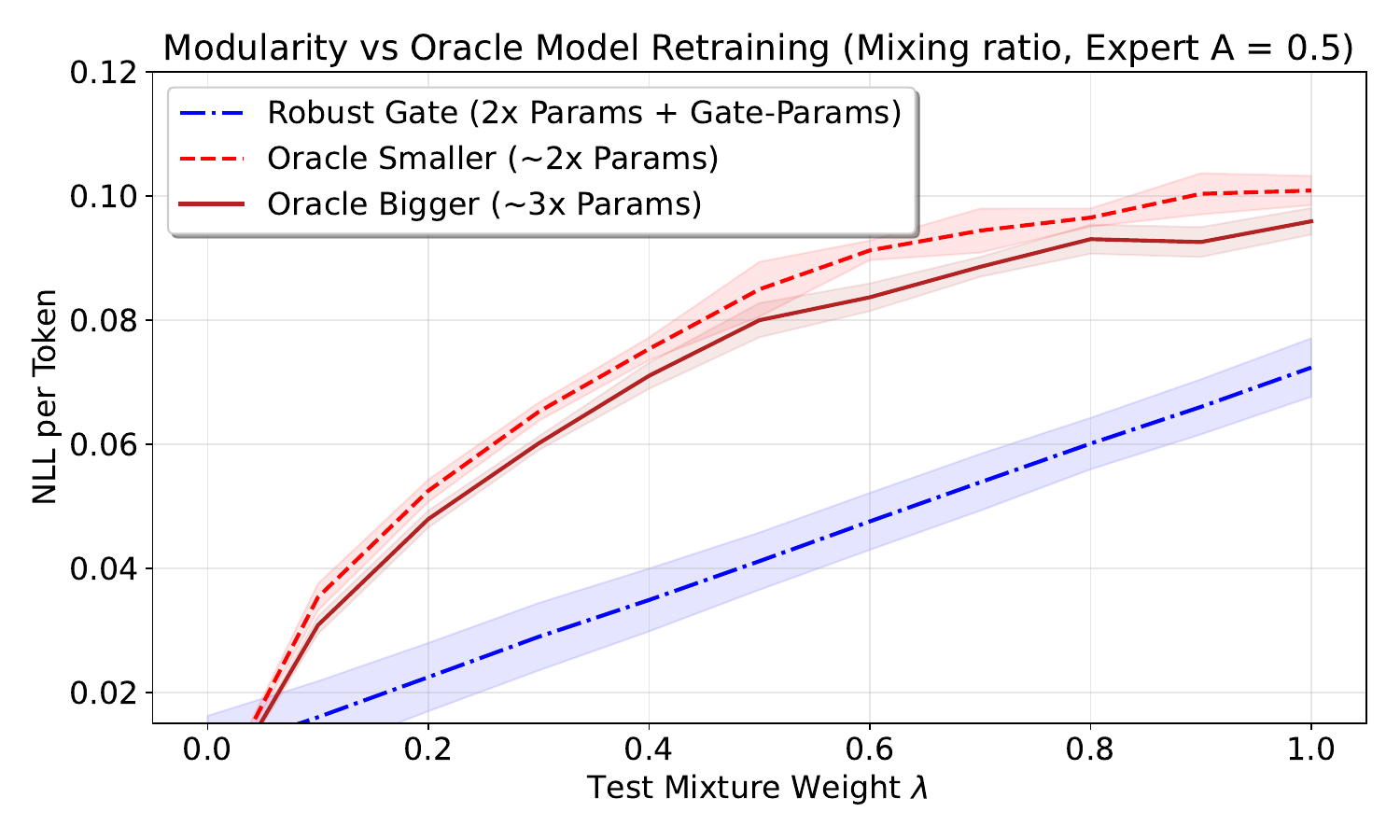}
    \caption{\textbf{Modularity overcomes gradient conflict at a 50-50\%
    mix of Domain A and a pure Domain B.} Left figure illustrates a
    comparison to the Fixed Smaller and Larger models (in Green), while
    the right figure illustrates the comparison to the Oracle models (in
    Red). The Robust Gate (blue) and the Fixed models in the left figure
    naturally obtain the best performance for small values of $\lambda$
    where the test distribution is predominantly made up by Domain B. As
    $\lambda$ increases, the test distribution contains more data from
    Domain A and gets harder for all models. Yet the Robust Gate
    maintains its clear advantage. The Oracle models in the right figure
    still has an advantage for the really skewed distribution and
    $\lambda \sim 0$, but loses to the Robust Gate for larger values of
    $\lambda$.}
    \label{fig:capacity_gap-0.5}
\end{figure}

\begin{figure}[t]
    \centering
    \includegraphics[width=0.49\linewidth]{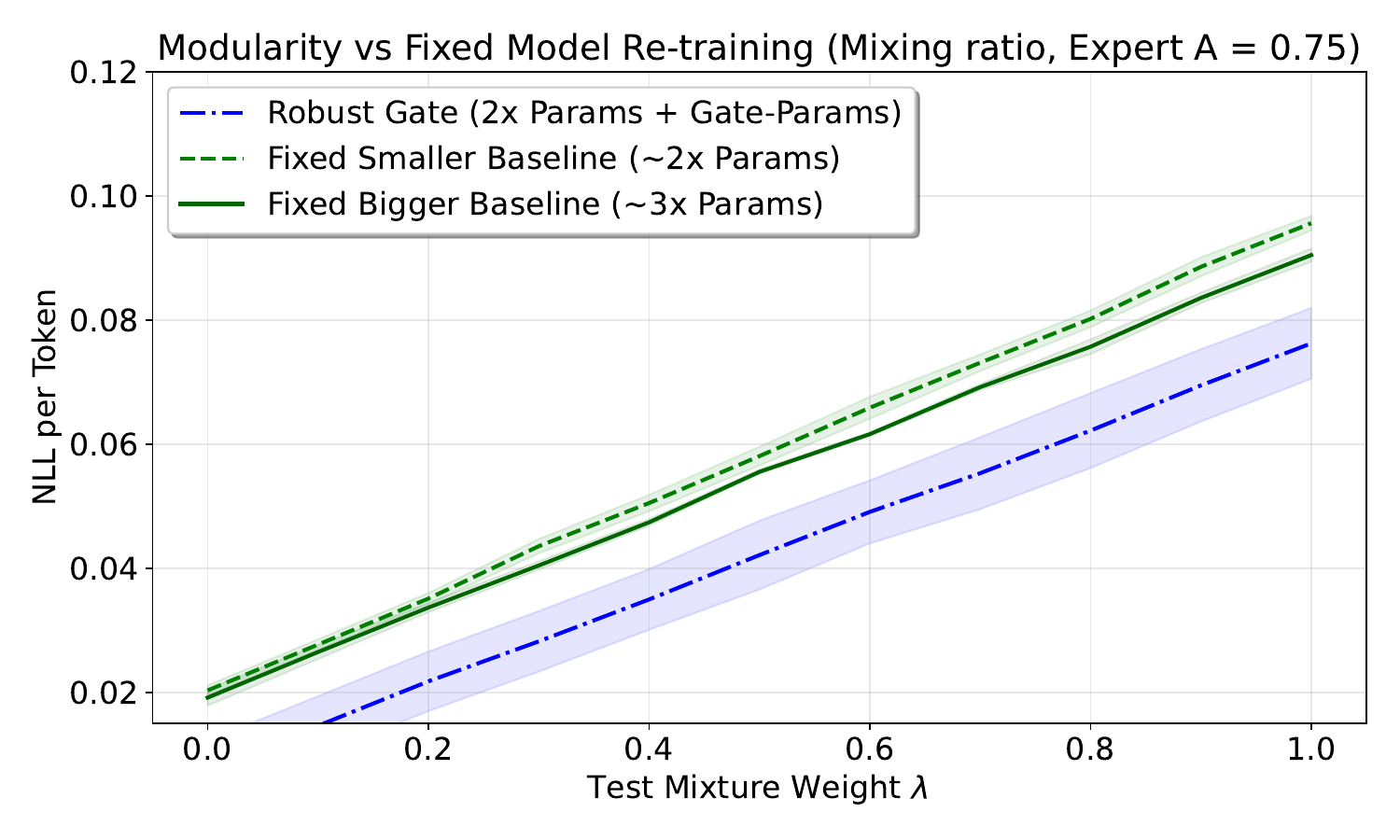}
    \hfill
    \includegraphics[width=0.49\linewidth]{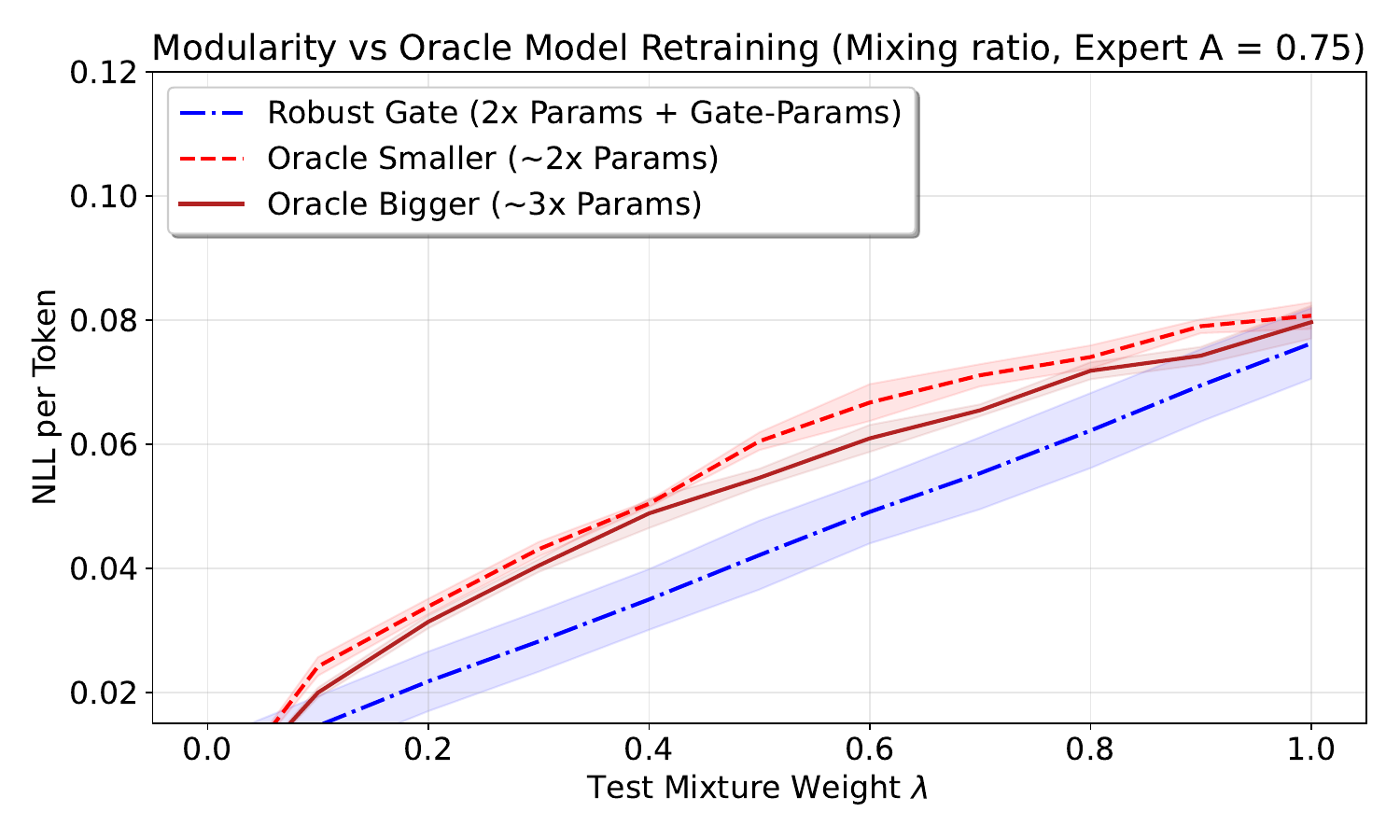}
    \caption{\textbf{Modularity overcomes gradient conflict at a
      25-75\% mix of Domain A and Domain B.} Left figure illustrates
      a comparison to the Fixed Smaller and Larger models (in Green),
      while the right figure illustrates the comparison to the Oracle
      models (in Red). The Robust Gate (blue) and the Fixed models in
      the left figure naturally obtain the best performance for small
      values of $\lambda$ where the test distribution is predominantly
      made up by Domain B. As $\lambda$ increases, the test
      distribution contains more data from Domain A and gets harder
      for all models. Yet the Robust Gate maintains its clear
      advantage. Results are illustrated with lines for the mean
      values over 5 runs and standard deviations indicated with shaded
      regions. The Oracle models in the right figure still has an
      advantage for the really skewed distribution and
      $\lambda \sim 0$, but loses to the Robust Gate for larger values
      of $\lambda$.}
    \label{fig:capacity_gap-0.75}
\end{figure}

\subsection{Detailed Experimental Setup: Synthetic Data}
\label{app:implementation_details}

All models are standard Transformer Encoders (masked for autoregression) with $L=1$ layer, $H=2$ attention heads, and feedforward dimension $d_{\text{ff}}=32$. The different-sized models are obtained by varying the embedding dimension, which is $d=6$ for the Robust Gate model, $d=8$ for the basic Expert, $d=16$ for the Smaller-Retrained model and $d=20$ for the Larger-Retrained model. Since various parts of the transformers scale either linearly or quadradically (like the feedforward network), the number of parameters do not exactly scale linearly in the embedding dimension. Two expert models plus the small gate match approximately the smaller-retrained model, while the larger retrained model has approximately $1.5 \times$ the number of parameters as the combined Robust Gate.

For these experiments, the vocabular size is 100, the sequence length is $T=10$ and batch size is $B=64$. Optimization uses AdamW with $\beta_1=0.9, \beta_2=0.999$ and zero dropout. The experts and baselines are trained for 800 steps with learning rate $\eta=10^{-2}$. The Robust Gate is trained for 800 steps with $\eta_{\text{gate}}=5 \times 10^{-3}$, and the dual variables are updated with $\eta_{\lambda}=0.2, \eta_{\mu}=0.1$. The partition function $Z$ is estimated using a running exponential moving average ($\alpha=0.9$) for variance reduction. The training set size for each expert was 800 batches of 64 examples or 51,000 examples. Both the gate and the smaller and larger models were trained with the union of 102,000 examples.

We evaluate performance across the full spectrum
of distribution shifts by varying the mixture weight
$\lambda \in [0, 1]$ in steps of $0.1$ in the test distribution
$p_\lambda(x) = \lambda p_A(x) + (1-\lambda) p_B(x)$. That is, $\lambda=0$ means all the test data comes from Domain B.
We illustrate the NLL per token across the
mixture range. For all figures, we provide mean values over 5 runs and indicate the standard deviation with shaded regions.

For structural distillation, the Causal Router was implemented with the same transformer architecture as already described. Learning rate and number of training steps are as for the baseline models. 

\subsection{Detailed Experimental Setup: Real-World Data}
\label{app:implementation_details_rw}

We trained 3 experts each on 80K sequences of length 128 on these dataset using the {\tt gpt2} tokenizer. The experts were chosen as two-layer transformers with 2 heads and an embedding dimension of 256, providing them with about 6.5M parameters. The gate model was implemented as a 2-headed 2-layered transformer with an internal dimension of 256. It has just about 290K parameters, making the combined Robust Gate of size 20M parameters. In comparison, we trained a 19.8M parameter model, a transformer with 4 heads, 3 layers and an internal dimension of 184. The Gate model and the Retrained model were trained on the union of the dataset. All models were tested on a hold out sample of 20K sequences.

The learning rate for the AdamW optimizer was set to $1e-4$ for the experts, the gate, and the retrained model. For the gate, the additional learning parameters were set to $\eta_{\lambda}$ = 0.05, $\eta_{\mu}$ = 0.02 and $\alpha = 0.9$.

\subsection{Real-World Robustness to Distribution Shift}
\label{app:robustness_details}

To further assess the stability of the method presented in Section~\ref{sec:experiments-real}, we tested the models on different compositions of the test data. In Table~\ref{tb:lambda} we provide the results from testing the Retrained and the Gate model on these distributions. The distributions are characterized by $\lambda$-test, with $(1/3, 1/3, 1/3)$ corresponding to the uniform distribution. The order of the distributions is  (1) \texttt{\small{wikimedia/wikipedia}}:
High-quality factual prose; (2) \texttt{\small{bigcode/the-stack-smol}}:
Source code across 30+ languages; (3) \texttt{\small{fineweb-edu}}: Filtered
high-quality educational web content.

The performance of the models naturally varies with the distribution, but the Gate model is more robust against these changes and systematically exhibits a lower NLL loss.

\begin{table}[ht]
\vskip 0.1in
    \centering
    \caption{Robustness test results for different test distributions. The results are mean values $\pm$ 1 standard deviation, obtained over 5 runs with different initialization of the model training.}
    \begin{tabular}{ccc}
        \toprule
       \textbf{$\lambda$-test} & \textbf{Retrained, NLL $\pm$ std.dev.} & \textbf{Gate, NLL $\pm$ std.dev.}\\
        \midrule
        1/3, 1/3, 1/3 & 5.133 $\pm$ 0.010 & 4.994 $\pm$ 0.013 \\
        1/3, 1/2, 1/6 & 5.190 $\pm$ 0.080 & 5.068 $\pm$ 0.014 \\
        1/6, 1/3, 1/2 & 5.226 $\pm$ 0.011 & 5.099 $\pm$ 0.014 \\
        1/2, 1/3, 1/6 & 5.042 $\pm$ 0.011 & 4.890 $\pm$ 0.005 \\
        1/2, 1/6, 1/3 & 5.298 $\pm$ 0.006 & 5.117 $\pm$ 0.005 \\
        1/3, 1/6, 1/2 & 5.363 $\pm$ 0.054 & 5.187 $\pm$ 0.057 \\
        1/6, 1/2, 1/3 & 5.279 $\pm$ 0.017 & 5.181 $\pm$ 0.020 \\
        \bottomrule
    \end{tabular}
    \label{tb:lambda}
\end{table}
\ignore{
The results in Figure~\ref{fig:robustness_real} illustrate the model performance relative to the Uniform distribution (set to zero at Index 1). The other points represent the deviation from this uniform test error, averaged over 5 different training initializations. As shown in the figure, the Gate Model is in general more robust to changes in the test distribution, maintaining a performance closer to its baseline compared to the monolithic Retrained Model.}
\end{document}